\documentclass{AIMS}

\usepackage[utf8x]{inputenc}

\usepackage{amsmath}
  \usepackage{paralist}
  \usepackage{graphics} %% add this and next lines if pictures should be in esp format
  \usepackage{epsfig} %For pictures: screened artwork should be set up with an 85 or 100 line screen
\usepackage{graphicx}  \usepackage{epstopdf} %This is to transfer .eps figure to .pdf figure; please compile your paper using PDFLaTex or PDFTeXify.
 \usepackage{cite}
 \usepackage[colorlinks=true]{hyperref}
 \usepackage{comment}
\hypersetup{urlcolor=blue, citecolor=red}

\def\currentvolume{0}
 \def\currentissue{0}
  \def\currentyear{2021}
   \def\currentmonth{12}
    \def\ppages{1--24}

\newtheorem{theorem}{Theorem}[section]

\newtheorem{lemma}[theorem]{Lemma}

\newtheorem{problem}{Problem}
\theoremstyle{definition}
\newtheorem{definition}[theorem]{Definition}

\usepackage{color}              % Define colors
\usepackage{tikz}
\usepackage{pgfplots, pgfplotstable}

  \usetikzlibrary{positioning}
  \pgfplotsset{compat=newest}
\usepackage{xcolor}
	\definecolor{bluecolor}{RGB}{122,166,218}
	\definecolor{redcolor}{RGB}{213, 78, 83}
	\definecolor{orangecolor}{RGB}{231, 140, 69}
	\definecolor{greencolor}{RGB}{185, 202, 74}
	\definecolor{purplecolor}{RGB}{195, 151, 216}

\theoremstyle{remark}

\newcommand{\bs}{\boldsymbol}
\newcommand{\dee}{\text{d}}

\DeclareMathOperator*{\argmin}{arg\,min}
\DeclareMathOperator*{\argmax}{arg\,max}

\DeclareMathAlphabet\mathbfcal{OMS}{cmsy}{b}{n}
\DeclareUnicodeCharacter{2212}{-}

\newcommand{\errorband}[5][]{ % x column, y column, error column, optional argument for setting style of the area plot
\pgfplotstableread{#2}\datatable
    \addplot [draw=none, stack plots=y, forget plot] table [
        x={#3},
        y expr=\thisrow{#4}-2*\thisrow{#5}
    ] {\datatable};

    \addplot [draw=none, fill=blue!50, stack plots=y, area legend, #1] table [
        x={#3},
        y expr=4*\thisrow{#5}
    ] {\datatable} \closedcycle;

    \addplot [forget plot, stack plots=y,draw=none] table [x={#3}, y expr=-(\thisrow{#4}+2*\thisrow{#5})] {\datatable};
}

\title[Bayesian NN priors for edge-preserving inversion]
{Bayesian neural network priors for edge-preserving inversion}


 \email{lichen@cims.nyu.edu}
 \email{matt.dunlop@nyu.edu}
 \email{stadler@cims.nyu.edu}

\thanks{Partially supported by the US National Science Foundation under
  grants \#1723211 and \#1913129.}

\begin{document}
\maketitle

% Enter the first author's name and address:
\centerline{\scshape Chen Li, Matthew Dunlop, and Georg Stadler}
\medskip
{\footnotesize
 \centerline{Courant Institute of Mathematical Sciences, New York University}
   \centerline{251 Mercer Street, New York, NY 10012, USA}
}

\bigskip

%The abstract of your paper
\begin{abstract}
We consider Bayesian inverse problems wherein the unknown state is
assumed to be a function with discontinuous structure a priori. A
class of prior distributions based on the output of neural networks
with heavy-tailed weights is introduced, motivated by existing results
concerning the infinite-width limit of such networks. We show
theoretically that samples from such priors have desirable
discontinuous-like properties even when the network width is finite,
making them appropriate for edge-preserving inversion. Numerically we
consider deconvolution problems defined on one- and two-dimensional
spatial domains to illustrate the effectiveness of these priors; MAP
estimation, dimension-robust MCMC sampling and ensemble-based
approximations are utilized to probe the posterior distribution. The
accuracy of point estimates is shown to exceed those obtained from
non-heavy tailed priors, and uncertainty estimates are shown to
provide more useful qualitative information.

\end{abstract}

\section{Introduction}\label{sec:intro}

Inverse problems aim at inferring information about unknown parameters $u \in \mathcal U$ from observations $y \in \mathcal Y$ and a model
that relates parameters and observations. Denoting the model by
$\mathcal F : \mathcal U \to \mathcal Y$ and assuming additive
observation errors $\varepsilon$, the relation between parameters,
model and observations is
\begin{align}\label{eq:invprob}
  y = \mathcal F (u) + \varepsilon.
\end{align}
In the Bayesian approach to inverse problems, one considers $u,y$ and
$\varepsilon$ as random variables. In particular, one requires a
distribution $p(u)$ that reflects prior knowledge about the parameters $u$. The solution to this Bayesian inverse problem is then the posterior
distribution $p(u|y)$ of the parameters, defined via Bayes' rule as
\begin{align}\label{eq:Bayesrule}
	p (u|y) \propto  p (y | u) p(u),
\end{align}
where `$\propto$' means equality up to a normalization constant, and
$p(y|u)$ is the likelihood derived from the forward model
\eqref{eq:invprob} and the distribution of the errors $\varepsilon$. The Bayesian approach to inverse problems has a long history
\cite{franklin1970, Kaipio2005}. While many important questions around
the characterization of the posterior distribution remain, the main
focus of this paper is on the choice of prior distributions for $u\in
\mathcal U$. In particular, we are interested in problems where
$\mathcal U$ is a space of functions defined over a domain $\mathcal
D\subset \mathbb R^d$, with $d\in \{1,2,3\}$, and we have the prior
knowledge that the true parameter might be a discontinuous function
over $\mathcal D$. Samples from the prior distribution $p(u)$ should
thus include discontinuous (or approximately discontinuous)
functions. This is known to be a challenging problem, and the aim of
this paper is to propose a novel method based on Bayesian neural network
parameterizations to define $p(u)$.

\paragraph{\em Approach}

We study neural network priors with weights drawn from $\alpha$-stable
distributions as proposed in \cite{neal1995BayesianLF}. In Figure~\ref{fig:2dnnprior}, we compare samples generated with
such neural network priors with Cauchy and Gaussian weights, which are
$\alpha$-stable random variables with $\alpha = 1, 2$. As can be seen, neural networks with Cauchy weights generate
samples that have large jumps, compared to those with Gaussian weights. This behavior is due to the
heavy tails of Cauchy distributions, resulting in strongly re-scaled
neural network activation functions that visually appear as
discontinuities in the network outputs.
The prior we propose builds on this observation. To study its properties, we first summarize known
theoretical results on the convergence of neural networks with
$\alpha$-stable weights in the limit of infinite width. Then, for
finite-width networks, we present results on the distribution of
pointwise derivatives of the output and use them to explain
the behavior shown in Figure~\ref{fig:2dnnprior}.  We also discuss
practical methods to condition these neural network priors with
observations using optimization as well as sampling methods.

\setlength{\fboxsep}{-.2pt}
\setlength{\fboxrule}{.5pt}
\begin{figure}[tb]\centering
\begin{tikzpicture}
% Cauchy
\node at (-1,2cm) (img1) {\fbox{\includegraphics[scale=0.31]{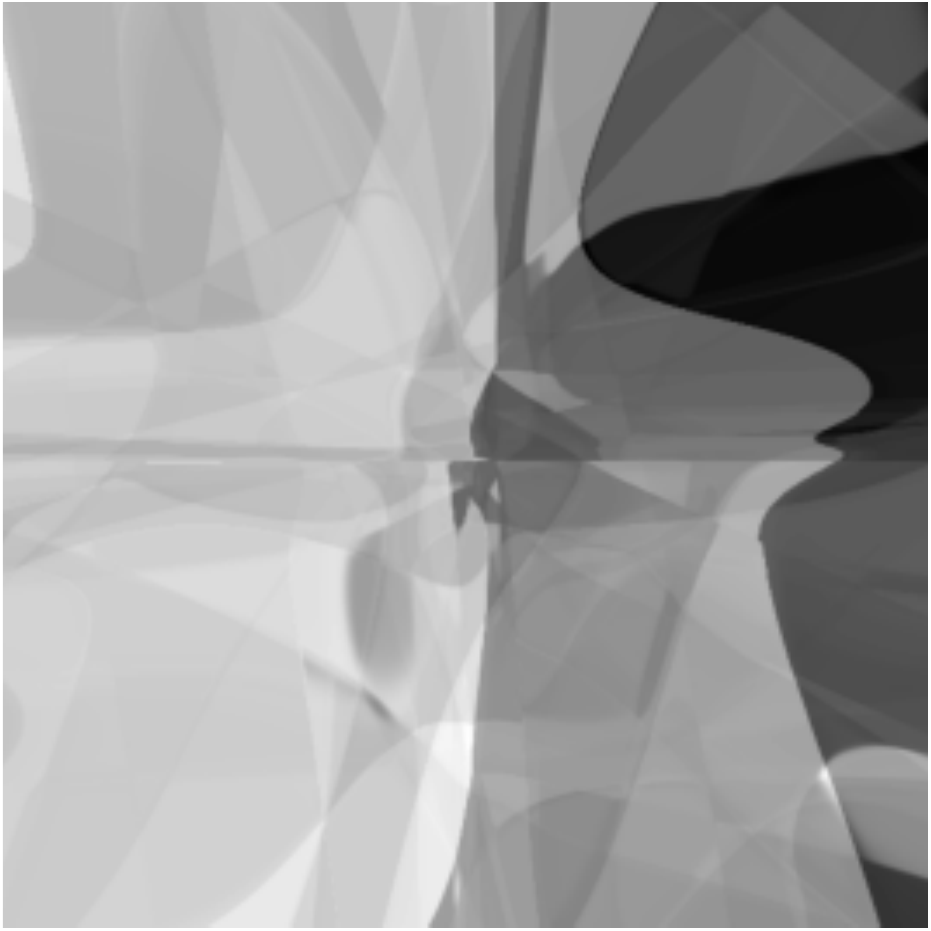}}};
  \node[left=.2cm of img1, node distance=0cm, rotate=90, anchor=center, yshift = 0cm, font=\color{black}] {Cauchy};
 \node at (2.1,2cm) {\fbox{\includegraphics[scale=0.31]{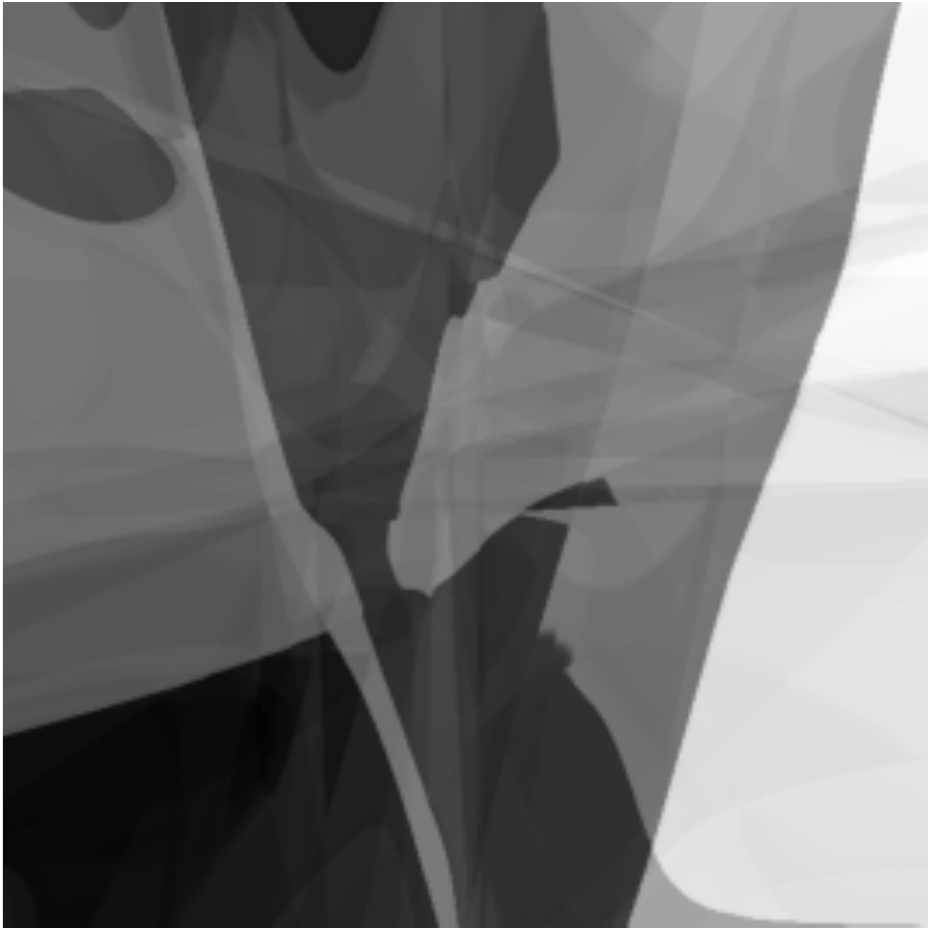}}};
 \node at (5.2,2cm) {\fbox{\includegraphics[scale=0.31]{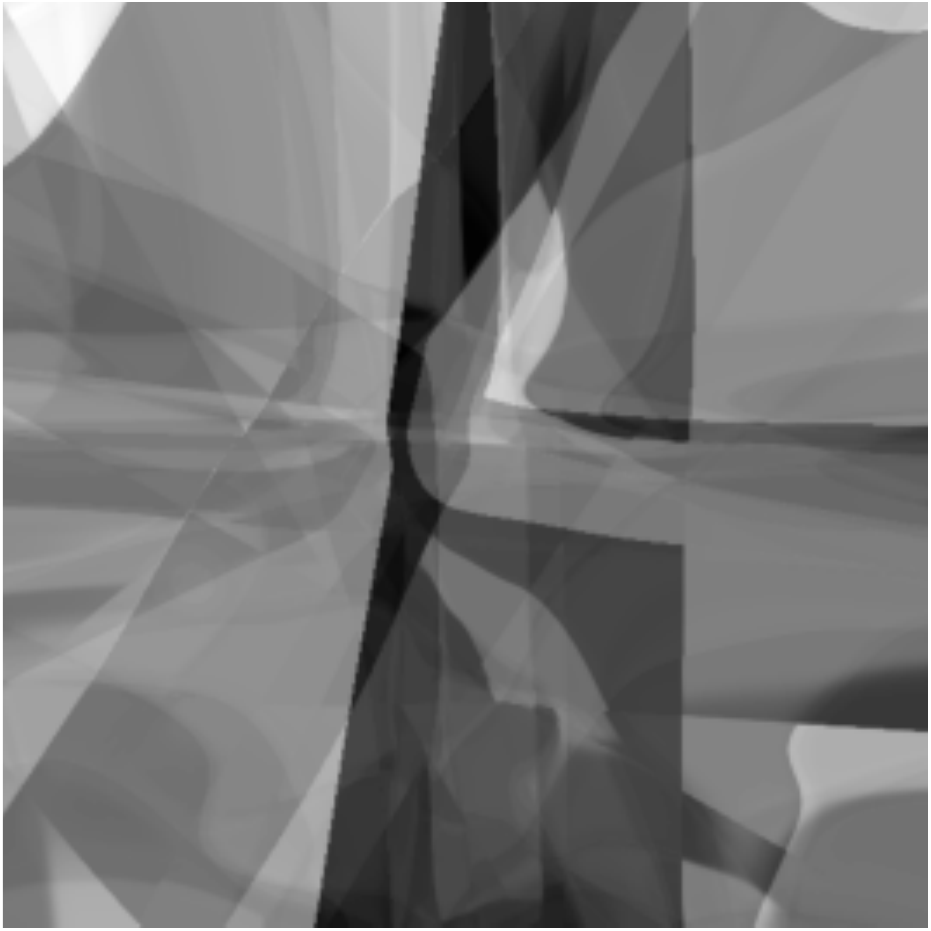}}};
\node at (8.3,2cm) {\fbox{\includegraphics[scale=0.31]{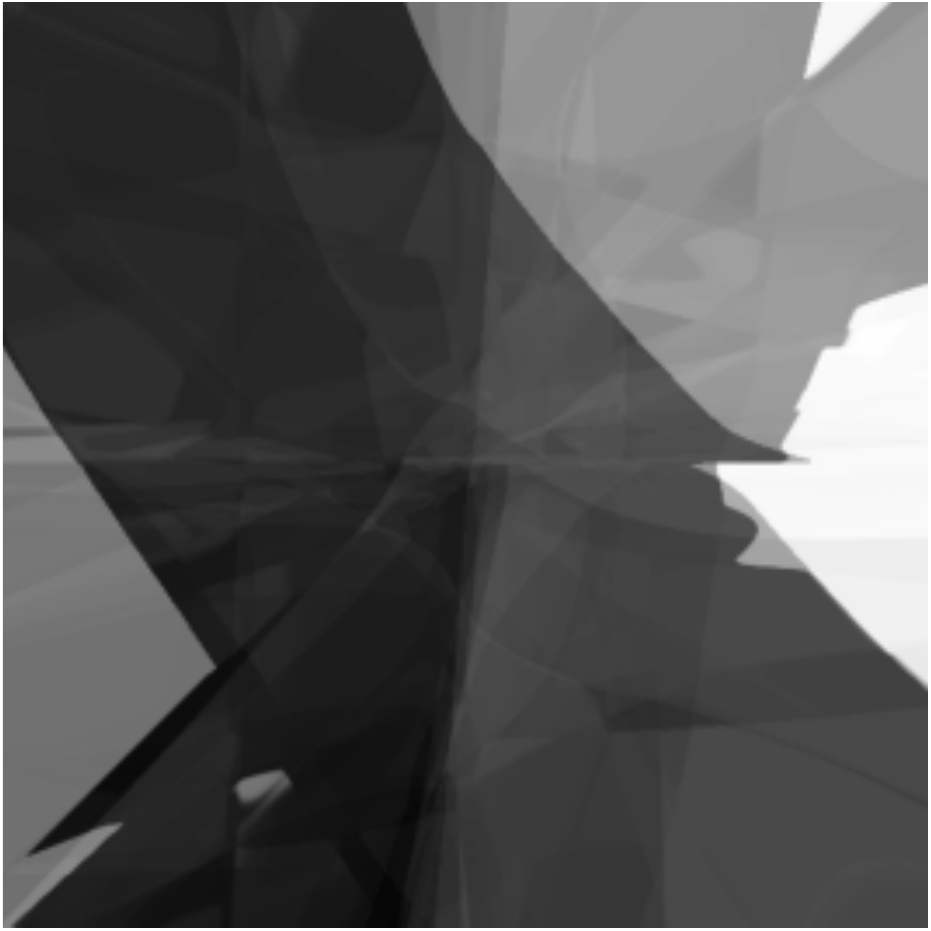}}};

% Gaussian
\node at (-1,-1.5cm) (img2) {\fbox{\includegraphics[ scale=0.31]{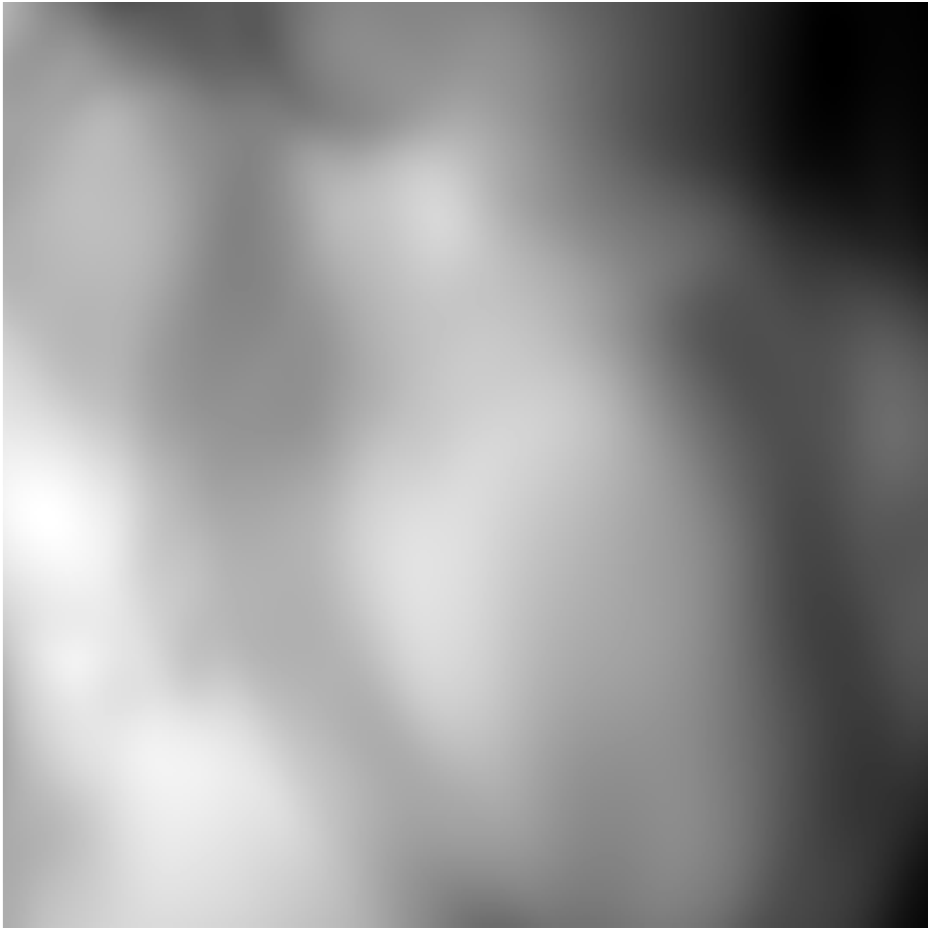}}};
  \node[left=.2cm of img2, node distance=0cm, rotate=90, anchor=center, yshift = 0cm, font=\color{black}] {Gaussian};
 \node at (2.1,-1.5cm) {\fbox{\includegraphics[scale=0.31]{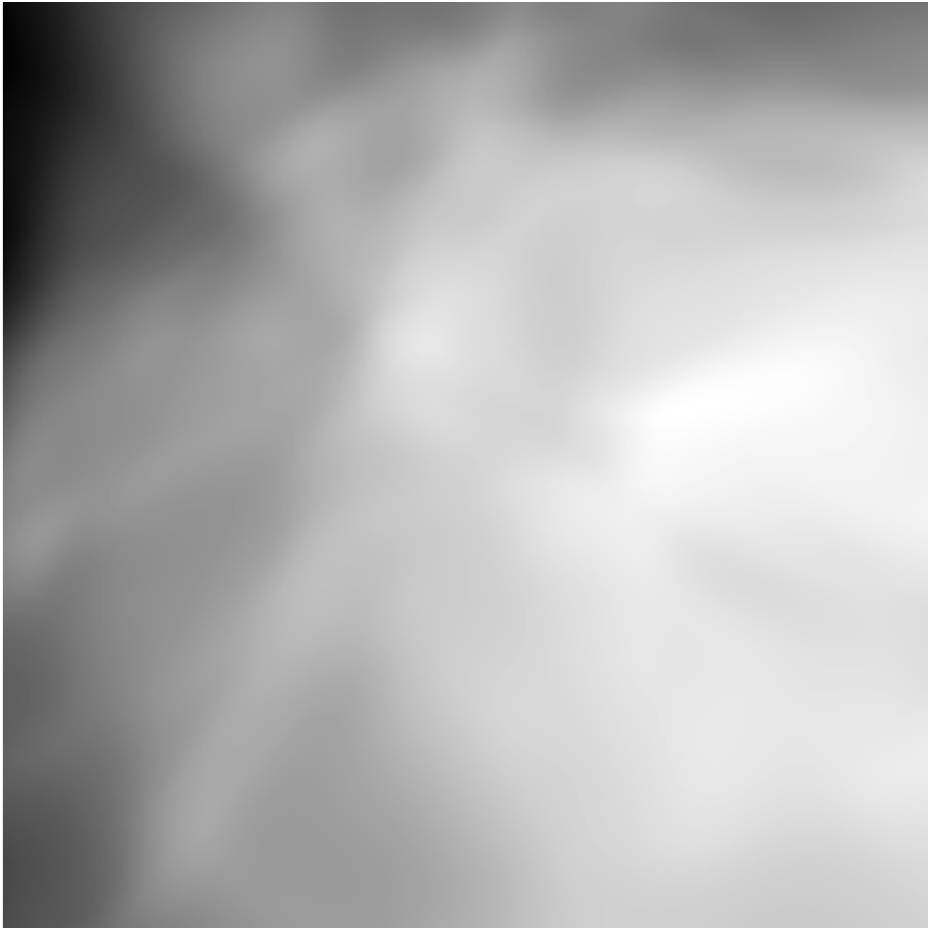}}};
 \node at (5.2,-1.5cm) {\fbox{\includegraphics[scale=0.31]{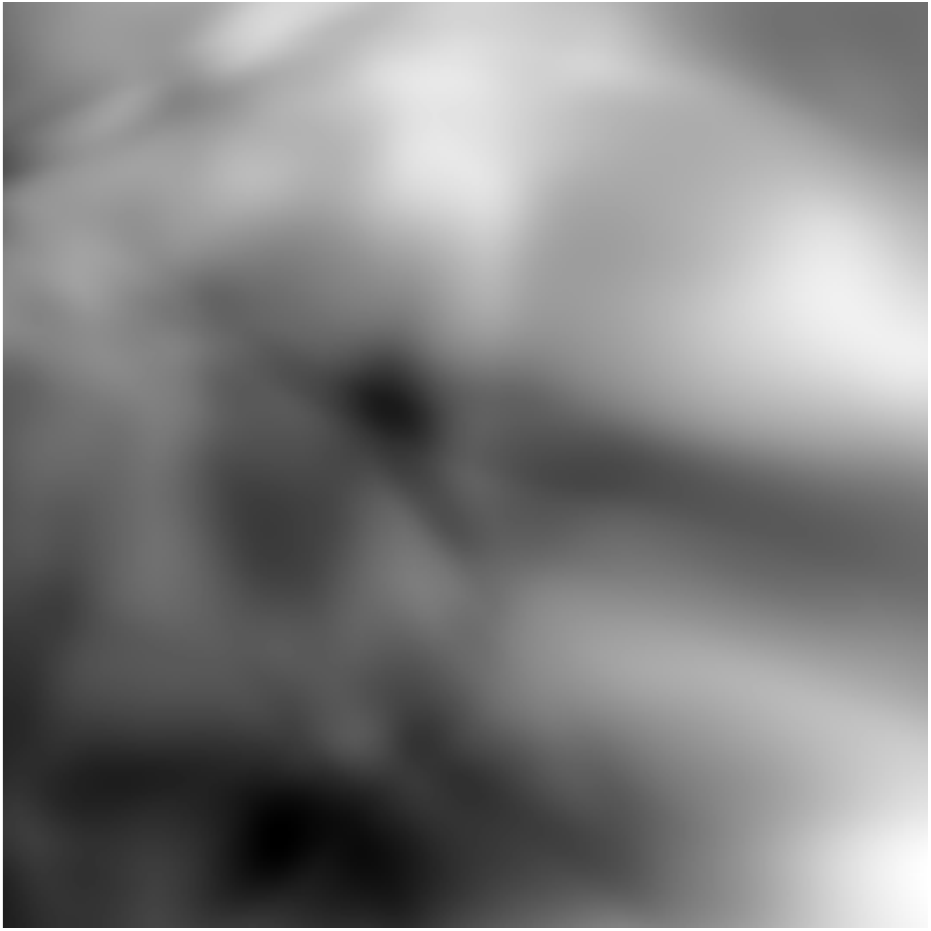}}};
\node at (8.3,-1.5cm) {\fbox{\includegraphics[scale=0.31]{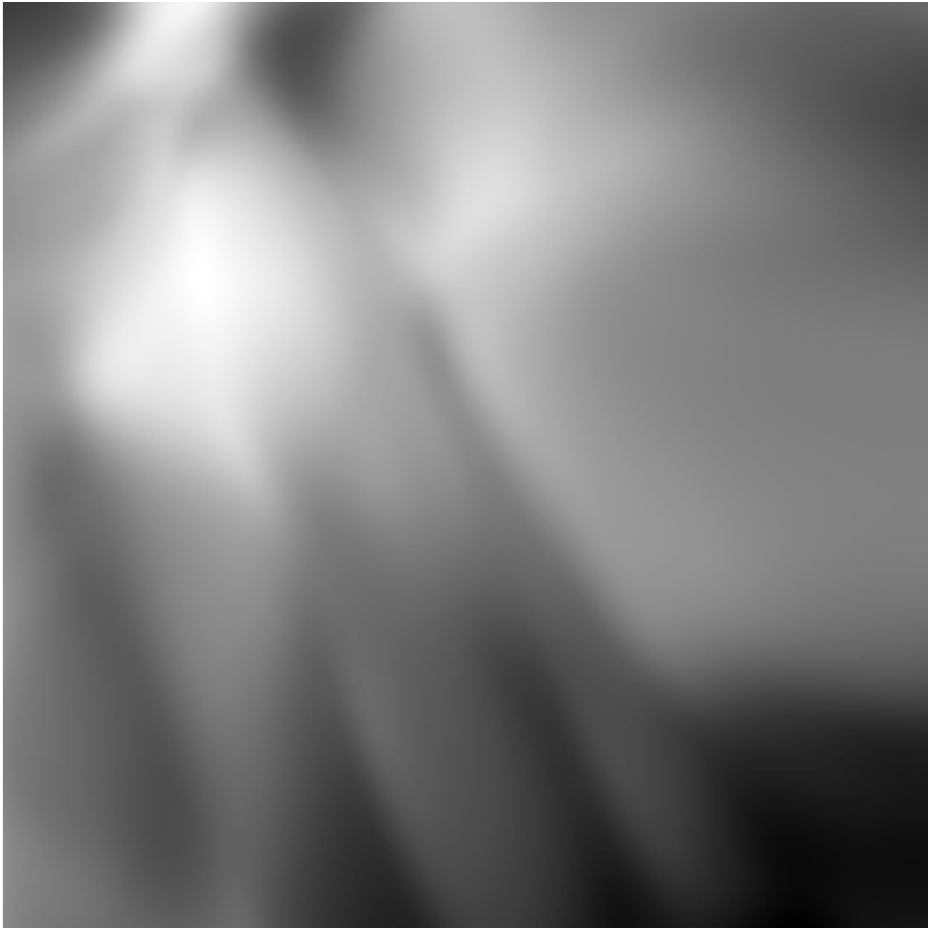}}};
\end{tikzpicture}
\caption{Comparison between outputs on $[-1,1]^2$ of Bayesian neural
  network priors with three hidden layers. Shown are realizations of
  networks with $\tanh(\cdot)$ as activation function and Cauchy
  weights (top) and Gaussian weights (bottom).\label{fig:2dnnprior}}
\end{figure}

\paragraph{\em Contributions}
The main contribution of this work are as follows.
(1) For finite-width networks, we prove that the distribution of the
derivative of the output function at an arbitrary but fixed point is heavy-tailed
provided the weights in the network are heavy-tailed. This explains
the large jumps in the output function generated by certain neural
networks.
(2) We discuss practical methods to condition these neural network
priors using observations, and present a comprehensive numerical study
using image deconvolution problems in one and two space dimensions.

\paragraph{\em Paper outline}

The outline of this work is as follows. In Sec.~\ref{sec:priormodel},
we review existing priors for Bayesian inverse problems and introduce
neural network priors with general $\alpha$-stable distributed
weights. We provide a systematic discussion of the properties of
neural network priors both in the cases of infinite width and finite
width. In Sec.~\ref{sec:BNN}, we review the Bayesian approach for
inverse problems and discuss optimization and sampling methods to
explore the posterior distribution. In
Sec.~\ref{sec:numerical}, we study the numerical performance of
different neural network priors using two deblurring examples with
discontinuous parameter functions. Finally, we draw conclusions and
point out potential research directions in Sec.~\ref{sec:conclu}.

\section{Prior modeling with neural networks}\label{sec:priormodel}
Here, we review priors for infinite-dimensional Bayesian inverse
problems and study properties of neural network priors. We provide an
overview of existing priors for Bayesian inverse problems in
Sec.~\ref{subsec:litereview}. In Sec.~\ref{subsec:nnprior}, we propose
functions parameterized using neural networks with random weights as
priors. Numerical tests motivate using neural networks with weights
drawn from $\alpha$-stable distributions as priors for discontinuous
parameter functions.  We introduce $\alpha$-stable distributions in
Sec.~\ref{subsec:alpha}. In Sec.~\ref{subsec:outputinfinite}, we
provide a review of the limiting properties of Bayesian neural
networks with Gaussian and $\alpha$-stable ($0 < \alpha < 2$) weights
as the width of network layers goes to infinity. We present our main
results on neural network priors with finite width in
Sec.~\ref{subsec:outputfinite}.

\subsection{Related literature}\label{subsec:litereview}

We consider priors for unknown parameter functions $u$ defined over a
domain $\mathcal D$. One approach to infinite-dimensional Bayesian
problems is to first discretize all functions and use a
finite-dimensional perspective \cite{cotter2010,kaipio2007}. Methods building on this approach may suffer from a dependence on the
discretization, i.e., their performance may degrade as the
discretization is refined. Alternatively, one can analyze and develop
methods for Bayesian inverse problems in function space and then
discretize them \cite{stuart_2010}. This approach typically avoids
mesh dependence but is theoretically more challenging.
We follow the latter approach, i.e., construct priors in function
space.

Priors encode a priori available information about the unknown
function $u$. A common choice are Gaussian priors, for which a vast
number of theoretical and computational results are available
\cite{stuart_2010, RasWill2005gaussian, BuiGeorg2013}. A sample $u$ from a Gaussian prior can be constructed using, e.g., the
Karhunen-Lo\`eve expansion \cite{RasWill2005gaussian}. Let $\{\varphi_j, \rho_j^2 \}_{j=1}^{\infty}$ be a
set of orthonormal eigenfunctions and eigenvalues of a positive
definite, trace-class covariance operator $\mathcal C_0:\mathcal
U\to\mathcal U$, and let $\{\xi_j\}_{j=1}^{\infty}$ be an
i.i.d.\ sequence $\xi_j \sim \mathcal N (0, 1)$. Then
\begin{align}\label{eq:kl_expansion}
	u := m_0 + \sum_{j =1}^{\infty} \rho_j \xi_j \varphi_j
\end{align}
is distributed according to $\mathcal N (m_0, \mathcal C_0)$, where
$m_0 \in \mathcal U$ is the mean.
%which is a positive semidefinite
%operator of trace class.
Examples of covariance operators are Mat\'ern covariance operators and
fractional inverse powers of elliptic differential operators.  One can
show that samples from these Gaussian priors are almost surely
continuous \cite{stuart_2010}. However, this might not always be
desirable.  For example, in many practical problems, the unknown
function $u$ might have discontinuities. Thus, a prior
that allows for discontinuous samples might be preferable to a
Gaussian prior.

Several priors have been proposed to overcome the limitations of
Gaussian priors. For example, total variation priors based on the
total variation regularization \cite{Lassas_2004, gonzalez2017,
  Chambolle10anintroduction} can be defined for a fixed
discretization. However, these priors converge to Gaussian fields as
the mesh is refined, which is referred to as lack of discretization
invariance in \cite{Lassas_2004}. That is, total variation priors
depend on the discretization, which is undesirable from an
infinite-dimensional perspective. Non-Gaussian priors can also be based on generalizations of
the expansion \eqref{eq:kl_expansion} with non-Gaussian coefficients.
For instance, when
$\{\varphi_j\}_{j=1}^\infty$ is a wavelet basis of $\mathcal U$, Besov priors are obtained if
$\{ \xi_j \}_{j =1}^{\infty}$ are i.i.d.\ random variables with
probability density function $p_{\xi}(z) \propto \exp \left(− \frac 1
2 |z|^q \right)$ \cite{dashti2011besov}. This results in a
non-Gaussian discretization-invariant prior that is able to produce
discontinuous samples. However, choosing a proper wavelet basis for Besov priors is
challenging and the location of discontinuities depends on this
basis \cite{matti2009samuli,dashti2011besov}.

Due to properties related to their heavy tails,
$\alpha$-stable processes have recently been proposed as priors \cite{chada2019posterior,
  sullivan2017}. These can be realized by choosing for $\{ \xi_j \}_{j
  =1}^{\infty}$ in \eqref{eq:kl_expansion} i.i.d.\ $\alpha$-stable
distributions. The processes are well-defined in function space
\cite{sullivan2017} and due to their heavy tails, they can take on
extreme values and thus incorporate large jumps. Variations of $\alpha$-stable processes were proposed, e.g., the
Cauchy difference prior \cite{markkanen2016cauchy} and the Cauchy Markov
random field \cite{chada2021cauchy}. Such priors are difficult to sample in high dimensional problems \cite{markkanen2016cauchy}.

Priors based on neural networks have been considered
\cite{neal1995BayesianLF, Asim2020}. One ongoing research direction in
constructing neural network priors are generative methods, which can
tailor priors to a specific application and are then used in
inference problems \cite{Asim2020, ardizzone2019analyzing}. An alternative approach is constructing Bayesian neural networks,
i.e., to use neural networks to parameterize functions
\cite{neal1995BayesianLF} and assume that the weights in the neural network
are random variables. Here, we follow this latter approach and thus introduce Bayesian neural networks next.

\subsection{Neural network priors}\label{subsec:nnprior}

Motivated by their relation to Gaussian processes under certain
conditions \cite{neal1995BayesianLF, matthews2018gaussian}, neural
networks can be used to construct priors in function space. These
novel priors are defined as parameterizations given by a neural
network with random weights drawn from specific distributions.  The
network is defined on a spatial domain $\mathcal 
D \subseteq \mathbb R^{D_0}$, where $D_0=d$ is the dimension of the
input space. Let $\Psi(\bs w): \mathcal D \to \mathbb R$ denote the
neural network, where $\bs w$ summarizes all weights in the neural
network. The network (we only consider general fully connected
networks) consists of $L$ hidden layers and its
structure is
\begin{subequations}\label{eq:neural_nets}
\begin{align}
	\bs h^{(0)}(\bs x) &= \bs x, \\
	\bs h^{(l + 1)}(\bs x) &= \phi \left(\bs b^{(l)} + \bs V^{(l)} \bs h^{(l)}(\bs x) \right), \quad l = 0, \ldots, L-1, \\
	u (\bs x) &= \bs V^{(L)} \bs h^{(L)}(\bs x),
\end{align}
\end{subequations}
where $\bs h^{(l)}(\bs x)$ is the $l$-th hidden layer vector with
width $D_l$. The matrices $\bs V^{(l)} \in \mathbb R^{D_{l + 1}
  \times D_l}$, for $0\le l<L$, and $\bs V^{(L)} \in \mathbb R^{1 \times D_L}$ are the
weights of the $l$-th hidden layer and the output layer, respectively,
and $\bs b^{(l)} \in \mathbb R^{D_{l + 1} \times 1}$ is the bias of
the $l$-th hidden layer. We denote the output of the neural network as $u(\bs x)$ for
input $\bs x$. Note that, in general, $u$ can be vector-valued, but
for simplicity we consider only scalar-valued $u$ in this paper. The
function $\phi: \mathbb R \to \mathbb R$ denotes a nonlinear
activation function in the neural network, which acts on vectors
component-wise. We will mainly use the hyperbolic tangent
function $\tanh(\cdot)$ as activation in this paper. Since this
is a smooth function, the input-to-output mapping of
the neural network is smooth too. In particular, while samples
generated by these networks appear to be discontinuous, they are
continuous but have very sharp jumps that visually appear to be
discontinuities.

We consider priors parameterized using \eqref{eq:neural_nets} with
weights $\bs w$ drawn from certain distributions. Such neural
network priors are flexible; one can change the distributions on the
weights or the number and widths of the hidden layers to obtain
different priors. One choice of weight distribution
is Gaussian \cite{neal1995BayesianLF, Williams96computing}, but non-Gaussian
weights have also been studied \cite{neal1995BayesianLF,
  weiss06platt}. In this paper, we consider neural network priors with
weights drawn from $\alpha$-stable distributions ($0 < \alpha \leq 2$),
in particular, Cauchy and Gaussian distributions (i.e., $\alpha = 1$
or $\alpha=2$).

To illustrate the difference between neural network priors with
Cauchy and Gaussian weights, we show realizations from both neural
networks in one dimension in
Figure~\ref{fig:1dnnprior}. Here, we use the neural network structure
in \eqref{eq:neural_nets} with three hidden layers of widths $[80, 80,
  100]$. The distributions assigned to the weights are the standard
Cauchy and Gaussian distributions. We observe that the
realizations of the neural network prior with Cauchy weights
have large jumps while realizations with Gaussian weights tend to be
smooth. This is the same behavior as observed in two dimensions in Figure~\ref{fig:2dnnprior}.

\begin{figure}[tb]\centering
  \begin{tikzpicture}
    \begin{scope}[xshift=0cm]
    \begin{axis}[width=.54\columnwidth,xmin=-1,xmax=1,ymax=5.5, compat=1.13, legend pos=south west, legend style={nodes={scale=.85, transform shape}}, xlabel= Cauchy NN]
        \addplot[color=green!70!white,mark=none,thick] table [x=x,y=s2]{plot_paper/prior/1dpriorcauchy.txt};
        \addplot[color=violet!70!white,mark=none,thick] table [x=x,y=s3]{plot_paper/prior/1dpriorcauchy.txt};
        \addplot[color=red!70!white,mark=none,thick] table [x=x,y=s4]{plot_paper/prior/1dpriorcauchy.txt};
        \addplot[color=orange!70!white,mark=none,thick] table [x=x,y=s5]{plot_paper/prior/1dpriorcauchy.txt};
        \addplot[color=purple!70!white,mark=none,thick] table [x=x,y=s7]{plot_paper/prior/1dpriorcauchy.txt};
    \end{axis}
    \node[color=black] at (0.4,4.05cm) {a)};
	\end{scope}
	\begin{scope}[xshift=6.5cm]
    \begin{axis}[width=.54\columnwidth,xmin=-1,xmax=1,compat=1.13, legend pos=south west, legend style={nodes={scale=.85, transform shape}}, xlabel= Gaussian NN]
        \addplot[color=green!70!white,mark=none,thick] table [x=x,y=s2]{plot_paper/prior/1dpriorgaussian.txt};
        \addplot[color=violet!70!white,mark=none,thick] table [x=x,y=s3]{plot_paper/prior/1dpriorgaussian.txt};
        \addplot[color=red!70!white,mark=none,thick] table [x=x,y=s4]{plot_paper/prior/1dpriorgaussian.txt};
        \addplot[color=orange!70!white,mark=none,thick] table [x=x,y=s5]{plot_paper/prior/1dpriorgaussian.txt};
        \addplot[color=purple!70!white,mark=none,thick] table [x=x,y=s7]{plot_paper/prior/1dpriorgaussian.txt};
    \end{axis}
    \node[color=black] at (0.4,4.05cm) {b)};
	\end{scope}
\end{tikzpicture}
\caption{Realizations of the neural network with two different weight
  distributions on the interval $[-1, 1]$. Shown are realizations with
  Cauchy weights (a) and Gaussian weights (b).\label{fig:1dnnprior}}
\end{figure}
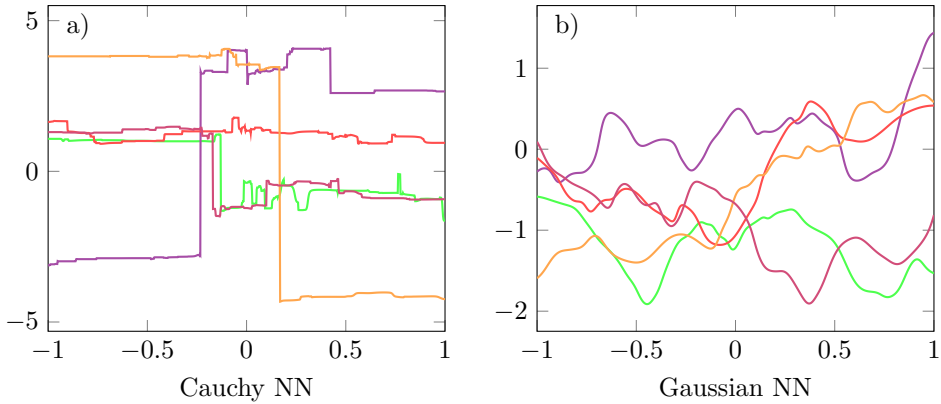

Note that realizations of neural network priors tend to exhibit
larger variations closer to the origin, i.e., the resulting processes
are non-stationary. To illustrate why this is the case, we consider
a simplified neural network with a one-dimensional input, one hidden
layer, and $V^{(0)}\in \mathbb R^{D_1}$ is a vector of all ones:
\begin{align*}
	u(x) = \sum_{i=1}^{D_1} V^{(1)}_i H \left( x - b^{(0)}_i \right),
\end{align*}
where the activation function $H$ is taken to be the Heaviside
function for illustration purposes. For fixed biases $b_i^{(0)}$,
the network output is a random walk with jumps occurring
at each $b_i^{(0)}$ with magnitude $V_i^{(1)}$. Thus, if the biases $b_i^{0}$ are
concentrated around zero (as in  a Cauchy or Gaussian
distribution), more jumps occur near the origin. If one wishes to
reduce this non-stationarity, $b_i^{(0)}$ could be taken to
be uniformly distributed in some zero-centered interval. However, this
will not guarantee stationarity for more general networks.

\subsection{$\alpha$-stable distributions}\label{subsec:alpha}

In the previous section, we reviewed neural networks and
how to consider their outputs as priors for Bayesian inverse problems. Before
moving to the theoretical discussion of neural network priors with
$\alpha$-stable weights ($0 < \alpha \leq 2$), we review scalar
$\alpha$-stable distributions.

First introduced by L\'evy, $\alpha$-stable distributions are a class 
of heavy-tailed distributions, which have a large probability of attaining 
extreme values. Given scalars $\alpha \in (0,2]$ (the stability parameter) and 
$\gamma > 0$ (the scaling parameter), the $\alpha$-stable distribution 
$\text{St}(\alpha, \gamma)$ may be defined by its characteristic function
\begin{align*}
	\Phi(t) = \mathbb E \left[\exp\left(it\text{St}(\alpha, \gamma)\right) \right] = e^{-\gamma^{\alpha} |t|^{\alpha}}.
\end{align*}
Additional parameters representing skew and location may also be introduced, 
though we consider these fixed at zero in this article. The absolute moment of an 
$\alpha$-stable random variable $\text{St}(\alpha, \gamma)$, i.e., 
$\mathbb E \left[|\text{St}(\alpha,\gamma)|^{\beta} \right]$, is finite when 
$\beta \in [0, \alpha)$  and infinite when $\beta \in [\alpha, \infty)$.
A key property of $\alpha$-stable distributions is that they are closed under 
independent linear combinations \cite{Borak2005} leading to generalizations of central limit theorems beyond the Gaussian regime \cite{gnedenko1954limit}.

Note that $\text{St}(2, \gamma)$ coincides with
$\mathcal N(0, 2\gamma^2)$, i.e., a mean-zero Gaussian random
variable. Another special case of an $\alpha$-stable distribution is
the Cauchy distribution for $\alpha = 1$. It has the probability density
function
\begin{align*} 
	f(x; \gamma) = \frac{1}{\pi \gamma} \left( \frac{\gamma^2}{x^2 + \gamma^2}\right),
\end{align*}
where the scaling parameter $\gamma$ specifies the half-width at
half-maximum (HWHM). In general, for $\alpha \neq 1,2$, the density of $\text{St}(\alpha,\gamma)$ 
may not be expressed analytically.

As discussed above, we construct different neural network priors by
using different $\alpha$-stable distributions ($0 < \alpha \leq 2$) on
the weights. Gaussian and Cauchy neural network priors are obtained
for $\alpha = 1, 2$, respectively. It is possible to define neural
network priors with weights drawn from general $\alpha$-stable
distributions, but in practice this is complicated by the lack
of analytic forms. In the following, we study theoretical results
for general $\alpha$-stable distributions but only show numerical
experiments from neural network priors with Gaussian and Cauchy
weights.

\subsection{Output properties of infinite-width neural network priors}\label{subsec:outputinfinite}
We have seen that realizations of neural networks with Cauchy weights
have jumps, while realizations of
networks with Gaussian weights tend to be smooth; see
Figures~\ref{fig:2dnnprior} and \ref{fig:1dnnprior}. In this section,
we review the convergence of neural networks with different weight
distributions in the limit of infinite width \cite{neal1995BayesianLF,
  matthews2018gaussian}. The results are stated for Gaussian ($\alpha
= 2$) and $\alpha$-stable distributions ($0 < \alpha < 2$).

\subsubsection{Gaussian weights}
Here we summarize theoretical results for Gaussian neural networks as
the widths of the hidden layers approach infinity. We start with the
case of one hidden layer, i.e., \eqref{eq:neural_nets} has the form
\begin{subequations}\label{eq:neural_net_1hidden}
\begin{alignat}{2}
	h^{(1)}_i(\bs x) &= \phi \left(b^{(0)}_i + \sum_{j=1}^{D_0} V^{(0)}_{ij} x_j \right), \\
	 u_{D_1}(\bs x) &= \frac{1}{\sqrt{D_1} } \sum_{i=1}^{D_1} V^{(1)}_{i} h^{(1)}_i(\bs x),
\end{alignat}
\end{subequations}
where $V_{ij}, b_{i}, h_i$ are the components of $\bs V, \bs b, \bs
h$, respectively. We denote the $j$-th component of $\bs x$ as
$x_j$. The normalization factor $\frac{1}{\sqrt{D_1}}$ could
also be absorbed into the variance of the Gaussian distribution.
Then, for Gaussian weights, the output converges to a Gaussian distribution
when the width of the hidden layer goes to
infinity, as summarized in the following theorem from
\cite{neal1995BayesianLF}.

\begin{theorem}\label{thm:gaussian_convergence}
  Consider the neural network \eqref{eq:neural_net_1hidden} with all
  weights following Gaussian distributions, i.e., $V_{ij}^{(0)},
  V_{i}^{(1)} \stackrel{iid}{\sim} \mathcal N(0, \sigma^2_V)$,
  $b_i^{(0)}\stackrel{iid}{\sim} \mathcal N(0, \sigma^2_b)$.
  Assume that the activation function $\phi(\cdot)$ is bounded and 
  fix the input $\bs x$. Then, as the width $D_1 \to \infty$, the
  output $u_{D_1}(\bs x)$ converges to a Gaussian distribution
  $\mathcal N(0, \sigma^2_O (\bs x))$, with
  \begin{align}
    \sigma^2_O (\bs x) = \sigma_V^2 \mathbb E \left[ \left( h_1^{(1)}(\bs x) \right)^2 \right].
  \end{align}
\end{theorem}
More generally, one can show that in the limit of infinite width, the
process $u_{D_1}$ is distributed as a centered Gaussian process with
kernel
$$K(\bs x, \bs x') = \sigma_V^2 \mathbb E \left[ h_1^{(1)}(\bs
  x) h_1^{(1)}(\bs x') \right].$$
If we further assume that $\phi(
h_1^{(1)}(\bs x))$ has a finite second moment, one can prove a similar
result for neural networks with multiple hidden layers by induction
\cite{matthews2018gaussian}.

\subsubsection{$\alpha$-stable weights ($0<\alpha < 2$)}
We now consider the case that all the weights of the neural network are
drawn from an $\alpha$-stable distribution with $\alpha \in (0,
2)$. As $D_L$ approaches infinity, we study the
convergence of the weighted sum
\begin{align}\label{eq:neu_lasthidden}
u_{D_L}(\bs x) = \frac{1}{{D_L}^{1/\alpha}} \sum_{i = 1}^{D_L} V_i^{(L)} h_{i}^{(L)}(\bs x),	
\end{align}
where $V_i^{(L)}$ satisfies the $\alpha$-stable distribution and
$h_{i}^{(L)}$ is the $i$-th node of the last hidden layer defined as
\begin{align}\label{eq:final_hidden_cauchy}
h_{i}^{(L)}(\bs x) = \phi \left(b^{(L - 1)}_i + \sum_{j=1}^{D_{L-1}}V^{(L - 1)}_{ij} h^{(L-1)}_j (\bs x) \right).
\end{align}
Note that for a fixed input $\bs x$, $\large\{h_j^{(L)}(\bs x), V_j^{(L)} \large \}_{j=1}^{D_{L-1}}$ are
pairwise independent and each $h_i^{(L)}(\bs x)$ follows the same
distribution. Thus,  we neglect the indices and simply use
$h(\bs x):=h_i^{(L)}(\bs x)$. By letting the width of each hidden layer tend to infinity, it is shown in
\cite{weiss06platt} that the final output also converges
to an $\alpha$-stable distribution. The precise result is stated next.
\begin{theorem}\label{thm:stable_convergence}
Assume that all the weights of \eqref{eq:neu_lasthidden} are i.i.d.,
and follow an $\alpha$-stable distribution with scale parameter
$\gamma$, where $\alpha \in (0, 2)$. Assume also that the activation
$\phi(\cdot)$ satisfies that $\mathbb E \left[ h(\bs x) \right]^{\alpha} <
\infty$, where $h_{i}^{(L)}(\bs x)$ is defined in
\eqref{eq:final_hidden_cauchy}. Then, as $D_L \to \infty$, the
output $u_{D_L}(\bs x)$ converges in distribution to an
$\alpha$-stable random variable $u(\bs x)$ with characteristic
function $\Phi_{u(\bs x)} (t) = \exp \left(-|\gamma t|^\alpha \mathbb
E \left[ h(\bs x) \right]^{\alpha} \right)$.
\end{theorem}
We note that the assumption $\mathbb E \left [h(\bs x) \right]^\alpha <
\infty$ in the theorem always holds for bounded activation functions
such as $\tanh(\cdot)$ and $\text{sgn}(\cdot)$.

\subsection{Outputs of neural network priors with finite width}\label{subsec:outputfinite}

Since infinite-width neural networks
are not practical, we next consider neural network
priors with finite width. In particular, we study the
distribution of the output's derivative at a fixed point $\bs x$. We
distinguish two cases, namely neural networks with heavy-tailed (e.g.,
Cauchy) weights and with finite moment (e.g., a Gaussian) weights.

We start  with the case of a single hidden layer. We first
assume that the input is one-dimensional and extend the result to
multi-dimensional input later. The network
\eqref{eq:neural_nets} becomes
\begin{align}\label{eq:neu_onediminput}
u(x) = \sum_{i = 1}^{D_1} V_i^{(1)} \phi \left( b^{(0)}_i + V^{(0)}_{i} x  \right).
\end{align}
Assuming $\phi$ is differentiable, the gradient of
$u(x)$ with respect to $x$ is 
\begin{align}\label{eq:derivative_}
	u'(x) &= \sum_{i = 1}^{D_1} V_i^{(1)} V_i^{(0)}\phi' ( b_i^{(0)} + V_i^{(0)} x).
\end{align}
We now show that the distribution of the derivative
\eqref{eq:derivative_} is heavy-tailed if the weights from the input
to the hidden layer follow a heavy-tailed distribution. A
distribution is referred to as heavy-tailed if its tail is heavier
than an exponential distribution. The formal definition is given next.
\begin{definition}\label{eq:heavy-tailed}
  A random variable $X$ with cumulative distribution function $F_X(x)$ is said to be heavy-tailed if
 	\begin{align*}
 	\int_{-\infty}^{\infty} e^{t|x|} \, \dee F_X(x) = \infty \quad \text{for all} \ t > 0.
 	\end{align*}
\end{definition}
To motivate our interest in \eqref{eq:derivative_},
consider the Taylor expansion of $u$ at a point $x$,
\begin{align*}
u(x + \delta) - u(x) = u'(x) \delta + o(|\delta|),	
\end{align*}
where $|\delta|$ is small. When, for fixed $x$, $u'(x)$ follows a heavy-tailed distribution, the
difference $|u(x + \delta) - u(x)|$ is very large with a
non-negligible probability, resulting in a large jump in $u$.
To study when \eqref{eq:derivative_} is heavy-tailed, we
first consider
%the component
$ V_i^{(0)} \phi' (b_i^{(0)} + V_i^{(0)} x)$. To simplify notation, we
neglect indices and introduce the random variable
\begin{align}\label{eq:gen_deri}
 G := V \phi'(B + V x),
\end{align}
which depends
on the i.i.d.\ heavy-tailed and general
symmetric random variables $V$ and $B$, respectively. We next show that the
distribution of $G$ is heavy-tailed for appropriate activation
functions. Recall that for two measurable spaces $X, Y$ with
$\mu$ being a measure on $X$ and a measurable function $f: X \to Y$,
the induced measure $\nu$ on $Y$ defined by $\nu(A) =
\mu(f^{-1}(A))$, for any measurable set $A$, satisfies
\begin{align}\label{eq:measure-transport}
  \mathbb E_{x  \sim \mu } (f(x))	 = \int_X f(x)\,\mu(\dee x)= \int_Y  y\,\nu(\dee y) = \mathbb E_{y \sim \nu} y.
\end{align}
The following theorem states the main results.
\begin{theorem}\label{thm:grad_heavy_tailed}
	Assume that in \eqref{eq:gen_deri}, $V$ follows a 
	    symmetric heavy-tailed distribution and $B$ is
        a symmetric random variable. Furthermore, assume that
        $\phi(\cdot)$ is differentiable and its derivative is bounded away from zero,
        i.e., $|\phi'(\cdot)| \geq c > 0$. Then the
        distribution of $G$ is heavy-tailed.
\end{theorem}

\begin{proof}
We denote the joint cumulative distribution function of $V$ and $B$ as
$F_{V, B}(\cdot, \cdot)$, and the cumulative distribution functions of
$V$ and $B$ as $F_V(\cdot)$ and $F_B(\cdot)$, respectively.  For $t >
0$, using \eqref{eq:measure-transport}, we have
\begin{align*}
\int_{\mathbb{R}} e^{t|g|} \, \dee F_G(g) =  \mathbb E_{V, B} \left [ e^{t|V \phi'(B + V x)|} \right] &= \int_{\mathbb R} \int_{\mathbb R} e^{t|v \phi'(b + v x)|} \, \dee F_{V, B}(v, b) \\
	&= \int_{\mathbb R} \int_{\mathbb R} e^{t|v \phi'(b + v x)|} \, \dee F_V(v)\,\dee F_B(b),
\end{align*}
since $V$ and $B$ are independent. The
boundedness of $\phi$ from below implies
\begin{align*}
\int_{\mathbb{R}} e^{t|g|} \,\dee F_G(g) &\geq \int_{\mathbb R} \int_{\mathbb R} e^{ct |v |} \,\dee F_V(v)\,\dee F_B(b) \geq \int_{\mathbb R} e^{tc|v |} \,\dee F_V(v) = \infty,
\end{align*}
where we used $\int_{\mathbb R} \,\dee F_B(b) = 1$ and that
$V$ is heavy-tailed. Thus, $G$ is heavy-tailed.
\end{proof}

We note that the
assumption of the derivative bounded from below is satisfied for most activation
functions including leaky-ReLU and SeLU. Furthermore, we can apply
this theorem to more activation functions, e.g., $\tanh(\cdot)$, ReLU,
by using
%\begin{align*}
$\tilde \phi (x) = \phi (x) + \varepsilon x$,
%\end{align*}
with small $\varepsilon > 0 $. Having established that each $ V_i^{(0)} \phi' (b_i^{(0)} + V_i^{(0)}
x)$ in \eqref{eq:derivative_} is heavy-tailed, we next show that
$u'(x)$ is also heavy-tailed. We first
formulate a basic lemma, whose proof can be found in the appendix.
\begin{lemma}\label{lemma:heavy_product_sum}
	If $X$ is a heavy-tailed random variable and $Y$ is an independent
        symmetric continuous random variable, then we have the
        following properties:
	\begin{enumerate}[(i)]
	\item The product of $XY$, is also a heavy-tailed random variable.
	\item If we further assume that $Y$ is also heavy-tailed, then
          the sum of $X + Y$, is also a heavy-tailed random variable.
	\end{enumerate}
\end{lemma}

We now show that all components of the gradient \eqref{eq:derivative_} are heavy-tailed
under the assumptions in Theorem~\ref{thm:grad_heavy_tailed}. Note
that we only require the weights in the last hidden layer to the
output to be symmetric continuous random variables. While the result
above is for one-dimensional input $x$, the generalization to
multi-dimensional input $\bs x$ by considering the partial derivatives
is straightforward. A neural network with multi-dimensional input $\bs
x$ and one hidden layer is given as
\begin{align}\label{eq:neu_multidiminput}
u(\bs x) = \sum_{i = 1}^{D_1} V_i^{(1)} \phi \left(b^{(0)}_i + \sum_{j=1}^{D_0} V_{ij}^{(0)} x_j \right).
\end{align}
The next theorem extends
Theorem~\ref{thm:grad_heavy_tailed} to networks of the form
\eqref{eq:neu_multidiminput}. 

\begin{theorem}\label{thm:gradone_heavy_tailed}
 Assume $V_{ij}^{(0)}, b^{(0)}_i$ follow i.i.d.\ heavy-tailed and
 symmetric distributions, respectively, and $V_i^{(1)}$ are
 i.i.d.\ symmetric continuous random variables. Further assume that
 $\phi(\cdot)$ is differentiable and its derivative is 
        bounded away from zero. Then the partial derivatives of
 \eqref{eq:neu_multidiminput},
 \begin{align}\label{eq:gr_partial}
   \partial_{x_k} u(\bs x) &= \sum_{i = 1}^{D_1} V_i^{(1)} V_{ik}^{(0)}	\phi' \left(b_i^{(0)} + \sum_{j=1}^{D_0} V_{ij}^{(0)} x_j \right)
\end{align}
are also heavy-tailed.
\end{theorem}
The proof follows by using that each component is
heavy-tailed. Next, we prove our main theorem for neural networks with
multiple hidden layers and weights from heavy-tailed
and finite-moments distributions.
\begin{theorem}\label{thm:gradmulti_heavy_tailed}
	Assume the neural network \eqref{eq:neural_nets} and that $\bs
        b^{(l)}$ are symmetric random variables for $l = 0, \ldots,
        L-1$. Furthermore, assume that $\phi(\cdot)$ is differentiable, 
        and its derivative is bounded and bounded away from zero. Then, for a fixed $\bs x$, the distribution of
        $\partial_{x_k} u(\bs x)$ satisfies:
	\begin{enumerate}[(i)]
	\item If the weights in $\bs V^{(l)}$ are i.i.d.\ symmetric
          heavy-tailed and $\bs V^{(L)}$ are i.i.d.\ symmetric continuous random
          variables, then all $\partial_{x_k}u(\bs x)$ are heavy-tailed.
	\item If $\bs V^{(l)}$ are i.i.d.\ symmetric random variables
          of finite $k$-th order moment, then all $\partial_{x_k}u(\bs
          x)$
          have finite moment of $k$-th order.
	\end{enumerate}
\end{theorem}

\begin{proof}
We use induction for the proof. The partial derivative
%$\partial_{x_k}u(\bs x)$
is
\begin{align}\label{eq:grmulti_partial}
	\partial_{x_k} u(\bs x) &= \sum_{i = 1}^{D_L} V_i^{(L)}	\phi' \bigg( b_i^{(L-1)} + \sum_{j=1}^{D_{L-1}} V_{ij}^{(L-1)} h_j^{(L-1)}(\bs x) \bigg) \bigg( \sum_{j=1}^{D_{L-1}} V_{ij}^{(L-1)} \partial_{x_k} h_j^{(L-1)}(\bs x) \bigg),
\end{align}
where $\partial_{x_k} h_j^{(L-1)}(\bs x)$ is the partial
derivative of the $j$-th component of the $(L-1)$-th hidden layer.

We consider first case (i). In Theorem~\ref{thm:gradone_heavy_tailed},
for one hidden layer, we have shown that the distribution of the
partial derivative is heavy-tailed. We assume that this
argument holds for neural networks with $L-1$ hidden layers. It is
easy to see that
\begin{align}\label{eq:gramulti_part}
\sum_{j=1}^{D_{L-1}} V_{ij}^{(L-1)} \partial_{x_k} h_j^{(L-1)}(\bs x)
\end{align}
is also heavy-tailed using Lemma~\ref{lemma:heavy_product_sum} and
since each component is the partial derivative of the output of a
neural network of $L-1$ hidden layers and . Since $\phi'(\cdot)$ is
bounded away from zero, $V_i^{(L)} \phi' \left(b_i^{(L-1)} +
\sum_{j=1}^{D_{L-1}} V_{ij}^{(L-1)} h_j^{(L-1)}(\bs x) \right)$ is
heavy-tailed following the proof in
Theorem~\ref{thm:grad_heavy_tailed}, and thus that
\eqref{eq:grmulti_partial} is also heavy-tailed by
Lemma~\ref{lemma:heavy_product_sum}.

For case (ii), we start with a one-hidden-layer neural network given by
\begin{align}\label{eq:neu_onediminputgau}
u(\bs x) =  \sum_{i = 1}^{D_1} V_i^{(1)} \phi \left(b^{(0)}_i + \sum_{j=1}^{D_0} V_{ij}^{(0)} x_j \right),	
\end{align}
with corresponding partial derivative
\begin{align*}
\partial_{x_k} u(\bs x)= \sum_{i = 1}^{D_1} V_i^{(1)} V_{ij}^{(0)}	\phi' \left(b_i^{(0)} + \sum_{j=1}^{D_0} V_{ij}^{(0)} x_j \right).
\end{align*}
Note that $|\phi'(\cdot)|$ is bounded by the assumption. Therefore,
the partial derivative of the output has finite $k$-th moment since
$V_i^{(1)}, V_{ij}^{(0)}$ are i.i.d.\ random variables of finite
$k$-th moment.
We now assume that the result holds for any neural network with $L-1$
hidden layers. Then \eqref{eq:gramulti_part} has
finite $k$-th moment since both $V_{ij}^{(L-1)}$ and $\partial_{x_k} h_j^{(L-1)}(\bs x)$ have finite $k$-th moment and
are independent of each other. This implies that the partial
derivative in \eqref{eq:grmulti_partial} also has finite $k$-th
moment. Thus, the result follows by induction.
\end{proof}

Again, note that the results in this section hold for heavy-tailed
weights, which includes $\alpha$-stable weights. From
Theorem~\ref{thm:gradmulti_heavy_tailed}, one can see that the
derivative of the neural network output with Gaussian weights has
finite moments, which implies smooth outputs in practice. In contrast,
the derivative of the neural network output is heavy-tailed if all the
weights before the last hidden layer are $\alpha$-stable distributed,
$0 < \alpha <2$. This implies that the derivative of the output can
have extreme values with non-negligible probability and thus the
corresponding neural network outputs can have
large jumps, emulating discontinuities even when the activation
function is smooth. Furthermore, we only assume the weights before the last hidden layer to be
$\alpha$-stable. Biases can follow any symmetric random
distribution. For the numerical experiments in
Sec.~\ref{sec:numerical}, we study one type of such neural network
priors, which has Cauchy random variables as the weights except for
the last hidden layer, which has Gaussian weights. We refer to this
setup as the Cauchy-Gaussian prior.

% **********************************
% **********************************

\section{Bayesian inference with neural network priors}\label{sec:BNN}

In this section, we study Bayesian inverse problems in function space
with neural network priors. We first use observation data to define the posterior on the
weights. Several approaches to
probe the posterior distribution are then discussed, including
maximum a posterior (MAP) estimation via optimization, and Markov
chain Monte Carlo (MCMC) sampling.

In the Bayesian approach, one treats the inverse problem as
statistical inference of the function $u$ in a function space
$\mathcal U$. This amounts to finding the posterior distribution of
$u$ that reflects the prior information and the observations. The
forward map is the parameter-to-observable operator
$\mathbfcal{F}: \mathcal U \to \mathbb R^{N_{\text{obs}}}$ that maps
the parameter field $u$ to observations %$\bs y_{\text{obs}} \in
in $\mathbb R^{N_{\text{obs}}}$, where $N_{\text{obs}}$ is the
dimension of the observables. We assume
additive Gaussian errors $\bs \varepsilon \sim \mathcal N(0, \eta^2
I_{N_{\text{obs}}})$, so that the observations are
\begin{align}\label{eq:fwd_inf}
	\bs y_{\text{obs}} = \mathbfcal F(u) + \bs \varepsilon.
\end{align}
Hence, the likelihood for given $u$ is
\begin{align*}
	p(\bs y_{\text{obs}}|u) \propto \exp \left(-\frac{1}{2\eta^2} \left\| \mathbfcal F(u) - \bs y_{\text{obs}} \right\|_2^2 \right).
\end{align*}
Given a prior measure $\mu^0$, the
posterior measure $\mu^y$ of $u$ is defined via Bayes' rule
\begin{align}\label{eq:bayesinf}
\frac{\dee \mu^y}{\dee \mu^0} = \frac 1 Z p(\bs y_{\text{obs}}|u),
\end{align}
where the left-hand side is the Radon--Nikodym derivative of the
posterior measure $\mu^y$ with respect to the prior measure $\mu^0$
and $Z = \int_{\mathcal U} p(\bs y_{\text{obs}}|u)\,\dee \mu_0$ is the
normalization constant.

Note that in infinite dimensions, Bayes' formula cannot
be written in terms of probability density functions since there is no
Lebesgue measure against which to define the densities of the prior
and posterior distributions. Thus, finite-dimensional
approximations of the prior and posterior distributions are
proposed in \cite{BuiGeorg2013,cotter2010}. In particular, we assume
that the prior distribution is approximated by a finite-dimensional
measure $\mu^{0, h}$ which is absolutely continuous with respect to
the Lebesgue measure $\lambda$; the resulting posterior $\mu^{y,h}$ is
then also absolutely continuous with respect to $\lambda$. If we
define $p_{\text{post}}(u | \bs y_{\text{obs}})$ and
$p_{\text{prior}}(u)$ as the Lebesgue densities of $\mu^{0, h}$ and
$\mu^{y, h}$, respectively, we have
\begin{align}\label{eq:bayesinfappro}
	p_{\text{post}}(u | \bs y_{\text{obs}}) = \frac{\dee \mu^{y, h}}{\dee \lambda} = \frac{\dee \mu^{y, h}}{\dee \mu^{0, h}} \frac{\dee \mu^{0, h}}{\dee \lambda} \propto p(\bs y_{\text{obs}}|u) p_{\text{prior}}(u).
\end{align}

In this paper, we estimate $u$ defined by the  the finite-dimensional
network parameterization. The output of the neural network is a
function of the weights, which we denote by $u =
\Psi (\bs w)$, where $\bs w$ represents the weights inside the neural
network. This connects the weights $\bs w$ with the
observations $\bs y_{\text{obs}}$ as
\begin{align*} 
\bs y_{\text{obs}} = \mathbfcal F \left(\Psi \left(\bs w \right) \right) + \bs \varepsilon.	
\end{align*}
This allows one to infer $u$ by learning the
posterior distribution on the weights of the neural network
instead. Again, the posterior measure $\mu_{w}^y$ on the weights can
be computed by Bayes' rule as
\begin{align}\label{eq:Bayesrulenn}
	\frac{\dee \mu_{w}^y}{\dee \mu_{w}^0} \propto  p(\bs y_{\text{obs}}| \bs w),
\end{align}
where $p(\bs y_{\text{obs}}| \bs w)$ is the likelihood for given
weights $\bs w$.  The finite-dimensional equation \eqref{eq:Bayesrulenn}
can be expressed in terms of densities.
For the prior density on $\bs w$, we use the
product of one-dimensional densities on each component.

\subsection{Posterior approximation based on minimizers}

Here, we give an overview of methods that approximate the
posterior distribution using optimization. Minimization-based point estimation only provides limited uncertainty
information, but it is computationally much cheaper
than sampling, which we discuss in Sec.~\ref{subsubsec:MCMC_pcn}.

\subsubsection{MAP estimation}
A widely-used estimation of the posterior \eqref{eq:Bayesrulenn}
is maximum a posteriori (MAP) estimation. MAP estimates $\bs
w_{\text{map}}$ are obtained by maximizing the posterior distribution
$p_{\text{post}}(\bs w|\bs y_{\text{obs}})$:
\begin{align*}
  \bs w_{\text{map}} = \argmax_{\bs w} p_{\text{post}}(\bs w | \bs y_{\text{obs}}).
\end{align*}
This is equivalent to
\begin{align}\label{eq:mapobjective}
  \bs w_{\text{map}} = \argmin_{\bs w} J(\bs w) := \frac{1}{2 \eta^2} \left\| \mathbfcal F \left(\Psi (\bs w) \right) - \bs y_{\text{obs}} \right\|_2^2 + R(\bs w),
\end{align}
where $R(\bs w) := -\log \left(p_{\text{prior}}( \bs w) \right)$ is the regularization
term. We assume that the prior on $\bs w$ is an
$\alpha$-stable distribution. Note that \eqref{eq:mapobjective} is a non-convex optimization problem
due to the presence of the neural network and due to a possibly
non-linear map $\mathbfcal F$. It is thus not guaranteed (and even
unlikely) that we find the global minimum numerically. Instead, we
likely find a local minimizer, which we denote by $\bs w_{\text{loc}}$.
Which local minimizer is found may depend on the initialization of the
optimization algorithm and the algorithm itself.
We will use local minima
to construct different approximations to the posterior distribution as
discussed next.

\subsubsection{Laplace approximation}\label{sec:laplace}
Based on a local minimum $\bs w_{\text{loc}}$, the Laplace
approximation is defined as the Gaussian distribution
\begin{align}\label{eq:laplaceapp}
  \mathcal N \left(\bs w_{\text{loc}},  \nabla^2 J (\bs w_{\text{loc}})^{-1} \right),
\end{align}
where $\nabla^2 J (\bs w_{\text{loc}})$ denotes the Hessian of the
objective function \eqref{eq:mapobjective} at the local minimum $\bs
w_{\text{loc}}$, assuming that it is positive definite.
That is, the Laplace approximation replaces the true posterior using a
Gaussian distribution centered at the $\bs w_{\text{loc}}$ with
`local' covariance information of the posterior distribution
\cite{schillings20philipp, BuiGeorg2013}. The Laplace approximation
coincides with the true posterior if the mapping $\bs w \mapsto
\mathbfcal{F}(\Psi(\bs w))$ is linear and the prior and observation error
distributions are Gaussian. The approximation might be insufficient
when the posterior distribution is multimodal or heavy-tailed. As reported in \cite{ImmerKB21}, we observed numerically that
the Laplace approximation can
lead to substantial overestimation of the variance, and hence we do not
consider it further in this
paper.

\subsubsection{Ensemble method}\label{sec:ensemble}
A heuristic approach to estimate uncertainty is based on the set of
local minima computed with different initializations. This approach is
known as the ensemble method or model averaging method
\cite{rahaman2020uncertainty}. Local minimizers are averaged to
approximate the posterior mean and the variation of minimizers is
used to estimate uncertainty. Although not well understood
theoretically, the ensemble method can provide useful and cheap
uncertainty quantification compared to more sophisticated approaches
\cite{Laksh2017,zhou2002}. We will show results obtained with this
heuristic approach as part of our numerical results.

\subsubsection{Last-layer Gaussian regression}\label{sec:regression}
This approach starts with finding a minimizer for networks where the
last layer is Gaussian. The output of the neural network \eqref{eq:neural_nets} can be viewed as a
linear combination of the nodes in the last hidden layer. For linear
parameter-to-observable maps $\mathbfcal{F}(\cdot)$ and Gaussian
weights between the last hidden layer to the output,
\eqref{eq:neural_nets} reduces to a Gaussian regression problem upon
fixing the weights in all previous layers. This is related to the
majority voting algorithm \cite{Laksh2017} except that here we use it
to quantify uncertainty. The output function can be represented as
\begin{align*}
u(\bs x) = \sum_{j = 1}^{D_L} \nu_j f_i(\bs x) = \bs f(\bs x)^T \bs \nu,	
\end{align*}
where $D_L$ denotes the width of the last hidden layer, i.e., the
number of base functions, $\bs f(\bs x) = [f_1(\bs x), \ldots,
  f_{D_L}(\bs x)]^T$ denotes the vector containing the base functions,
and $\bs \nu = [\nu_1, \ldots, \nu_{D_L}]^T$ is the weight vector,
which follows a multivariate Gaussian prior $\mathcal N(0,
\frac{1}{D_L} I_{D_L})$. Thus, the forward model is
$
%\begin{align*}
	\bs y_{\text{obs}} = \mathbfcal F\bs f(\cdot)^T \bs \nu + \bs \varepsilon = \bs F^T \bs \nu + \bs \varepsilon,
$
%\end{align*}
where $\bs F = [\mathbfcal F f_1 , \ldots, \mathbfcal F f_{D_L} ]^T$
and $\mathbfcal F(\cdot)$ is the linear parameter-to-observable
mapping. By Bayes' formula, the posterior distribution of $\bs \nu$ follows a
multivariate Gaussian distribution $\mathcal N(\bs \nu_n, \Sigma_n)$
with
\begin{align*}
\bs \nu_n &= \frac{1}{\eta^2}\Sigma_n \bs F^* \bs y_{\text{obs}}, \quad \Sigma_n^{-1} = \frac{1}{\eta^2 }\bs F^* \bs F + I_{D_L},
\end{align*}
where $\bs F^*$ is the conjugate transpose of $\bs F$.

Note that this Gaussian distribution is $D_L$-dimensional, which is
much smaller than the number of weights.  The covariance $\Sigma_n$
can thus be computed and factored easily to generate samples. The base
functions $f_i$, which are found through optimization, typically
contain local structural information such as jumps, and so the $f_i$
essentially provide a basis adapted to the inverse problem \cite{Helmut2017Optimal}. The
assumption of a linear parameter-to-observable mapping $\mathbfcal{F}$
is restrictive. However, when $\mathbfcal{F}$ is nonlinear this
approach may still provide a form of dimension reduction for use with
other sampling/approximation methods.

\subsection{Posterior approximation using MCMC sampling}\label{subsubsec:MCMC_pcn}

Sampling methods aim at full exploration of the posterior
distribution.
MCMC explores the posterior distribution $p(\bs w|\bs
y_{\text{obs}})$ by constructing a Markov chain which targets the
posterior in stationarity; states of this chain then form a sequence
of correlated samples following convergence. One starts with an
initial state $\bs w^{(0)}$ and proposes another state based on a
Markov transition kernel. This proposed state is then accepted or
rejected based on an acceptance criterion. Iterating this procedure is
the basis for generating a sequence of MCMC samples.

Ideally one desires the correlation between states of the chain to be
minimal in order to reduce the cost of computing accurate uncertainty
estimates: generating each proposal and/or acceptance probability
typically involves evaluation of the likelihood function and possibly
its derivatives. Many classical MCMC algorithms suffer from increasing
correlations between samples when the dimension of the parameter space
is increased, which is an issue in settings such as ours where the
number of weights to be inferred is high \cite{hairer2014spectral}. A
simple MCMC method which does not suffer from this dimensional
dependence, in the case of a Gaussian prior, is the preconditioned
Crank-Nicolson (pCN) method \cite{cotter2013mcmc}. In this method the
proposed states encode the prior information via a rescaling and
random perturbation of the current state, and the acceptance
probability requires only evaluation of the likelihood; see
\cite{cotter2013mcmc} for a full statement of the algorithm. In this
algorithm one must choose a scalar parameter $\beta$ that controls the
size of the perturbation in proposed states: if the perturbations are
too large, proposed states will be unlikely to be accepted, but if they
are too small the states of the chain will be highly correlated, and
so a balance must be achieved. In practice this parameter may be
adapted on-the-fly to ensure a certain proportion of proposals are
accepted. Variants of pCN are available that use gradient and
curvature information of the likelihood in order to generate more
feasible proposals with larger perturbations, reducing correlation
between samples in exchange for increased computational cost per
sample; see \cite{beskos2017geometric} for a review. In this paper we
will use only the pCN method.

We remark that our prior is in general non-Gaussian, however we may
still utilize the above methods via a reparameterization typically
known as non-centering \cite{chen2019dimensionrobust}: we rewrite our prior as a nonlinear
transformation of a Gaussian distribution. Since our prior is assumed
to be an independent product of one-dimensional distributions this
mapping may be found using the inverse CDF method. For example, in the
case of a standard Cauchy prior we define the mapping
\begin{align*}
	\Lambda (p) = \tan \left[ \pi \left(\psi_G \left( p \right) - \frac 1 2 \right) \right],
\end{align*}
where $\psi_G$ is the standard Gaussian cumulative density function:
if $p \sim N(0,1)$, $\Lambda(p)$ is then a standard Cauchy sample. We
therefore compose our likelihood function with $\Lambda$ acting
componentwise, and apply the pCN method assuming independent $N(0,1)$
priors on each weight. Transforming the resulting MCMC samples with
the function $\Lambda$ componentwise then provides samples from our
desired posterior.

\section{Numerical experiments}\label{sec:numerical}
In this section, we study the behavior of
neural network priors for Bayesian inverse problems. Our goal is to
compare inversion results obtained with neural network priors with
Cauchy and Gaussian weight distributions. For that purpose, we use
deconvolution problems in one and
two dimensions. All numerical experiments are implemented using
the PyTorch 1.9.0 environment.

\subsection{One-dimensional deconvolution}

\begin{problem}\label{ex1:fun_appro} 
We first consider a deconvolution problem on the interval
$\mathcal D = [-1, 1]$ with forward model
\begin{align*}
  \bs y_{\text{obs}} = A (u)  + \bs \varepsilon, \quad \bs
  \varepsilon \sim \mathcal N(0, \eta^2 I_{N_{\text{obs}}}),
\end{align*}
where $u$ is the unknown, potentially discontinuous parameter to be
recovered. The forward operator is defined as $A = P\circ B$, where $B:
L^{\infty}(\mathcal D) \to \mathcal C^{\infty}$ denotes the blurring
operator defined by the convolution with a mean-zero Gaussian kernel
with variance $0.03^2$, and $P: \mathcal C^{\infty} \to \mathbb
R^{N_{\text{obs}}}$ is the evaluation operator at $N_{\text{obs}}=50$
uniformly distributed points. The variance of the errors $\bs
\varepsilon$ is $\eta^2 = 0.05^2$.
We discretize the forward operator using
a uniform mesh with $128$ points.

\end{problem}

The true model and the synthetic observations are shown in
Figure~\ref{fig:1ddatagpr}(a).  As a reference, we show the
reconstruction obtained with a Gaussian process regression in
Figure~\ref{fig:1ddatagpr}(b). Here, we used a mean-zero Gaussian
process prior with the Mat\'ern covariance operator $0.25 (I -
10\Delta)^{-2}$ with homogeneous Neumann boundary conditions.  Here
and in the remainder of this section, we show an uncertainty region
corresponding to the $95 \%$ credible interval. For Gaussians, this
corresponds to $\hat \mu \pm 1.96 \hat \sigma$, where $\hat \mu$ and
$\hat \sigma$ denote the sample mean and sample pointwise standard
deviation, respectively.  For the following neural network priors, we
use a network with 3 hidden layers with widths $[50, 50, 100]$.

\begin{figure}[tb]\centering
  \begin{tikzpicture}
    \begin{scope}[xshift=0cm]
    \begin{axis}[width=.54\columnwidth,xmin=-1,xmax=1,ymin=-1.8,ymax=0.8,compat=1.13, ytick ={-1.5, -1, -0.5, 0, 0.5}, legend pos=south west, legend cell align={left}, legend style={nodes={scale=.85, transform shape}}, xlabel = ]
      \addplot[mark=none, black, dotted, color=blue!20!red, thick] table [x=x,y=U]{plot_paper/ex1/data/1dtruth.txt};
      \addlegendentry{Truth $u^{\dagger}$}
      \addplot[color=blue!60!green,mark=none,thick] table [x=x,y=blur]{plot_paper/ex1/data/1dtruth.txt};
      \addlegendentry{Blurred model}
      \addplot[only marks, mark=x, color=black, mark size=1] table [x=x,y=yobs]{plot_paper/ex1/data/1dobs.txt};
      \addlegendentry{Observations}
    \end{axis}
    \node[color=black] at (0.4,4.1cm) {a)};
	\end{scope}
	\begin{scope}[xshift=6cm]
    \begin{axis}[width=.54\columnwidth,xmin=-1,xmax=1,ymin=-1.8,ymax=0.8,compat=1.13, xtick ={-0.5, 0, 0.5, 1}, yticklabels=\empty, legend pos=south west, legend cell align={left}, legend style={nodes={scale=.85, transform shape}}, xlabel =]
      \addplot[mark=none, black, dotted, color=blue!20!red, thick] table [x=x,y=U]{plot_paper/ex1/data/1dtruth.txt};
      \addlegendentry{Truth}
      \addplot[color=blue!40!orange,mark=none,thick] table [x=x,y=mean]{plot_paper/ex1/data/1dgpr.txt};
      \addlegendentry{Mean}
      \errorband[blue, opacity=0.2]{plot_paper/ex1/data/1dgpr.txt}{x}{mean}{std}
      \addlegendentry{$\pm 1.96$ Std. dev.}
    \end{axis}
    \node[color=black] at (0.4,4.1cm) {b)};
	\end{scope}
\end{tikzpicture}
\caption{Setup for Problem~\ref{ex1:fun_appro}. Shown in (a) are the
  truth model and synthetic observations. As reference, shown in (b)
  is the result of Gaussian process regression with the covariance
  operator $0.25(I - 10 \Delta)^{-2}$.\label{fig:1ddatagpr}}
\end{figure}
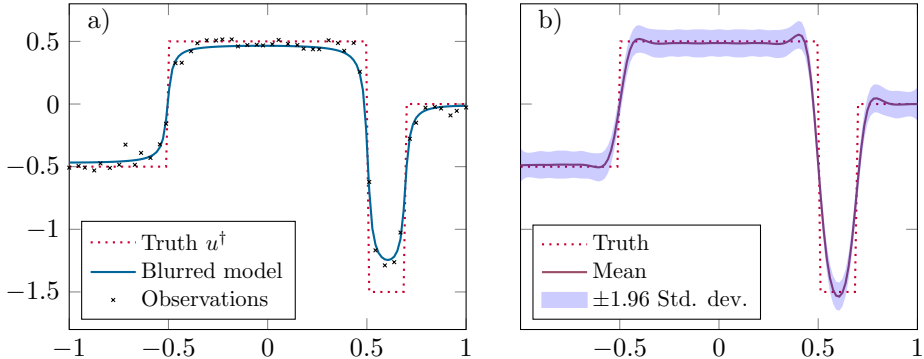

\subsubsection{Optimization-based methods}
For the optimization-based methods, the objective function to be minimized is
\begin{align}\label{eq:optobj}
J(\bs w) = \frac{1}{2\eta^2} \left\| A \left(\Psi (\bs w) \right) - \bs y_{\text{obs}} \right\|_2^2 + R(\bs w).
\end{align}
For Gaussian weights, the regularization $R(\bs w)$ is
$R_{\text{G}}(\bs w)$ and for Cauchy weights it is
$R_{\text{C}}(\bs w)$, defined as
\begin{align}\label{eq:regularization}
  R_{\text{G}}(\bs w) := \frac 1 2 \sum_i w_i^2, \quad R_{\text{C}}(\bs w) := \sum_i \log(1 + w_i^2),
\end{align}
where the summation is over all weights $\bs w$. We compare reconstructions for three regularizations, namely:
\begin{enumerate}[(i)]
\item Fully Gaussian weights,
\item Cauchy-Gaussian weights, i.e., all weights except those in the
  last layer follow a Cauchy distribution. The weights in the last
  layer are Gaussian,
\item Fully Cauchy weights.
\end{enumerate}
For each optimization, we first use $500$ iterations of the Adaptive Moment
Estimation (Adam) method, which adaptively changes the
stepsize. Following that, we use $1500$ iterations using the
Limited-memory Broyden–Fletcher–Goldfarb–Shannon (L-BFGS) algorithm
\cite{NocedalWright06} for a faster convergence to a local
minimum. Examples of reconstructions of the parameter function $u$
for the different regularizations are shown in
Figure~\ref{fig:1drecons}. We observe that while optimizations from
different initializations typically result in different weights, the
differences in the outputs $u$ are small. Moreover, the
reconstructions using Gaussian regularizations are smooth while the
reconstructions with Cauchy regularizations capture the
discontinuities in the parameter function better, although overall the
differences are rather small.

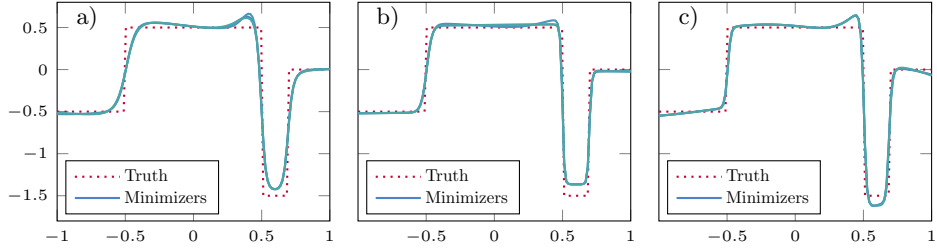
\begin{figure}[tb]\centering
  \begin{tikzpicture}
    \begin{scope}[xshift=0cm]
    \begin{axis}[width=.41\columnwidth,xmin=-1,xmax=1,ymin=-1.8,ymax=0.8,compat=1.13, ytick ={-1.5, -1, -0.5, 0, 0.5}, ticklabel style = {font=\tiny}, legend pos=south west, legend cell align={left}, legend style={nodes={scale=0.65, transform shape}}]
      \addplot[mark=none, black, dotted, color=blue!20!red, thick] table [x=x,y=U]{plot_paper/ex1/data/1dtruth.txt};
      \addlegendentry{Truth}
      \addplot[color=blue!70!green!70!white,mark=none,thick] table [x=x,y=recon1]{plot_paper/ex1/opt/Gau/recons.csv};
      \addplot[color=blue!60!green!70!white,mark=none,thick] table [x=x,y=recon3]{plot_paper/ex1/opt/Gau/recons.csv};
      \addplot[color=blue!50!green!70!white,mark=none,thick] table [x=x,y=recon7]{plot_paper/ex1/opt/Gau/recons.csv};
      \addlegendentry{Minimizers}
    \end{axis}
    \node[color=black] at (0.4,2.7cm) {\small a)};
	\end{scope}
	\begin{scope}[xshift=4.0cm]
    \begin{axis}[width=.41\columnwidth,xmin=-1,xmax=1,ymin=-1.8,ymax=0.8,compat=1.13, xtick ={-0.5, 0, 0.5, 1}, yticklabels=\empty, ticklabel style = {font=\tiny}, legend pos=south west, legend cell align={left}, legend style={nodes={scale=0.65, transform shape}}]
      \addplot[mark=none, black,dotted, color=blue!20!red, thick] table [x=x,y=U]{plot_paper/ex1/data/1dtruth.txt};
      \addlegendentry{Truth}
      \addplot[color=blue!70!green!70!white,mark=none,thick] table [x=x,y=recon1]{plot_paper/ex1/opt/Cau_Gau/recons.csv};
      \addplot[color=blue!60!green!70!white,mark=none,thick] table [x=x,y=recon2]{plot_paper/ex1/opt/Cau_Gau/recons.csv};
      \addplot[color=blue!50!green!70!white,mark=none,thick] table [x=x,y=recon10]{plot_paper/ex1/opt/Cau_Gau/recons.csv};
      \addlegendentry{Minimizers}
    \end{axis}
    \node[color=black] at (0.4,2.7cm) {\small b)};
	\end{scope}
	\begin{scope}[xshift=8cm]
    \begin{axis}[width=.41\columnwidth,xmin=-1,xmax=1,ymin=-1.8,ymax=0.8,compat=1.13, xtick ={-0.5, 0, 0.5, 1}, yticklabels=\empty, ticklabel style = {font=\tiny}, legend pos=south west, legend cell align={left}, legend style={nodes={scale=0.65, transform shape}}]
      \addplot[mark=none, black,dotted, color=blue!20!red, thick] table [x=x,y=U]{plot_paper/ex1/data/1dtruth.txt};
      \addlegendentry{Truth}
      \addplot[color=blue!70!green!70!white,mark=none,thick] table [x=x,y=recon1]{plot_paper/ex1/opt/Cau/recons.csv};
      \addplot[color=blue!60!green!70!white,mark=none,thick] table [x=x,y=recon2]{plot_paper/ex1/opt/Cau/recons.csv};
      \addplot[color=blue!50!green!70!white,mark=none,thick] table [x=x,y=recon3]{plot_paper/ex1/opt/Cau/recons.csv};
      \addlegendentry{Minimizers}
    \end{axis}
    \node[color=black] at (0.4,2.7cm) {\small c)};
	\end{scope}
\end{tikzpicture}
\caption{Shown are reconstructions using different initializations obtained through
  optimization for Problem~\ref{ex1:fun_appro}. The results correspond
  to the Gaussian (a), Cauchy-Gaussian (b), and fully Cauchy
  (c) weights.  \label{fig:1drecons}}
\end{figure}

We also study the uncertainty based
on these reconstructions obtained with the ensemble
method (Sec.~\ref{sec:ensemble}) and the last-layer Gaussian regression method
(Sec.~\ref{sec:regression}). The results obtained with both approaches for
Cauchy-Gaussian weights are shown in Figure~\ref{fig:1densemblegr}. A quantitative summary of these results using the ensemble method is shown in Table~\ref{tb:tab1}. Here, we report the quantities
${\|\mathbb E \left[ u \right] -
  u^{\dagger}\|_{L^1}}/{\|u^{\dagger}\|_{L^1}}$ and the mean and
standard deviation of
${\| u - u^{\dagger}\|_{L^1}}/{\|u^{\dagger}\|_{L^1}}$, where
$u^{\dagger}$ and $u$ denote the truth and reconstructions,
respectively. Based on the
numerical reconstructions and the $L^1$ relative error, we observe
that the results obtained with Cauchy-Gaussian and fully Cauchy
weights better fit the discontinuities of the truth parameter
better.

\begin{figure}[tb]\centering
  \begin{tikzpicture}
    \begin{scope}[xshift=0cm]
    \begin{axis}[width=.54\columnwidth,xmin=-1,xmax=1,ymin=-1.8,ymax=0.8,compat=1.13, ytick ={-1.5, -1, -0.5, 0, 0.5}, legend pos=south west, legend cell align={left}, legend style={nodes={scale=0.65, transform shape}}]
      \addplot[mark=none, black,dotted, color=blue!20!red, thick] table [x=x,y=U]{plot_paper/ex1/data/1dtruth.txt};
      \addlegendentry{Truth}
      \addplot[color=black!40!orange,mark=none,thick] table [x=x,y=mean]{plot_paper/ex1/opt/ensemble_caugau.txt};
      \addlegendentry{Mean}
      \errorband[blue, opacity=0.2]{plot_paper/ex1/opt/ensemble_caugau.txt}{x}{mean}{std}
      \addlegendentry{$\pm 1.96$ Std. dev.}
    \end{axis}
    \node[color=black] at (0.4,4.05cm) {a)};
	\end{scope}
	\begin{scope}[xshift=6cm]
    \begin{axis}[width=.54\columnwidth,xmin=-1,xmax=1,ymin=-1.8,ymax=0.8,compat=1.13, xtick ={-0.5, 0, 0.5, 1}, yticklabels=\empty, legend pos=south west, legend cell align={left}, legend style={nodes={scale=0.65, transform shape}}]
      \addplot[mark=none, black,dotted, color=blue!20!red, thick] table [x=x,y=U]{plot_paper/ex1/data/1dtruth.txt};
      \addlegendentry{Truth}
      \addplot[color=black!40!orange,mark=none,thick] table [x=x,y=mean]{plot_paper/ex1/opt/gpr_caugau.txt};
      \addlegendentry{Mean}
      \errorband[blue, opacity=0.2]{plot_paper/ex1/opt/gpr_caugau.txt}{x}{mean}{std}
      \addlegendentry{$\pm 1.96$ Std. dev.}
    \end{axis}
    \node[color=black] at (0.4,4.05cm) {b)};
	\end{scope}
\end{tikzpicture}
\caption{Shown are the means and pointwise standard deviations obtained
  with the ensemble method (a) and the last-layer Gaussian regression
  method (b). Both results are for Cauchy-Gaussian priors for
  Problem~\ref{ex1:fun_appro}. \label{fig:1densemblegr}}
\end{figure}
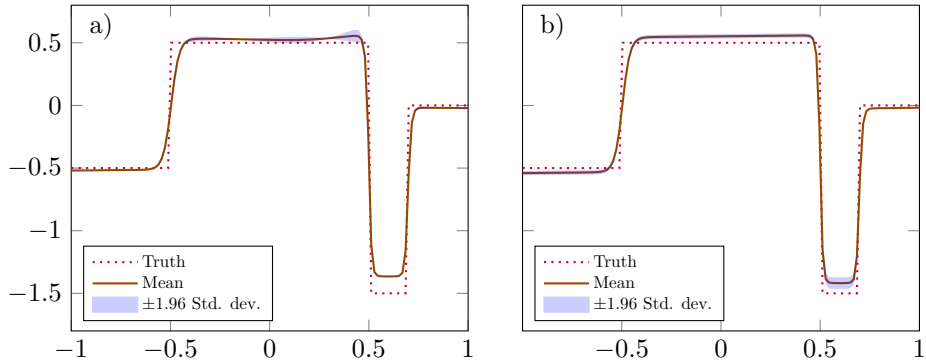

\begin{table}[ht!]
\centering
\caption{Relative $L^1$-error of reconstructions obtained using the ensemble method with
  Gaussian, Cauchy-Gaussian and Cauchy priors in
  Problem~\ref{ex1:fun_appro}.}
\begin{tabular}{ | p {3.5cm} | p {2.4cm} | p {2.9cm}| p {2.4cm}|}
\hline
\hfil Regularizations & \hfil Gaussian & \hfil Cauchy-Gaussian & \hfil Cauchy \\
\hline
 $\|\mathbb E \left[ u \right] - u^{\dagger}\|_{L^1}/\|u^{\dagger}\|_{L^1}$ &  \hfil 8.35 & \hfil 5.90 & \hfil 5.53 \\
 $\mathbb E [ \|u - u^{\dagger}\|_{L^1} ]/\|u^{\dagger}\|_{L^1}$ &  \hfil 8.44 & \hfil 5.91 & \hfil 5.56 \\
 $\text{Std} [\|u - u^{\dagger}\|_{L^1} ]/\|u^{\dagger}\|_{L^1}$ &  \hfil 0.10 & \hfil 0.11 & \hfil 0.13 \\
\hline
\end{tabular}
\label{tb:tab1}
\end{table}

\subsubsection{MCMC sampling}

We also aim at exploring the posterior using the MCMC method. We use a
local minimum obtained with the optimization method as the starting
point to reduce the burn-in phase. The pCN method is used to generate
$5 \times 10^6$ samples as discussed in
Sec.~\ref{subsubsec:MCMC_pcn}. We use an adaptive method
to adjust the value of $\beta$ to  maintain an
acceptance rate of about $30 \%$. Note that we always use Gaussian
increments in the pCN method. Numerical results obtained with each neural network prior are shown in Figure~\ref{fig:1dpcn}. We observe that the pCN samples based on
Cauchy weights capture the discontinuities better. The sample mean and
uncertainty region indicate a better fit to the truth model when we
use the Cauchy-Gaussian or the Cauchy neural network prior. We also
compare the $L^1$ relative error between the truth $u^\dagger$ and the
samples $u$ generated using each prior in Table~\ref{tb:tab2}. Note
that the landscape of the posterior is multi-modal due to the
nonlinearity of the neural network priors. Thus, these results might not fully explore the posterior distribution
but only provide a local estimate of the uncertainty around one (or
several) local minimizers.

\begin{figure}[tb]\centering
  \begin{tikzpicture}
    \begin{scope}[xshift=0cm]
    \begin{axis}[width=.41\columnwidth,xmin=-1,xmax=1,ymin=-1.8,ymax=0.8,compat=1.13, ytick ={-1.5, -1, -0.5, 0, 0.5}, ticklabel style = {font=\tiny}, legend pos=south west, legend cell align={left},legend style={nodes={scale=0.65, transform shape}}]
      \addplot[mark=none, black,dotted, color=blue!20!red, thick] table [x=x,y=U]{plot_paper/ex1/data/1dtruth.txt};
      \addlegendentry{Truth}
      \addplot[color=blue!40!green!70!white,mark=none,thick] table [x=x,y=r1]{plot_paper/ex1/MCMC/Gau/1dpcngau.txt};
      \addplot[color=blue!50!green!60!white,mark=none,thick] table [x=x,y=r2]{plot_paper/ex1/MCMC/Gau/1dpcngau.txt};
      \addplot[color=blue!60!green!50!white,mark=none,thick] table [x=x,y=r3]{plot_paper/ex1/MCMC/Gau/1dpcngau.txt};
      \addplot[color=blue!70!green!40!white,mark=none,thick] table [x=x,y=r4]{plot_paper/ex1/MCMC/Gau/1dpcngau.txt};
      \addlegendentry{Samples}
    \end{axis}
    \node[color=black] at (0.4,2.7cm) {\small a)};
	\end{scope}
	\begin{scope}[xshift=4cm]
    \begin{axis}[width=.41\columnwidth,xmin=-1,xmax=1,ymin=-1.8,ymax=0.8,compat=1.13, xtick ={-0.5, 0, 0.5, 1}, yticklabels=\empty, ticklabel style = {font=\tiny}, legend pos=south west, legend cell align={left},legend style={nodes={scale=0.65, transform shape}}]
      \addplot[mark=none, black,dotted, color=blue!20!red, thick] table [x=x,y=U]{plot_paper/ex1/data/1dtruth.txt};
      \addlegendentry{Truth}
      \addplot[color=blue!40!green!70!white,mark=none,thick] table [x=x,y=r1]{plot_paper/ex1/MCMC/Cau_Gau//1dpcncaugau.txt};
      \addplot[color=blue!50!green!60!white,mark=none,thick] table [x=x,y=r2]{plot_paper/ex1/MCMC/Cau_Gau/1dpcncaugau.txt};
      \addplot[color=blue!60!green!50!white,mark=none,thick] table [x=x,y=r3]{plot_paper/ex1/MCMC/Cau_Gau/1dpcncaugau.txt};
      \addplot[color=blue!70!green!40!white,mark=none,thick] table [x=x,y=r4]{plot_paper/ex1/MCMC/Cau_Gau/1dpcncaugau.txt};
      \addlegendentry{Samples}
    \end{axis}
    \node[color=black] at (0.4,2.7cm) {\small b)};
	\end{scope}
	\begin{scope}[xshift=8cm]
    \begin{axis}[width=.41\columnwidth,xmin=-1,xmax=1,ymin=-1.8,ymax=0.8,compat=1.13, xtick ={-0.5, 0, 0.5, 1}, yticklabels=\empty, ticklabel style = {font=\tiny}, legend pos=south west, legend cell align={left}, legend style={nodes={scale=0.65, transform shape}}]
      \addplot[mark=none, black,dotted, color=blue!20!red, thick] table [x=x,y=U]{plot_paper/ex1/data/1dtruth.txt};
      \addlegendentry{Truth}
      \addplot[color=blue!40!green!70!white,mark=none,thick] table [x=x,y=r1]{plot_paper/ex1/MCMC/Cau/1dpcncau.txt};
      \addplot[color=blue!50!green!60!white,mark=none,thick] table [x=x,y=r2]{plot_paper/ex1/MCMC/Cau/1dpcncau.txt};
      \addplot[color=blue!60!green!50!white,mark=none,thick] table [x=x,y=r3]{plot_paper/ex1/MCMC/Cau/1dpcncau.txt};
      \addplot[color=blue!70!green!40!white,mark=none,thick] table [x=x,y=r4]{plot_paper/ex1/MCMC/Cau/1dpcncau.txt};
      \addlegendentry{Samples}
    \end{axis}
    \node[color=black] at (0.4,2.7cm) {\small c)};
	\end{scope}

	\begin{scope}[xshift = 0cm, yshift=-4cm]
    \begin{axis}[width=.41\columnwidth,xmin=-1,xmax=1,ymin=-1.8,ymax=0.8,compat=1.13, ytick ={-1.5, -1, -0.5, 0, 0.5}, ticklabel style = {font=\tiny}, legend pos=south west, legend cell align={left}, legend style={nodes={scale=0.65, transform shape}}]
      \addplot[mark=none, black,dotted, color=blue!20!red, thick] table [x=x,y=U]{plot_paper/ex1/data/1dtruth.txt};
      \addlegendentry{Truth}
      \addplot[color=blue!40!orange,mark=none,thick] table [x=x,y=mean]{plot_paper/ex1/MCMC/Gau/1dpcngaumeanstd.txt};
      \addlegendentry{Mean}
      \errorband[blue, opacity=0.2]{plot_paper/ex1/MCMC/Gau/1dpcngaumeanstd.txt}{x}{mean}{std}
      \addlegendentry{$\pm 1.96$ Std. dev}
    \end{axis}
    \node[color=black] at (0.4,2.7cm) {\small d)};
	\end{scope}
	\begin{scope}[xshift = 4.0cm, yshift=-4cm]
    \begin{axis}[width=.41\columnwidth,xmin=-1,xmax=1,ymin=-1.8,ymax=0.8,compat=1.13, xtick ={-0.5, 0, 0.5, 1}, yticklabels=\empty, ticklabel style = {font=\tiny}, legend pos=south west, legend cell align={left}, legend style={nodes={scale=0.65, transform shape}}]
      \addplot[mark=none, black,dotted, color=blue!20!red, thick] table [x=x,y=U]{plot_paper/ex1/data/1dtruth.txt};
      \addlegendentry{Truth}
      \addplot[color=blue!40!orange,mark=none,thick] table [x=x,y=mean]{plot_paper/ex1/MCMC/Cau_Gau/1dpcncaugaumeanstd.txt};
      \addlegendentry{Mean}
      \errorband[blue, opacity=0.2]{plot_paper/ex1/MCMC/Cau_Gau/1dpcncaugaumeanstd.txt}{x}{mean}{std}
      \addlegendentry{$\pm 1.96$ Std. dev.}
    \end{axis}
    \node[color=black] at (0.4,2.7cm) {\small e)};
	\end{scope}
	\begin{scope}[xshift = 8cm, yshift=-4cm]
    \begin{axis}[width=.41\columnwidth,xmin=-1,xmax=1,ymin=-1.8,ymax=0.8,compat=1.13, xtick ={-0.5, 0, 0.5, 1}, yticklabels=\empty, ticklabel style = {font=\tiny}, legend pos=south west, legend cell align={left}, legend style={nodes={scale=0.65, transform shape}}]
      \addplot[mark=none, black,dotted, color=blue!20!red, thick] table [x=x,y=U]{plot_paper/ex1/data/1dtruth.txt};
      \addlegendentry{Truth}
      \addplot[color=blue!40!orange,mark=none,thick] table [x=x,y=mean]{plot_paper/ex1/MCMC/Cau/1dpcncaumeanstd.txt};
      \addlegendentry{Mean}
      \errorband[blue, opacity=0.2]{plot_paper/ex1/MCMC/Cau/1dpcncaumeanstd.txt}{x}{mean}{std}
      \addlegendentry{$\pm 1.96$ Std. dev.}
    \end{axis}
    \node[color=black] at (0.4,2.7cm) {\small f)};
	\end{scope}
\end{tikzpicture}
\caption{Shown in (a,b,c) are samples and in (d,e,f) the uncertainty of
  posterior distributions with different neural network priors for
  Problem~\ref{ex1:fun_appro}. The plots correspond to the Gaussian
  (a,d), Cauchy-Gaussian (b,e), and Cauchy (c,f) neural network
  priors. \label{fig:1dpcn}}
\end{figure}
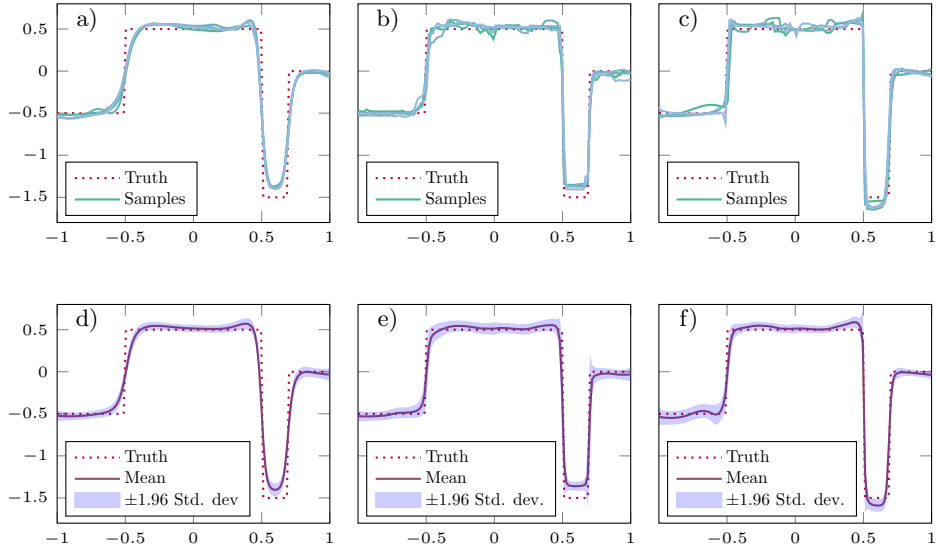

\begin{table}[ht!]
\centering
\caption{Relative $L^1$-error of samples computed by pCN
  with Gaussian, Cauchy-Gaussian and Cauchy priors in Problem~\ref{ex1:fun_appro}. }
\begin{tabular}{ | p {3.5cm} | p {2.4cm} | p {2.9cm}| p {2.4cm}|}
\hline
\hfil Neural network prior & \hfil Gaussian & \hfil Cauchy-Gaussian & \hfil Cauchy \\
\hline
 $\|\mathbb E \left[ u \right] - u^{\dagger}\|_{L^1}/\|u^{\dagger}\|_{L^1}$ &  \hfil 8.14 & \hfil 5.66 & \hfil 4.74 \\
 $\mathbb E [ \|u - u^{\dagger}\|_{L^1} ]/\|u^{\dagger}\|_{L^1}$ &  \hfil 8.57 & \hfil 6.45 & \hfil 5.63 \\
 $\text{Std} [\|u - u^{\dagger}\|_{L^1} ]/\|u^{\dagger}\|_{L^1}$ &  \hfil 0.63 & \hfil 0.57 & \hfil 0.73 \\
\hline
\end{tabular}
\label{tb:tab2}
\end{table}

\subsection{Two-dimensional deconvolution}

\begin{problem}\label{ex2:deblurring}
This is a two-dimensional deblurring problem on $\mathcal D = [-1, 1]^2$. The
forward model involves a PDE-solve, namely the solution of
\begin{subequations}\label{eq:deblurring_prob}
\begin{align}\label{eq:deblurring_prob:1}
(I - \kappa \Delta) y &= u \quad \text{in} \quad \mathcal D, \\
	\frac{\partial y}{\partial n} &= 0 \quad \text{on} \quad \partial \mathcal D,
\end{align}
\end{subequations}
where $y$ and $u$ are PDE solution and the unknown parameter field,
respectively.  In \eqref{eq:deblurring_prob}, $\kappa = 0.01$, which controls the
amount of blurring. The forward model is
denoted as
\begin{align*}
  \bs y_{\text{obs}} = A \left( u \right) + \bs \varepsilon, \quad \bs \varepsilon \sim  \mathcal N(0, \eta^2 I_{N_{\text{obs}}}),
\end{align*}
where $\eta=0.01$ and the forward operator is $A = P \circ B$. Here,
$B: L^infty(\mathcal D) \to \mathcal C^{\infty}$ is the PDE operator and
$P: \mathcal C^{\infty} \to \mathbb R^N_{\text{obs}}$ point evaluation
on a uniform grid of $14 \times 14$ points. We discretize the forward
PDE on a uniform mesh of $100 \times 100$ grid points using the
standard 5-point finite difference stencil.
\end{problem}

The true parameter function $u^\dagger$ and the blurred model are
shown in Figure~\ref{fig:2drecons}.  Similar to
Problem~\ref{ex1:fun_appro}, we test the optimization and pCN methods
on this problem with different neural network priors. For all tests,
we use the three-hidden-layer neural network structure with layer
widths of $[80, 80, 1000]$. Note that the input dimension of the
neural network is $2$, and the output dimension is $1$. This example
is computationally more costly than Problem~\ref{ex1:fun_appro} due to
the required PDE solves and the use of a wider network. Since the PDE
solve requires most of the computation time, we use an upfront
Cholesky decomposition of the discretized blurring operator to accelerate the computation.

\subsubsection{Optimization-based methods}
The objective to be minimized is formally as in the
one-dimensional case\label{eq:optobj}. We use the Adam method for the first $200$
steps followed by $1500$ iterations of L-BFGS.  The initial step size
is set to be $0.01$ for each method. Reconstructions obtained with
different regularizations and different initialization in the
optimization are shown in Figure~\ref{fig:2drecons}. We observe that
the reconstructions using the Cauchy neural network capture the edges
better, especially the middle edge and the right bottom block. In
contrast, the fully Gaussian neural network tends to result in smooth
reconstructions. We also note that the fully Cauchy neural network
results in better reconstructions of the linear ramp in the upper part
of the image compared to the Cauchy-Gaussian neural network.

\begin{figure}[tb]\centering
\begin{tikzpicture}
\node at (0,2cm) (img1) {\includegraphics[scale=0.31]{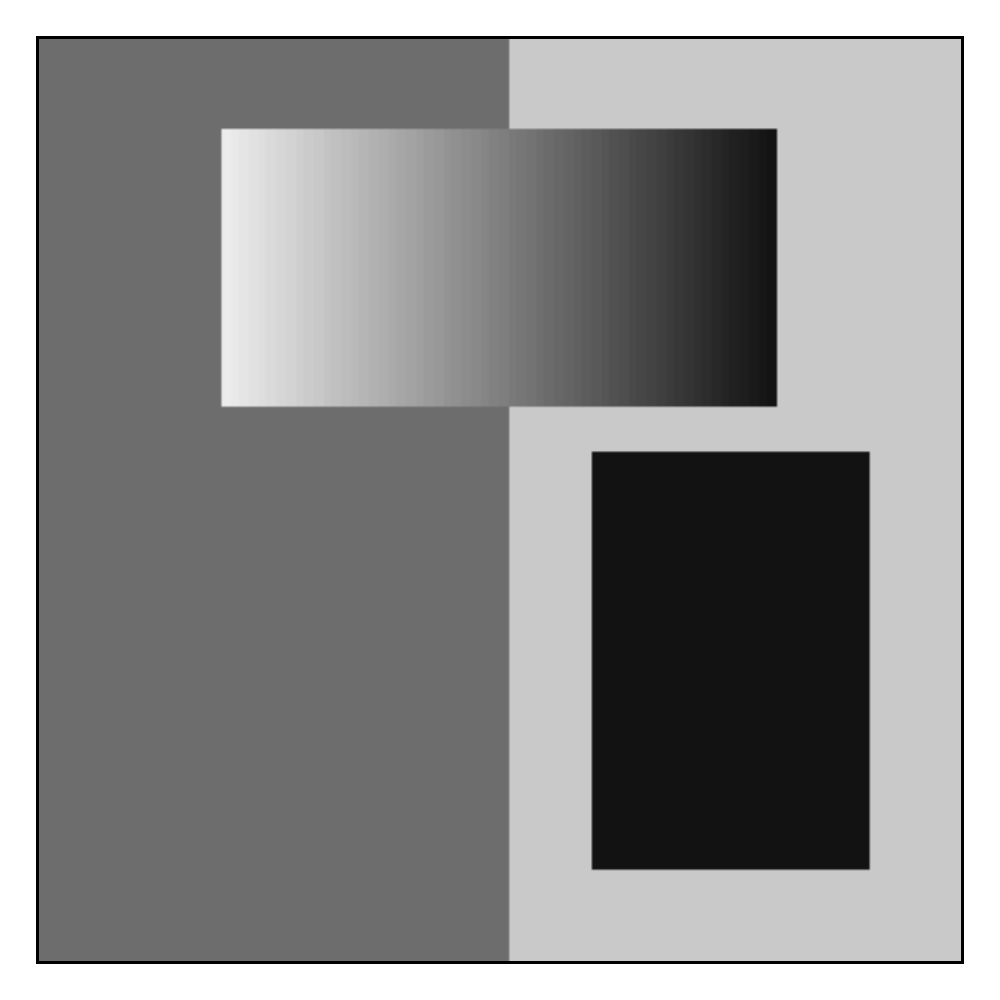}};
\node[above=of img1, node distance=0cm, yshift=-1.2cm,font=\color{black}] {Truth};
 \node at (3.1,2cm) (img12) {\includegraphics[scale=0.31]{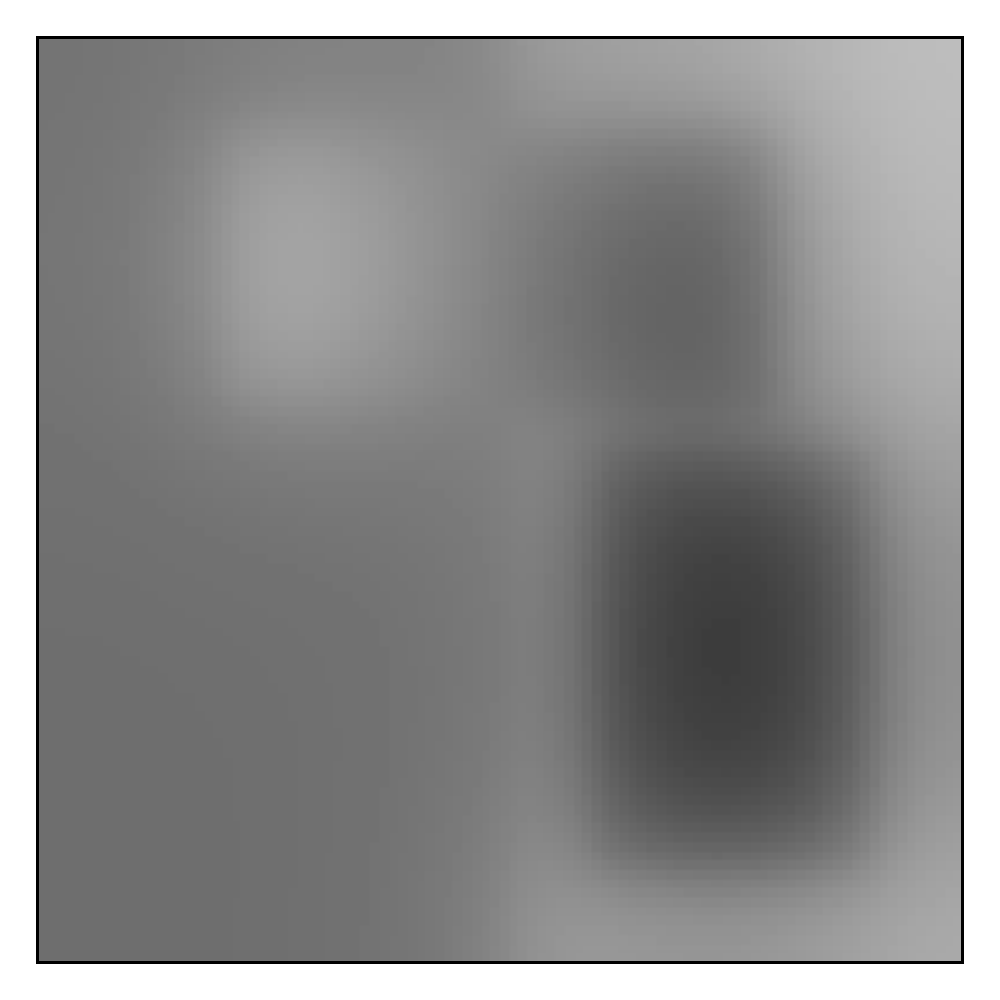}};
 \node[above=of img12, node distance=0cm, yshift=-1.2cm,font=\color{black}] {Convolution};
 \node at (7.8,1.7cm) (img13) {\includegraphics[scale=0.4]{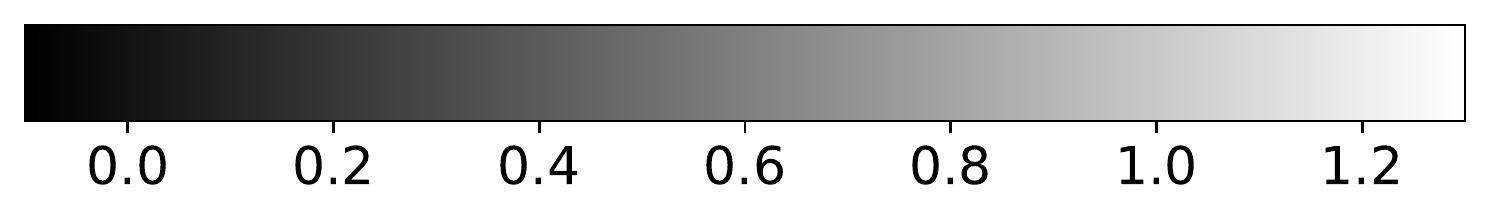}};
 \node[above=of img13, node distance=0cm, yshift=-1.2cm,font=\color{black}] {Colorbar $[-0.1, 1.3]$};

% Gaussian
\node at (0,-1.5cm) (img2) {\includegraphics[scale=0.31]{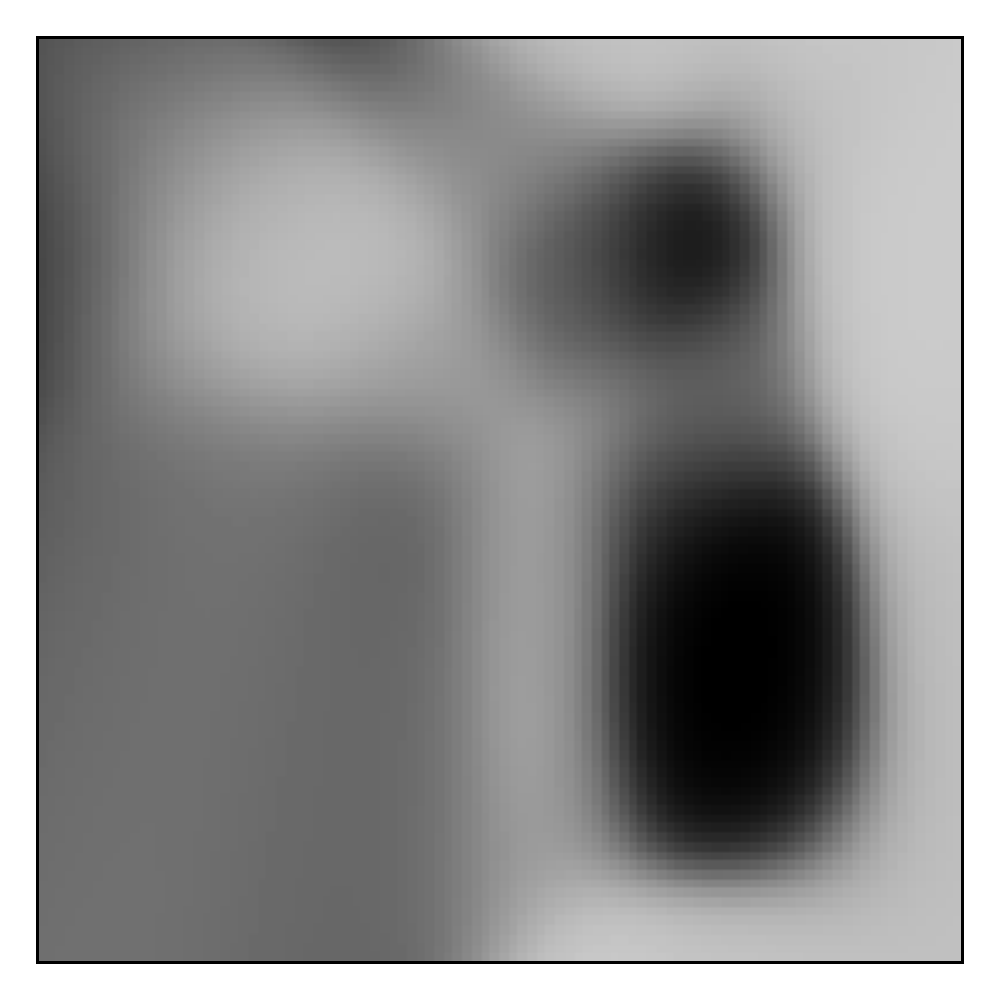}};
%\node[below=of img1, node distance=0cm, yshift=1cm,font=\color{red}] {x-axis};
  \node[left=.1cm of img2, node distance=0cm, rotate=90, anchor=center, yshift = 0cm, font=\color{black}] {Gaussian};
 \node at (3.1,-1.5cm) {\includegraphics[scale=0.31]{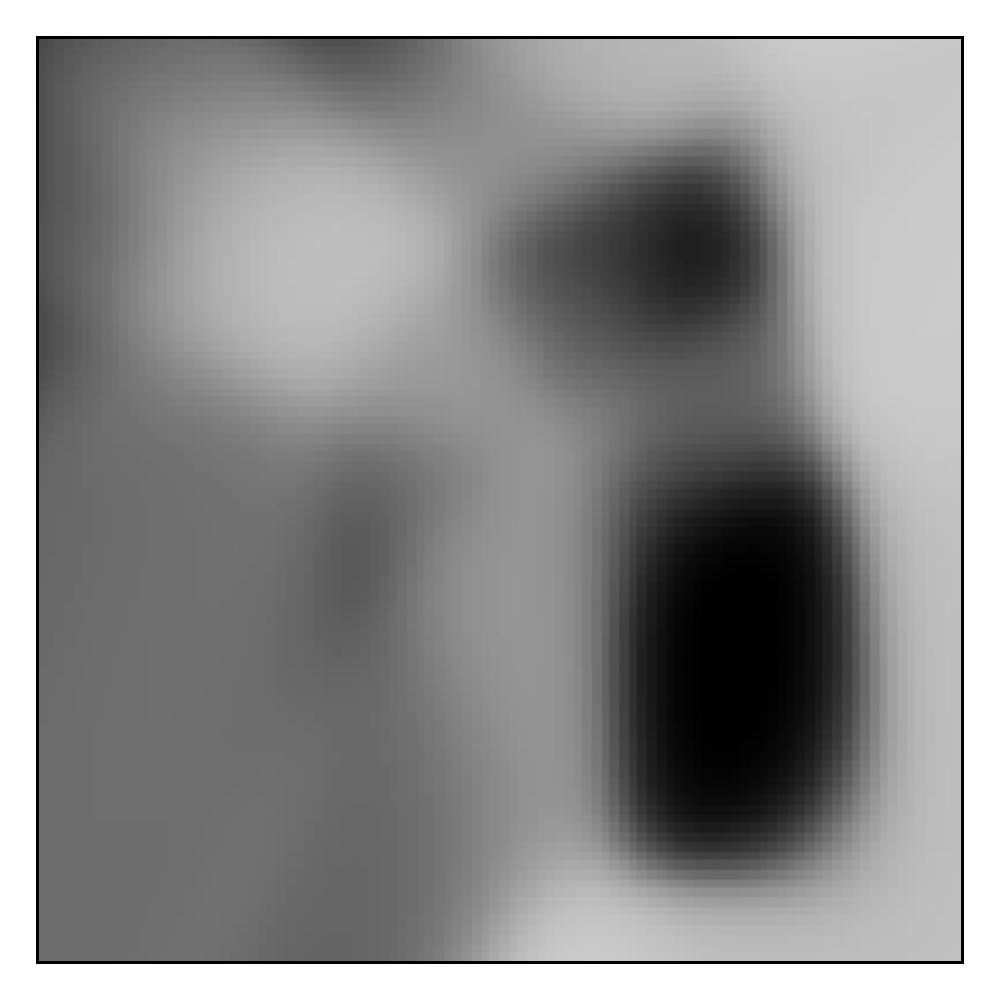}};
 \node at (6.2,-1.5cm) {\includegraphics[scale=0.31]{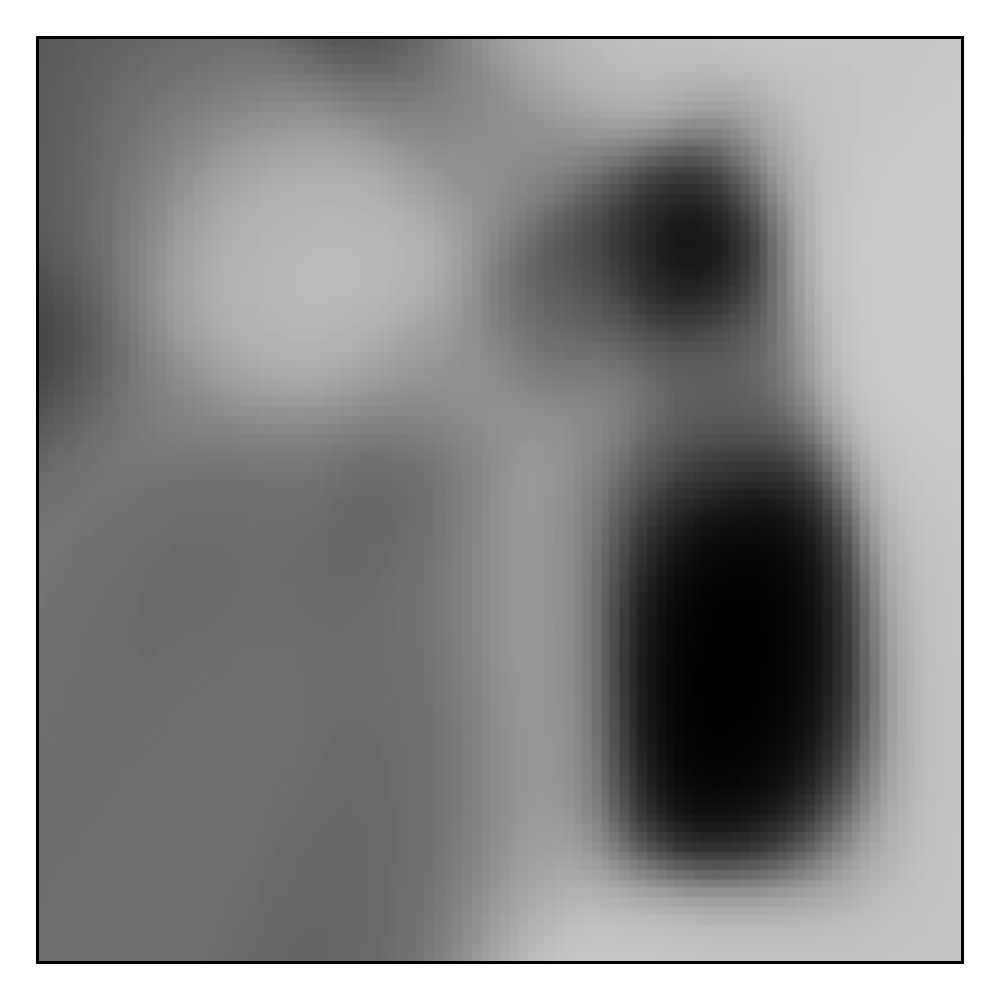}};
\node at (9.3,-1.5cm) {\includegraphics[scale=0.31]{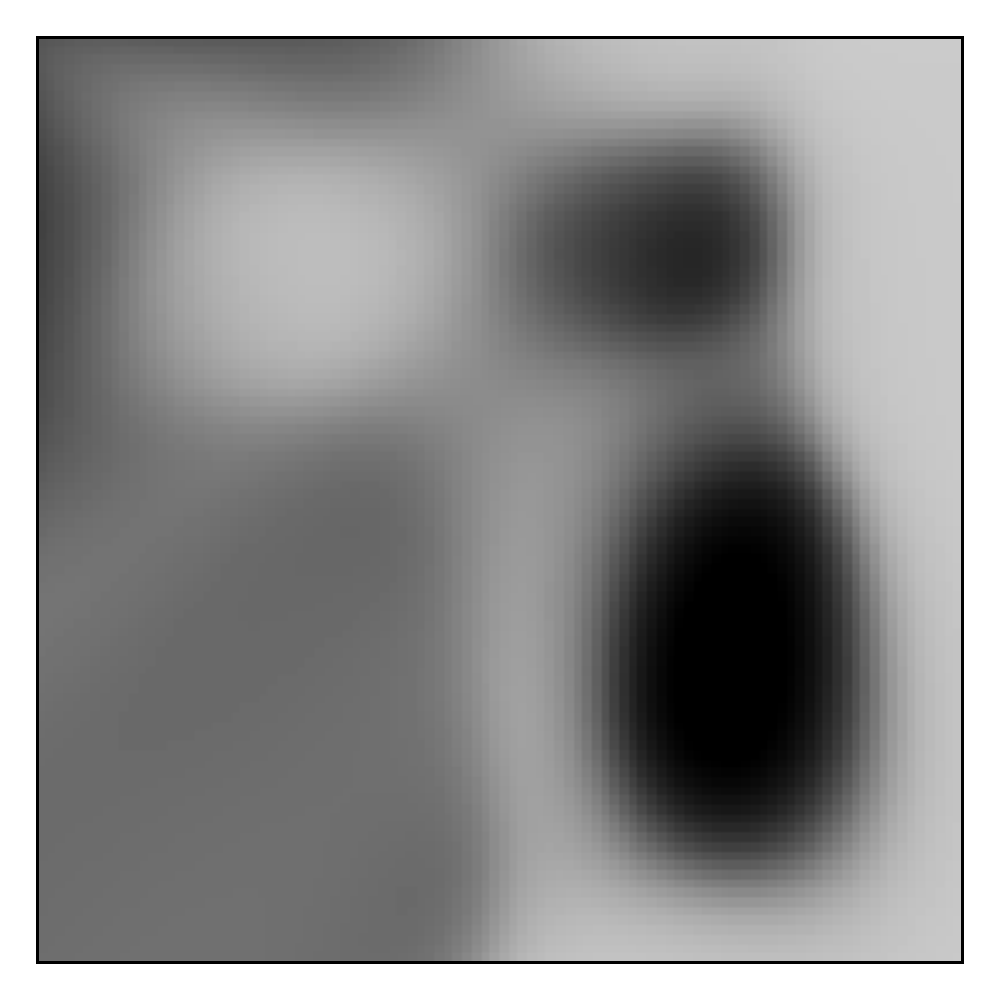}};

% Cauchy-Gaussian 
\node at (0,-5cm) (img3) {\includegraphics[scale=0.31]{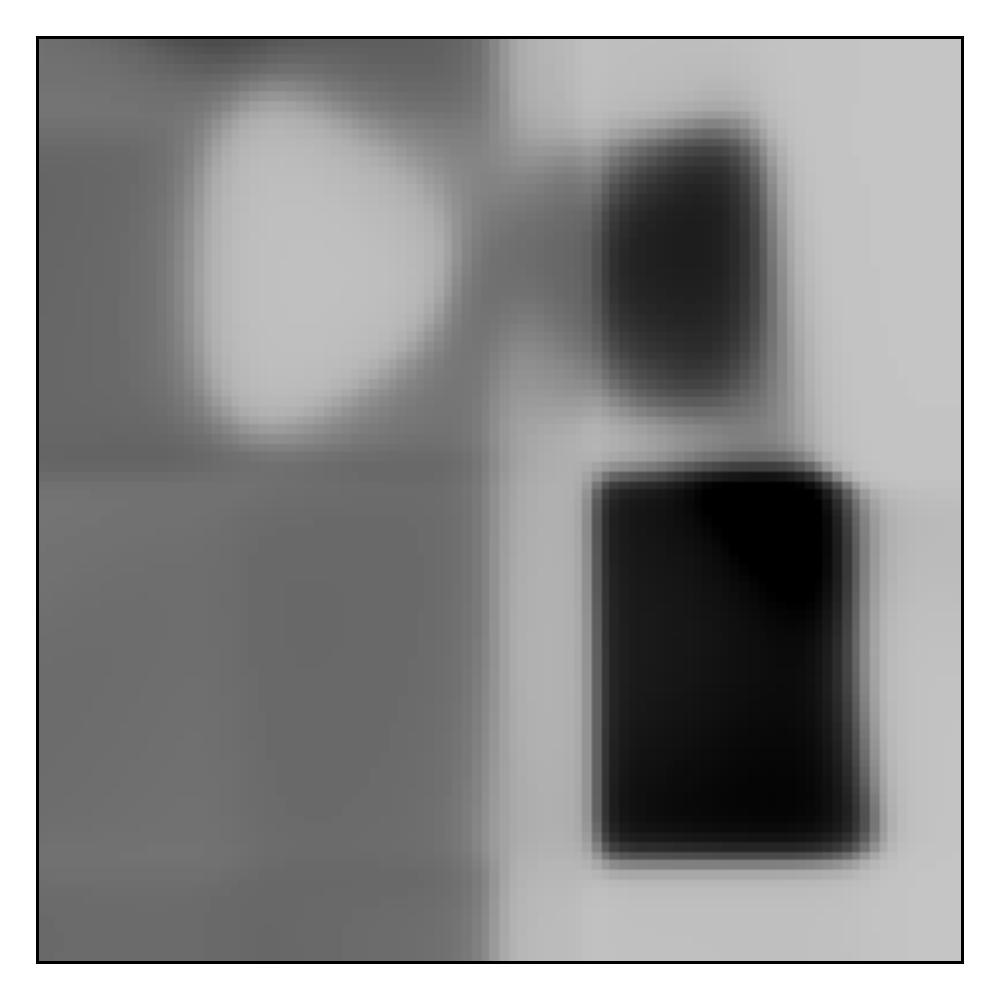}};
%\node[below=of img1, node distance=0cm, yshift=1cm,font=\color{red}] {x-axis};
  \node[left=.1cm of img3, node distance=0cm, rotate=90, anchor=center, yshift = 0cm, font=\color{black}] {Cauchy-Gaussian};
 \node at (3.1,-5cm) {\includegraphics[scale=0.31]{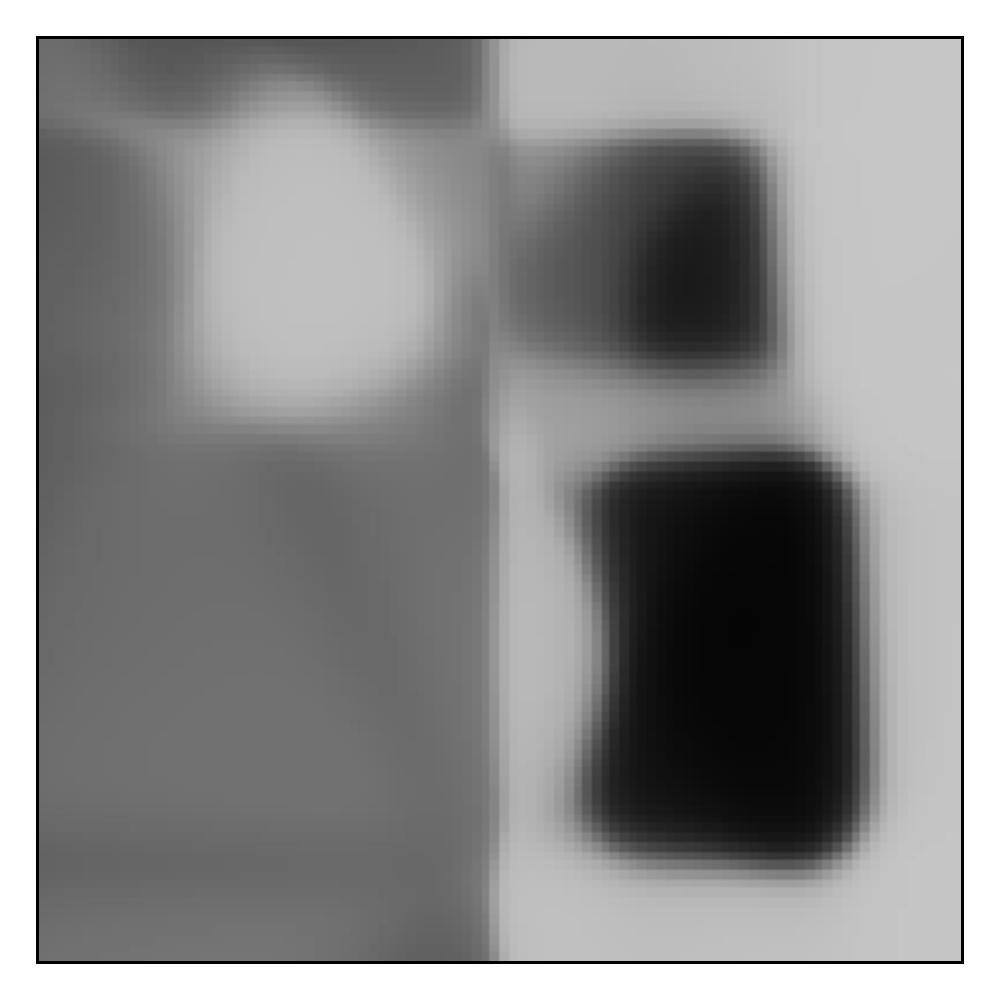}};
 \node at (6.2,-5cm) {\includegraphics[scale=0.31]{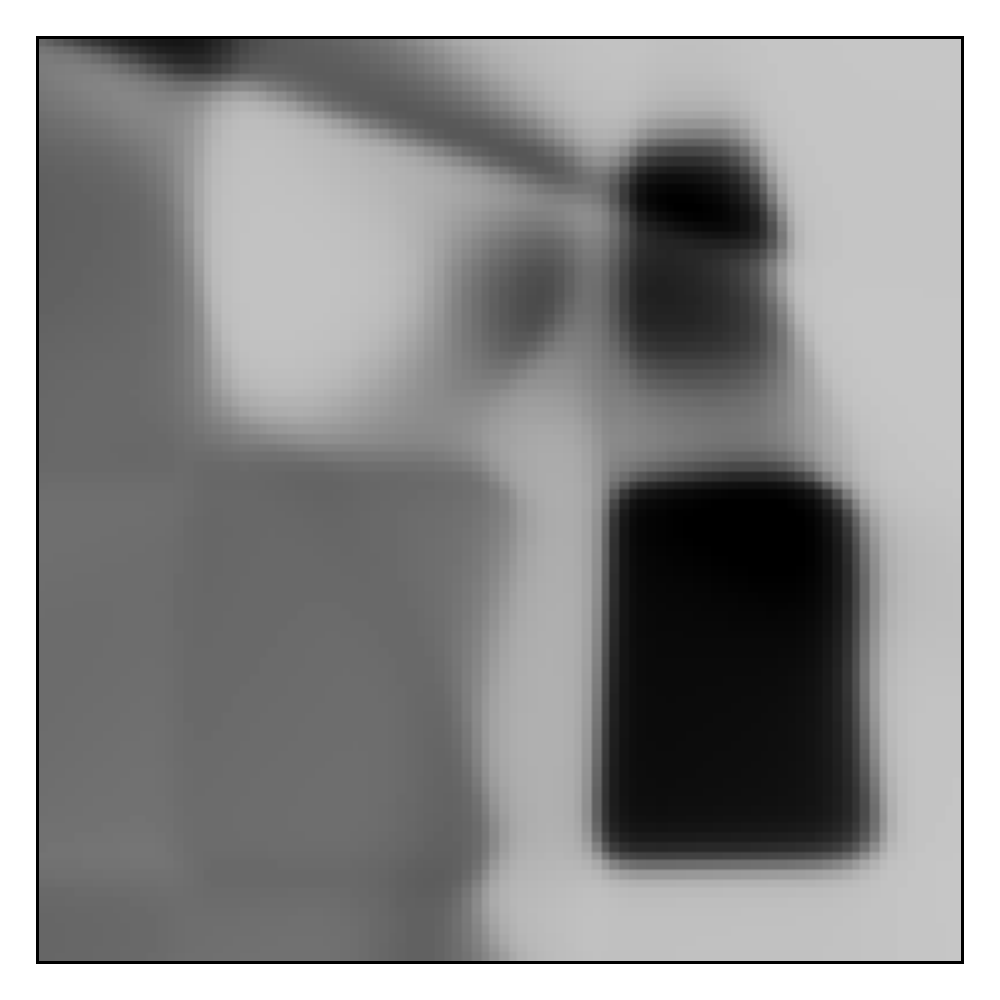}};
\node at (9.3,-5cm) {\includegraphics[scale=0.31]{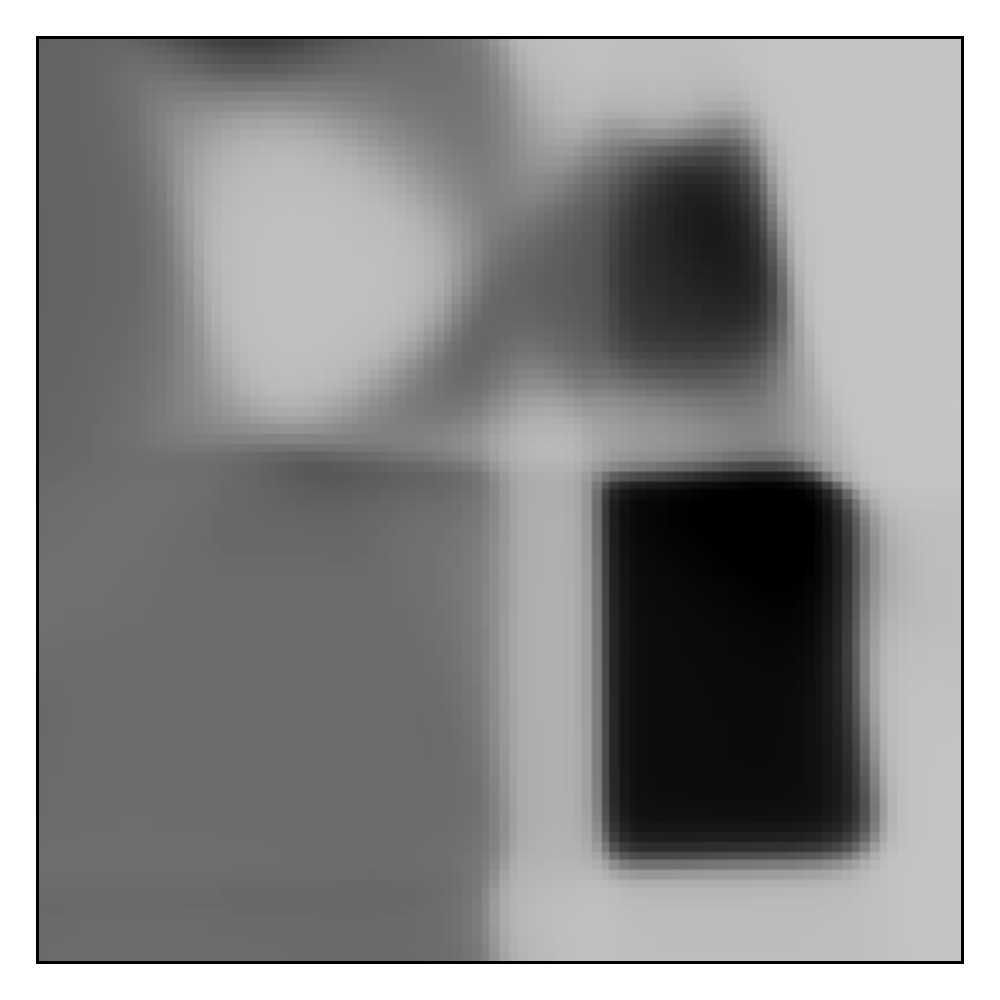}};

% Cauchy
\node at (0,-8.5cm) (img4) {\includegraphics[scale=0.31]{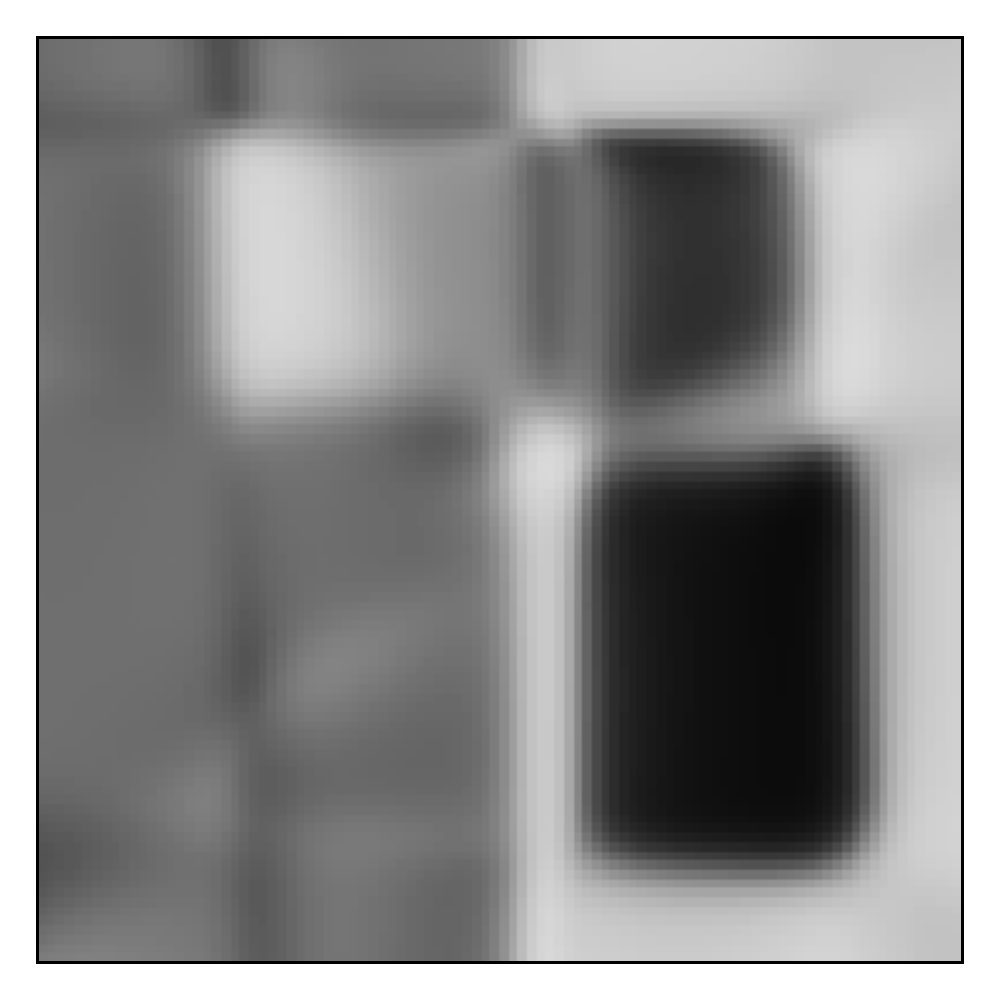}};
%\node[below=of img1, node distance=0cm, yshift=1cm,font=\color{red}] {x-axis};
  \node[left=.1cm of img4, node distance=0cm, rotate=90, anchor=center, yshift = 0cm, font=\color{black}] {Cauchy};
 \node at (3.1,-8.5cm) {\includegraphics[scale=0.31]{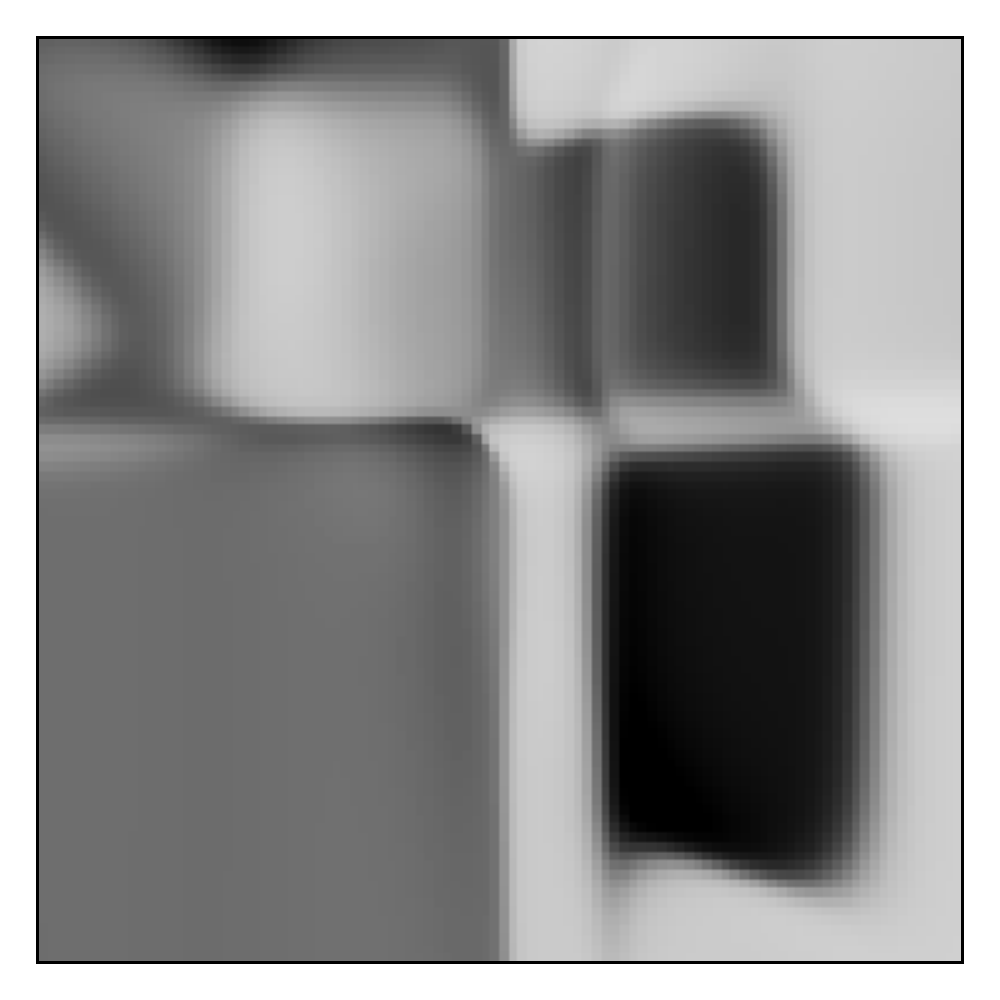}};
 \node at (6.2,-8.5cm) {\includegraphics[scale=0.31]{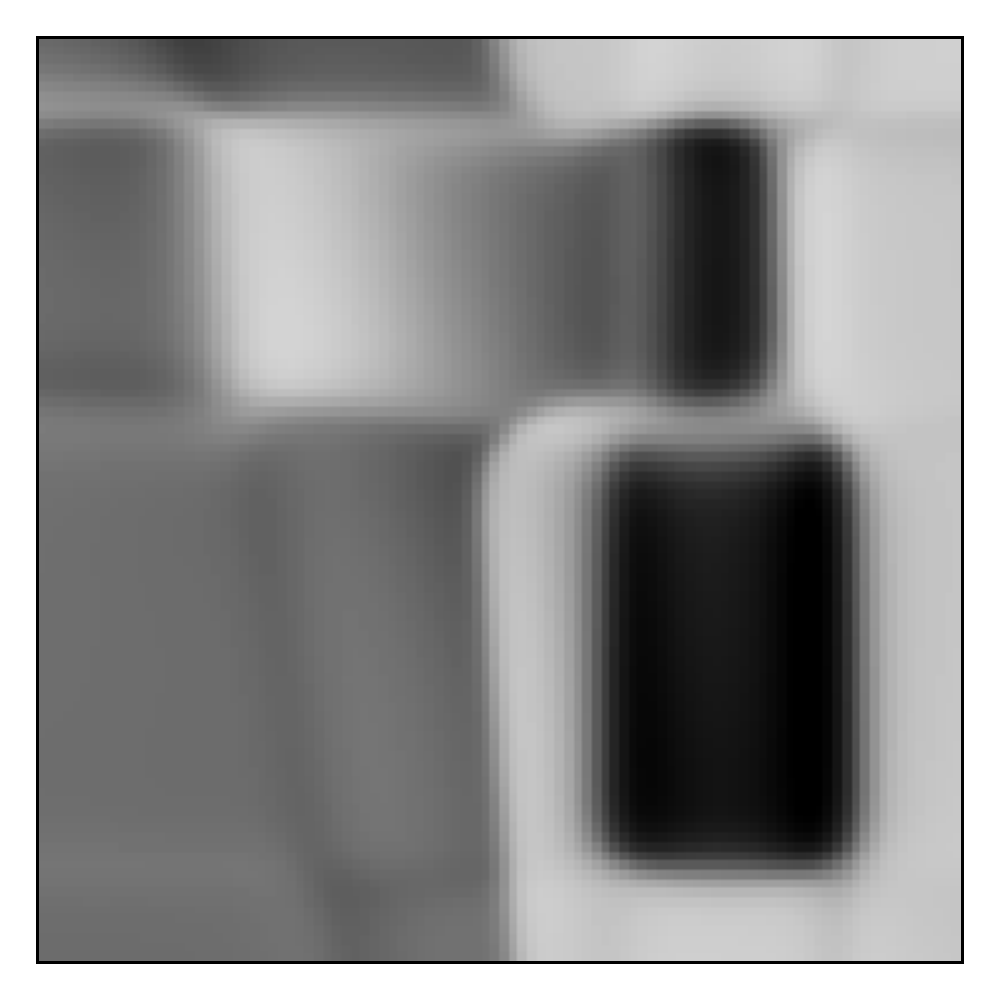}};
\node at (9.3,-8.5cm) {\includegraphics[scale=0.31]{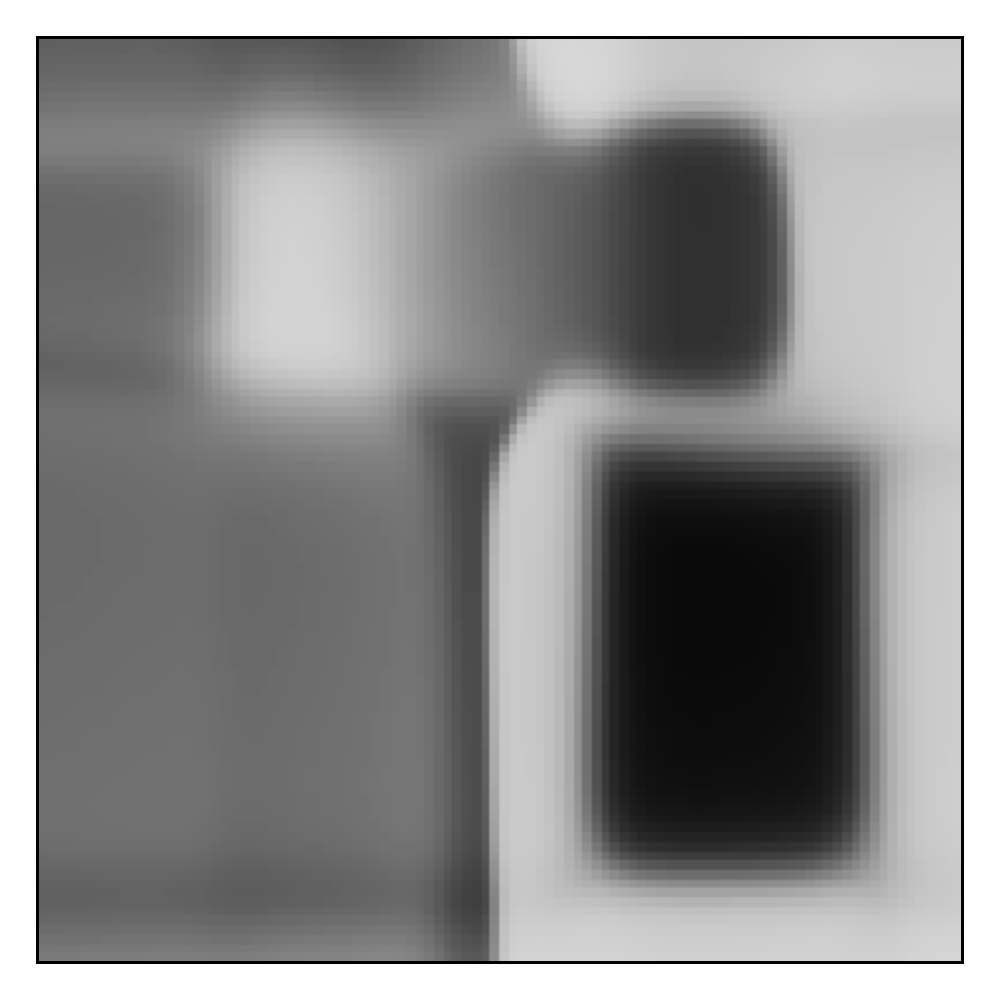}};
\end{tikzpicture}
\caption{Shown in the first row are the truth $u^\dagger$ and blurred
  models used for Problem~\ref{ex2:deblurring}. We show minimizers
  obtained using the optimization method with different
  initializations. Results are shown for Gaussian weights (second
  row), Cauchy-Gaussian weights (third row), and Cauchy weights (forth
  row). \label{fig:2drecons}}
\end{figure}
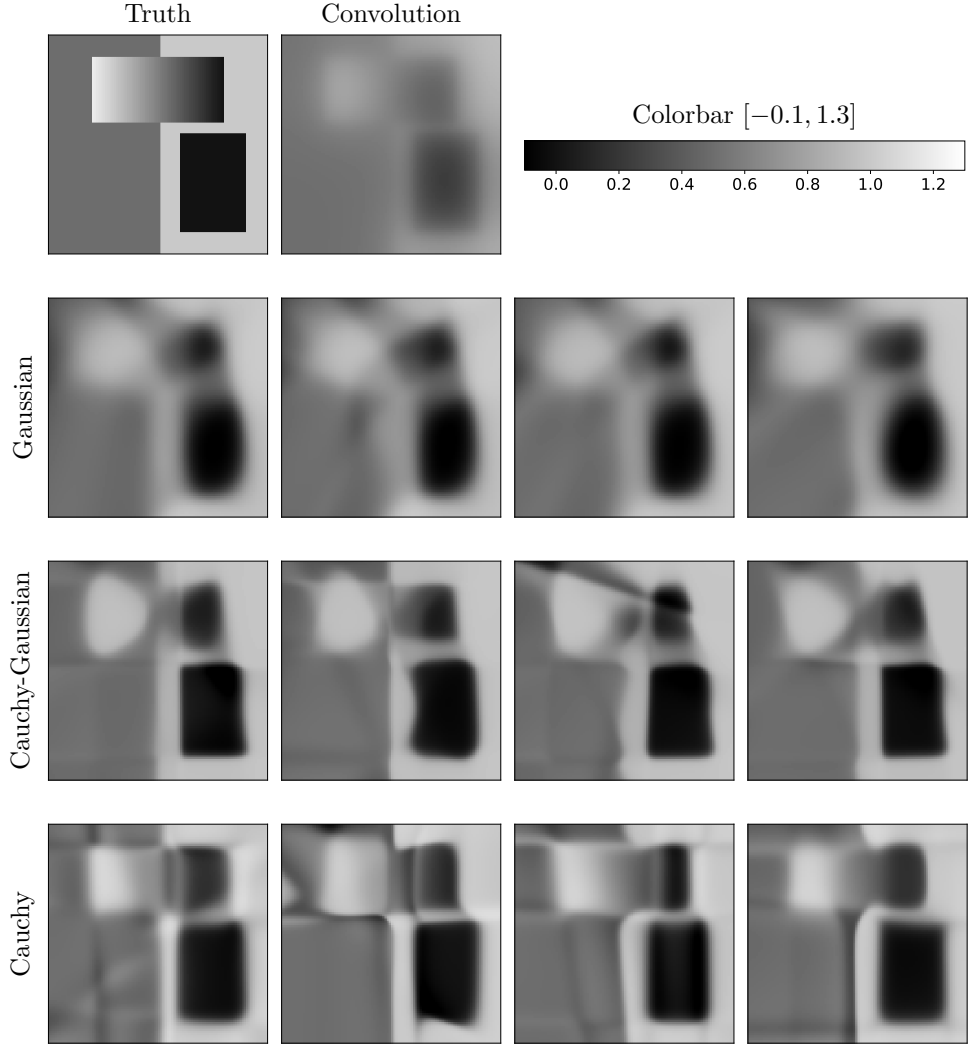

We also study the uncertainty using the ensemble method from
Sec.~\ref{sec:ensemble}. Results for different regularizations are shown in
Figure~\ref{fig:2densemble}. We observe that the Cauchy-Gaussian
neural network outperforms the Gaussian neural network for edge
detection.

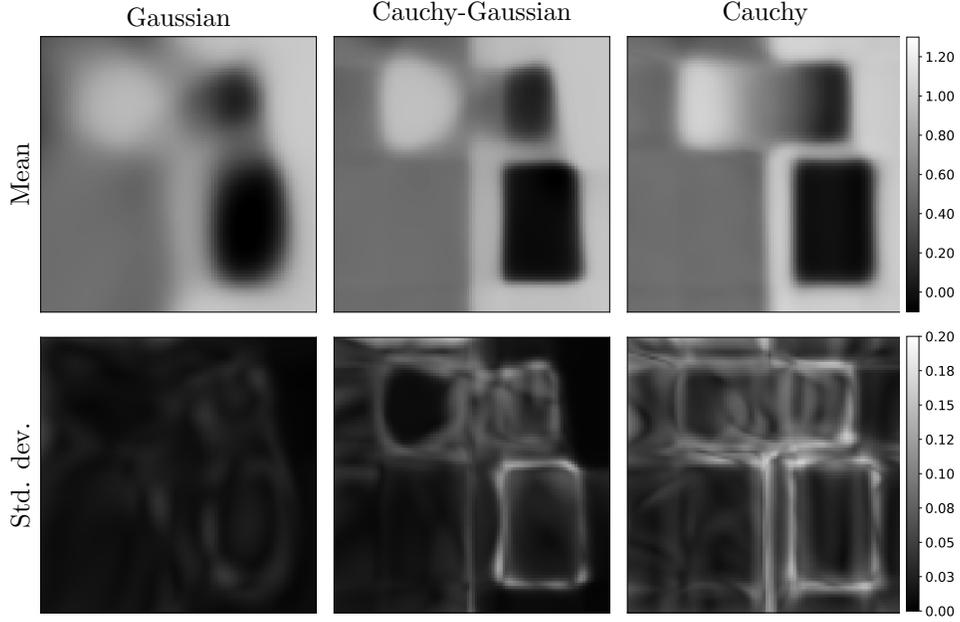
\begin{figure}[tb]\centering
\begin{tikzpicture}
% Mean
\node at (0,2cm) (img1) {\includegraphics[scale=0.39]{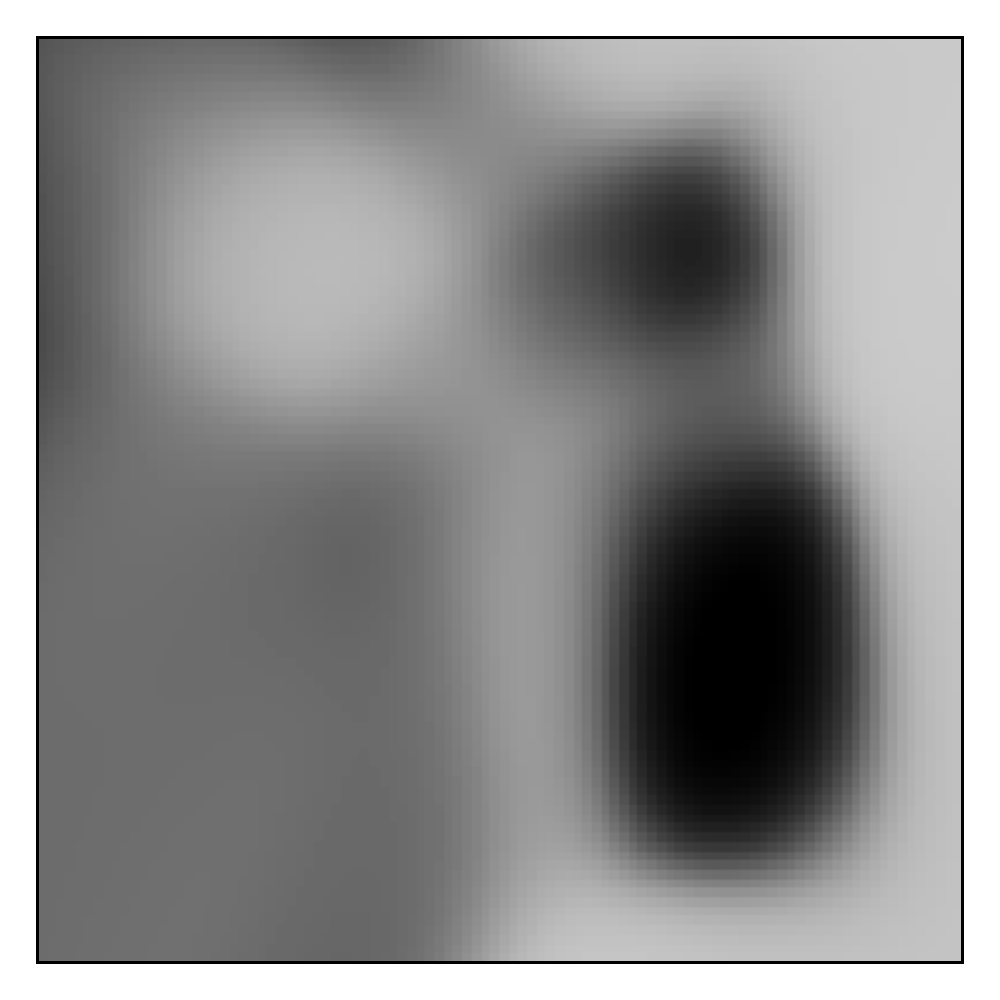}};
\node[left=0cm of img1, node distance=0cm, rotate=90, anchor=center, yshift = 0cm, font=\color{black}] {Mean};
\node[above=of img1, node distance=0cm, yshift=-1.25cm,font=\color{black}] {Gaussian};
 \node at (3.9,2cm) (img12) {\includegraphics[scale=0.39]{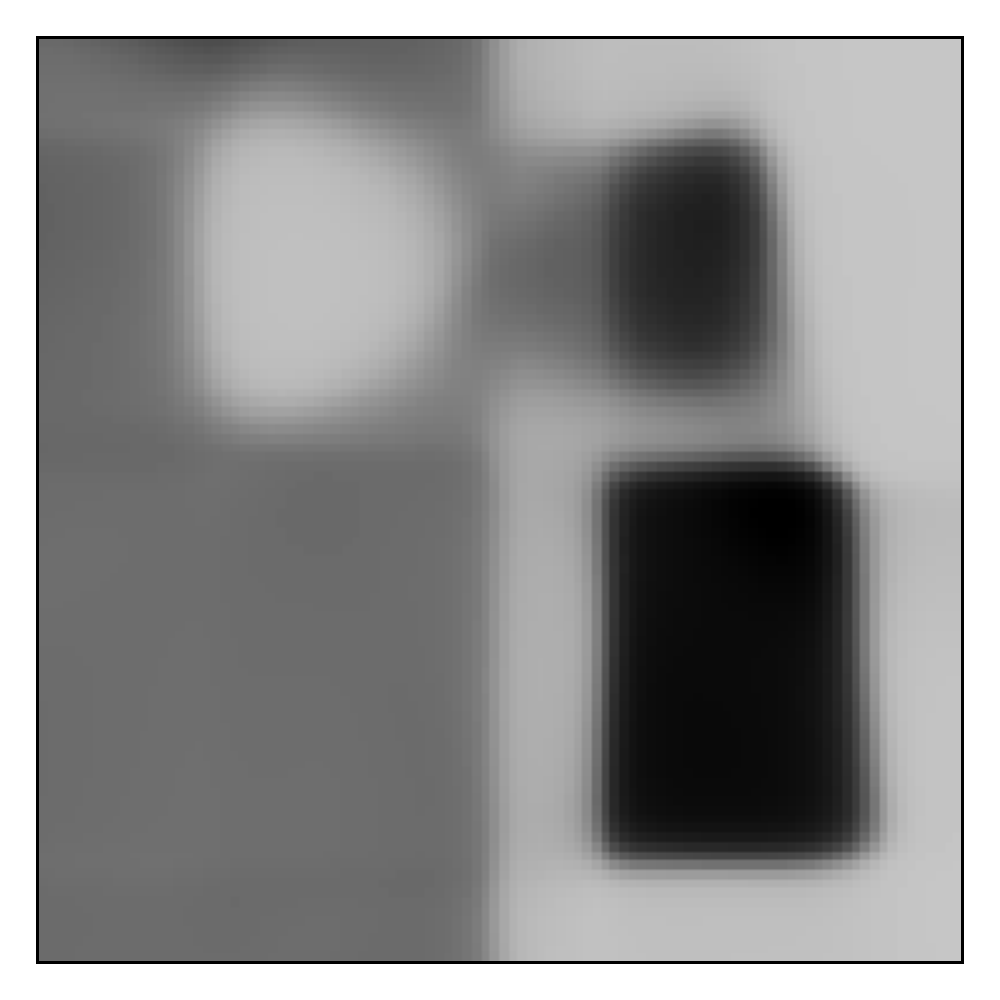}};
 \node[above=of img12, node distance=0cm, yshift=-1.25cm,font=\color{black}] {Cauchy-Gaussian};
 \node at (7.8,2cm) (img13) {\includegraphics[scale=0.39]{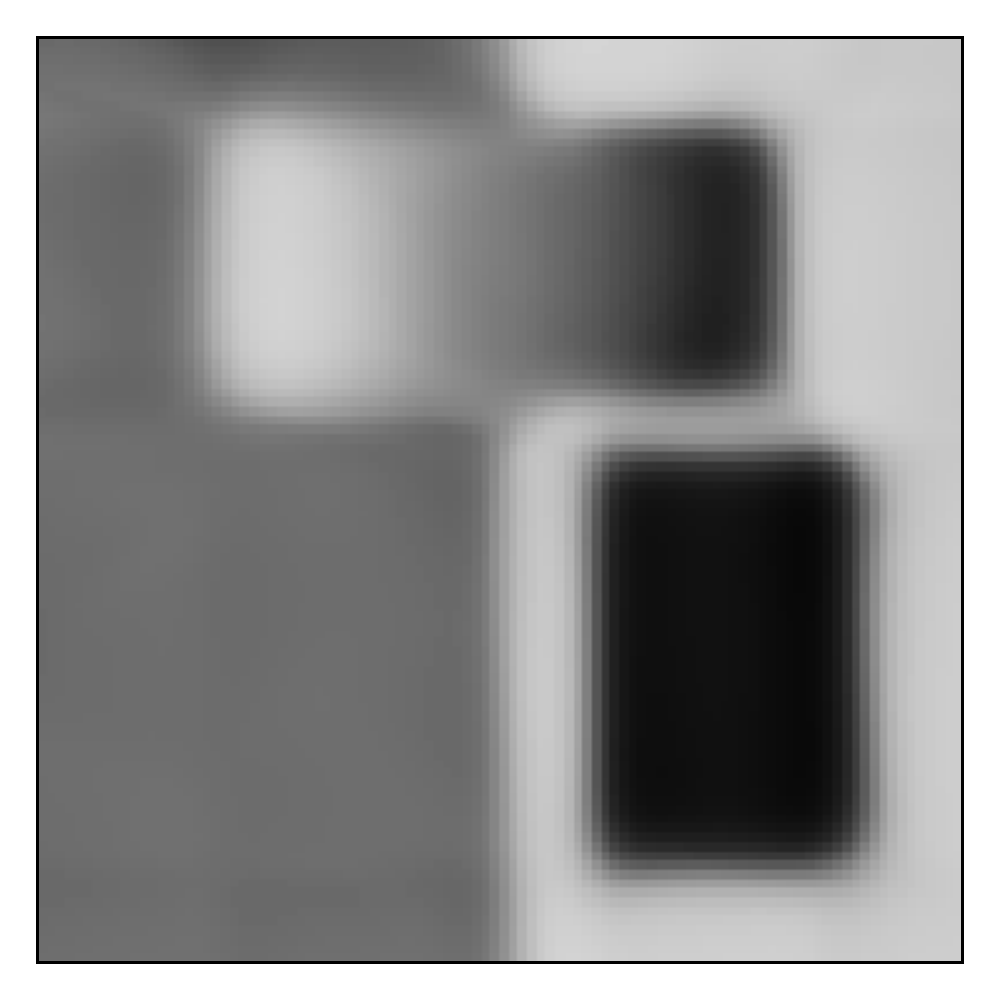}};
 \node[above=of img13, node distance=0cm, yshift=-1.25cm,font=\color{black}] {Cauchy};
\node at (10.0, 2cm) (img13) {\includegraphics[scale=0.333]{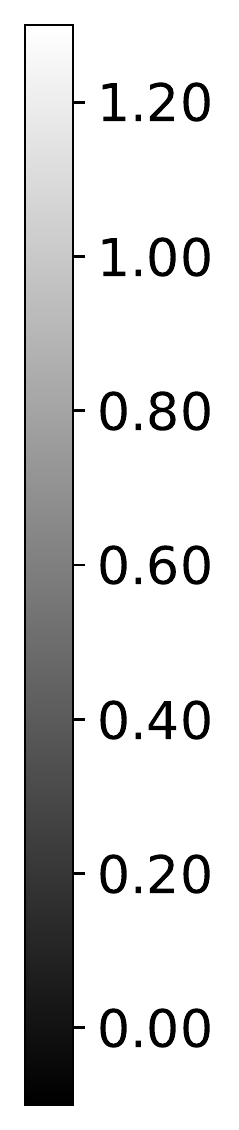}};

% Std
\node at (0,-2cm) (img1) {\includegraphics[scale=0.39]{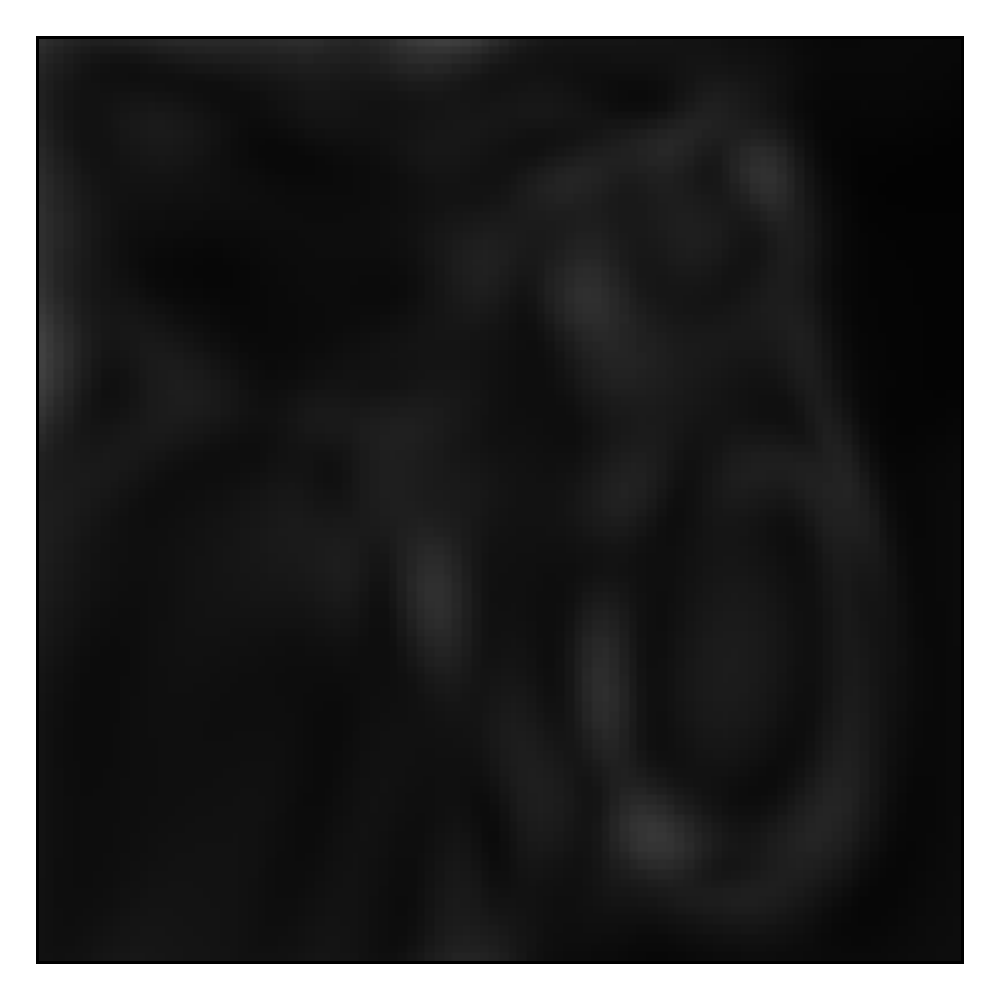}};
\node[left=0cm of img1, node distance=0cm, rotate=90, anchor=center, yshift = 0cm, font=\color{black}] {Std. dev.};
 \node at (3.9,-2cm) (img12) {\includegraphics[scale=0.39]{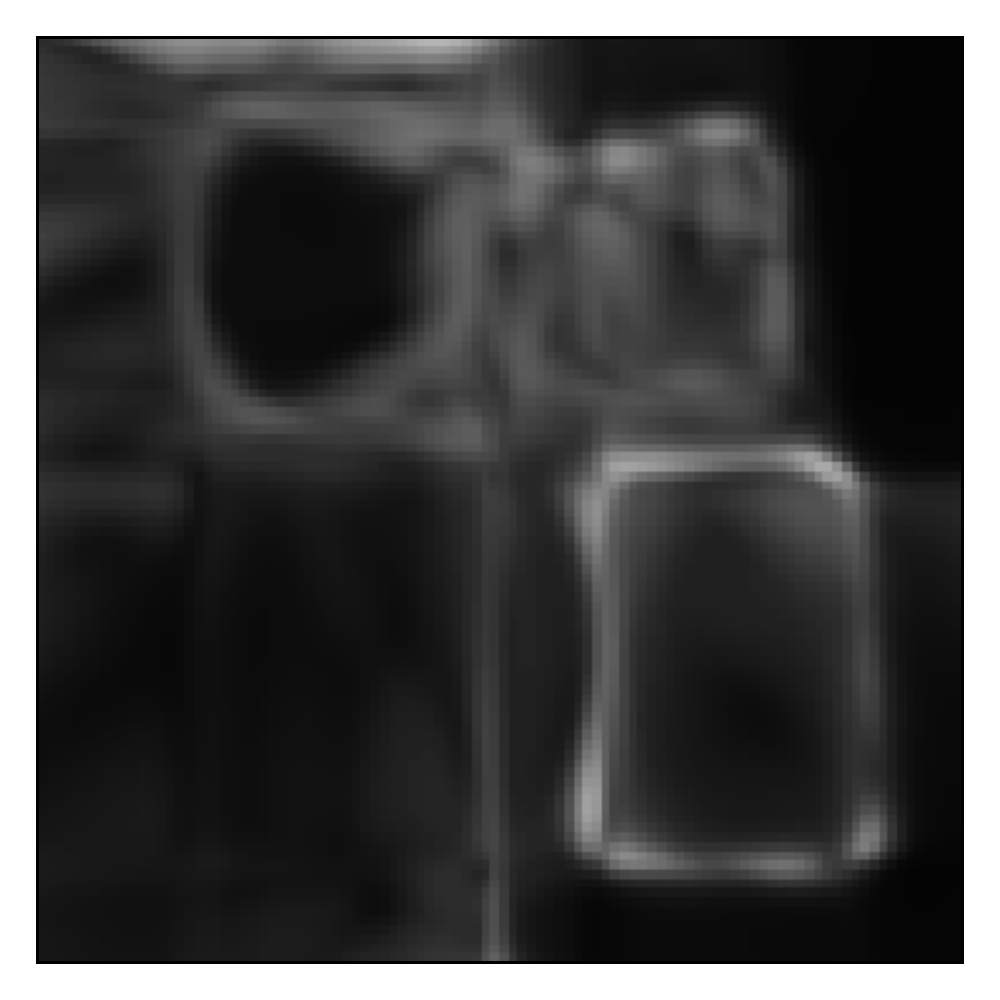}};
 \node at (7.8,-2cm) (img13) {\includegraphics[scale=0.39]{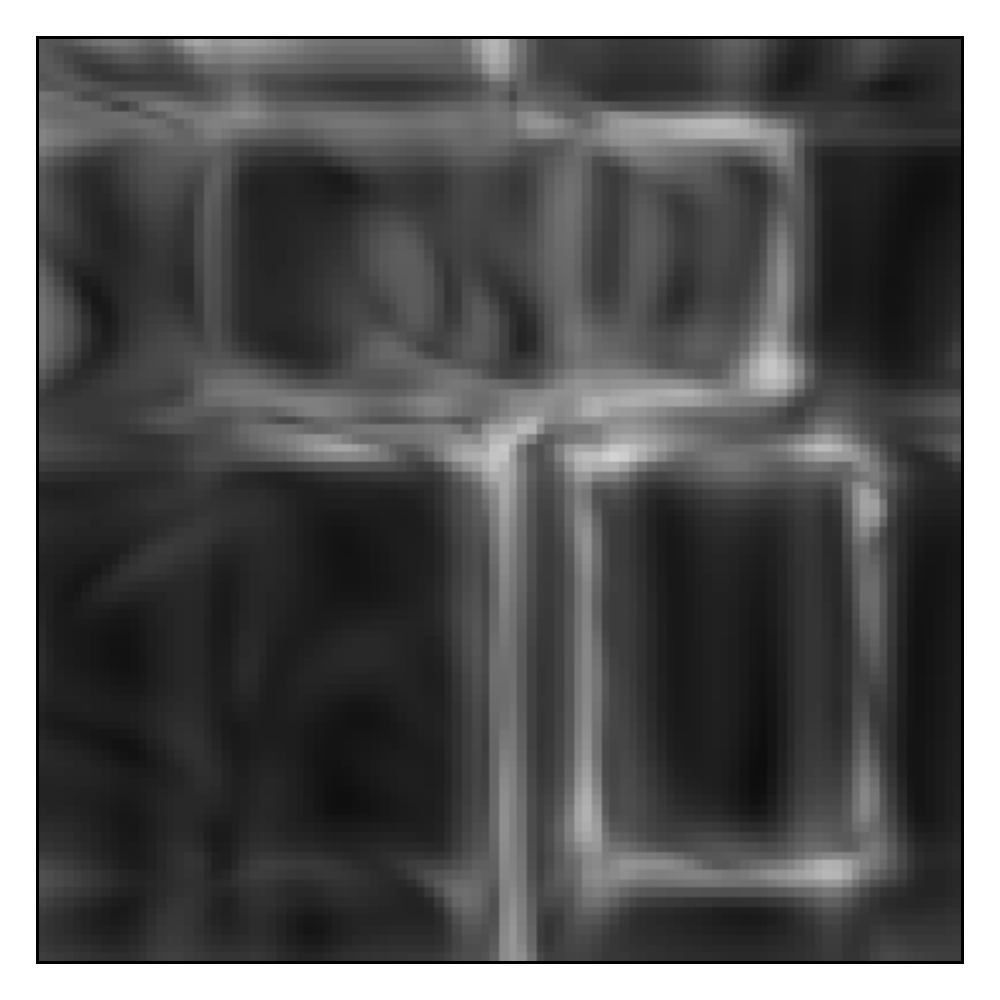}};
\node at (10.0, -2cm) (img13) {\includegraphics[scale=0.333]{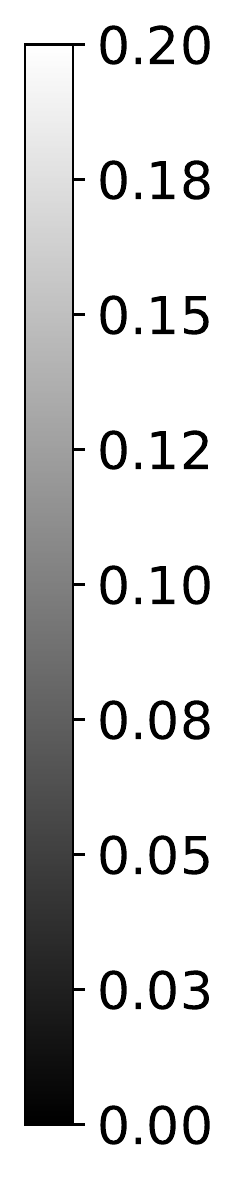}};
\end{tikzpicture}
\caption{Results obtained with ensemble method for
  Problem~\ref{ex2:deblurring}. Shown are the ensemble means (top row)
  and standard deviation (bottom row) obtained with Gaussian weights
  (left), Cauchy-Gaussian weights (middle), and fully Cauchy weights
  (right). \label{fig:2densemble}}
\end{figure}

An alternative uncertainty quantification approximation method is the
last-layer Gaussian regression based on the trained neural network
(Sec.~\ref{sec:regression}). The mean and standard deviation obtained by this approach from a fully Gaussian and a Cauchy-Gaussian neural network are shown in
Figure~\ref{fig:2dgpresult}. Since the mean of the
Gaussian regression is identical to the reconstruction we exploit, we
only show the standard deviation of each neural network with different
regularizations.

\begin{figure}[tb]\centering
\begin{tikzpicture}

% Std
\node at (0,-2cm) (img1) {\includegraphics[scale=0.5]{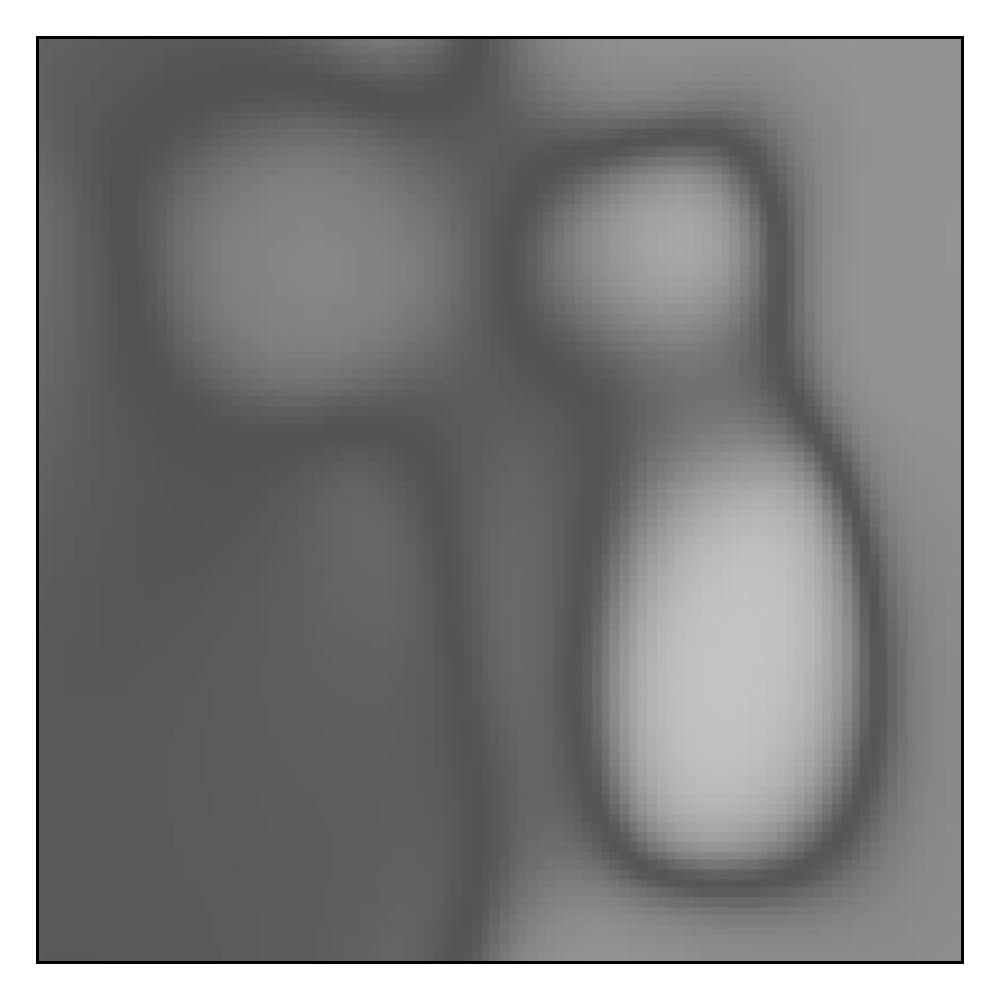}};
\node[above=of img1, node distance=0cm, yshift=-1.25cm,font=\color{black}] {Gaussian};
\node[left=.15cm of img1, node distance=0cm, rotate=90, anchor=center, yshift = 0cm, font=\color{black}] {Std. dev.};
 \node at (5.3,-2cm) (img12) {\includegraphics[scale=0.5]{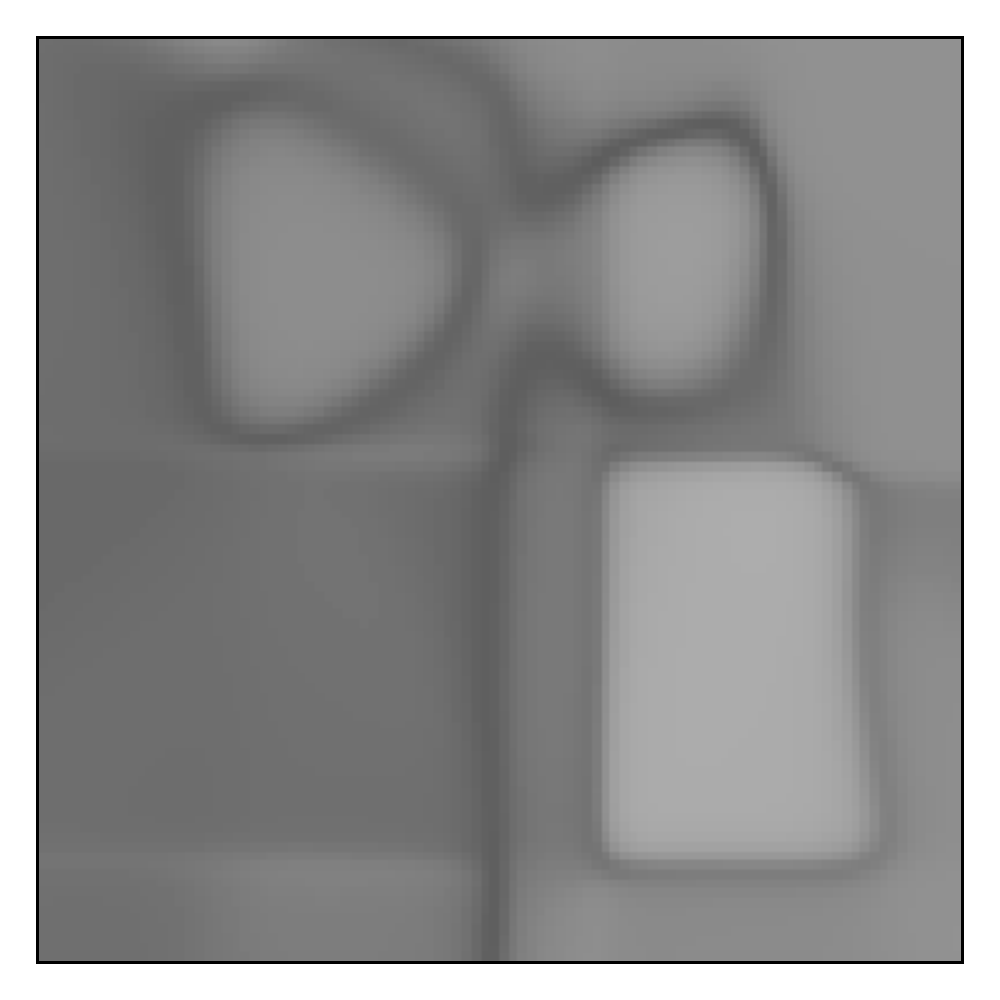}};
 \node[above=of img12, node distance=0cm, yshift=-1.25cm,font=\color{black}] {Cauchy-Gaussian};
\node at (8.4, -2cm) (img13) {\includegraphics[scale=0.43]{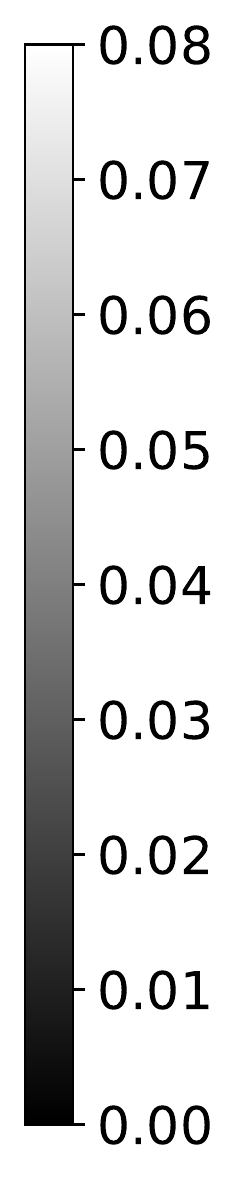}};
\end{tikzpicture}
\caption{Results for last-layer Gaussian regression,
  Problem~\ref{ex2:deblurring}.  Shown are the standard deviations
  with last-layer base functions from a pre-trained network. The
  figures are for networks with Gaussian (left) and Cauchy-Gaussian
  (right) weights. \label{fig:2dgpresult}}
\end{figure}
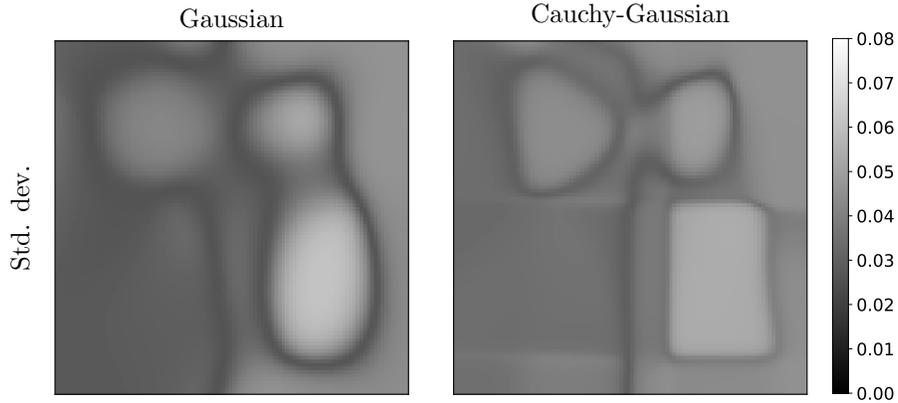

\subsubsection{MCMC sampling}

We also test the MCMC sampling method on this two-dimensional
deconvolution example using $10^6$ samples of the
adaptive pCN method. A numerical comparison of the uncertainty with
different priors is shown in Figure~\ref{fig:2dpcn}.
We observe that the mean of the chain with a Cauchy neural network
prior recovers the edges of the parameter function better while the
Gaussian neural network prior results in smoother
reconstructions. From the plots of the standard deviation, we also
note that the result obtained with a Cauchy neural network prior
possesses a larger variation around the edges compared to other
regions, which implies a better edge detection. We also note that the
neural network priors with fully Cauchy weights can learn the
linearly-changing block on the top better than the Cauchy-Gaussian
network prior.

\begin{figure}[tb]\centering
\begin{tikzpicture}
% Mean
\node at (0,2cm) (img1) {\includegraphics[scale=0.39]{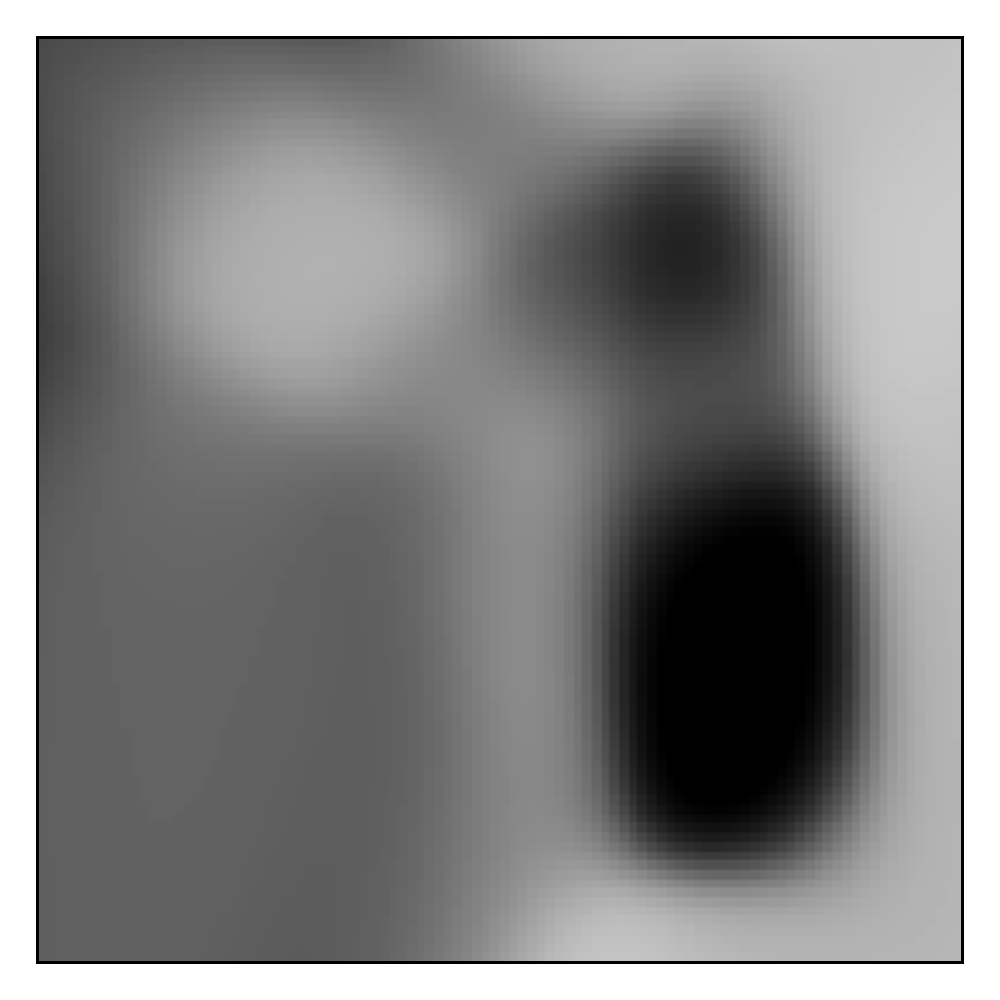}};
\node[left=0cm of img1, node distance=0cm, rotate=90, anchor=center, yshift = 0cm, font=\color{black}] {Mean};
\node[above=of img1, node distance=0cm, yshift=-1.25cm,font=\color{black}] {Gaussian};
 \node at (3.9,2cm) (img12) {\includegraphics[scale=0.39]{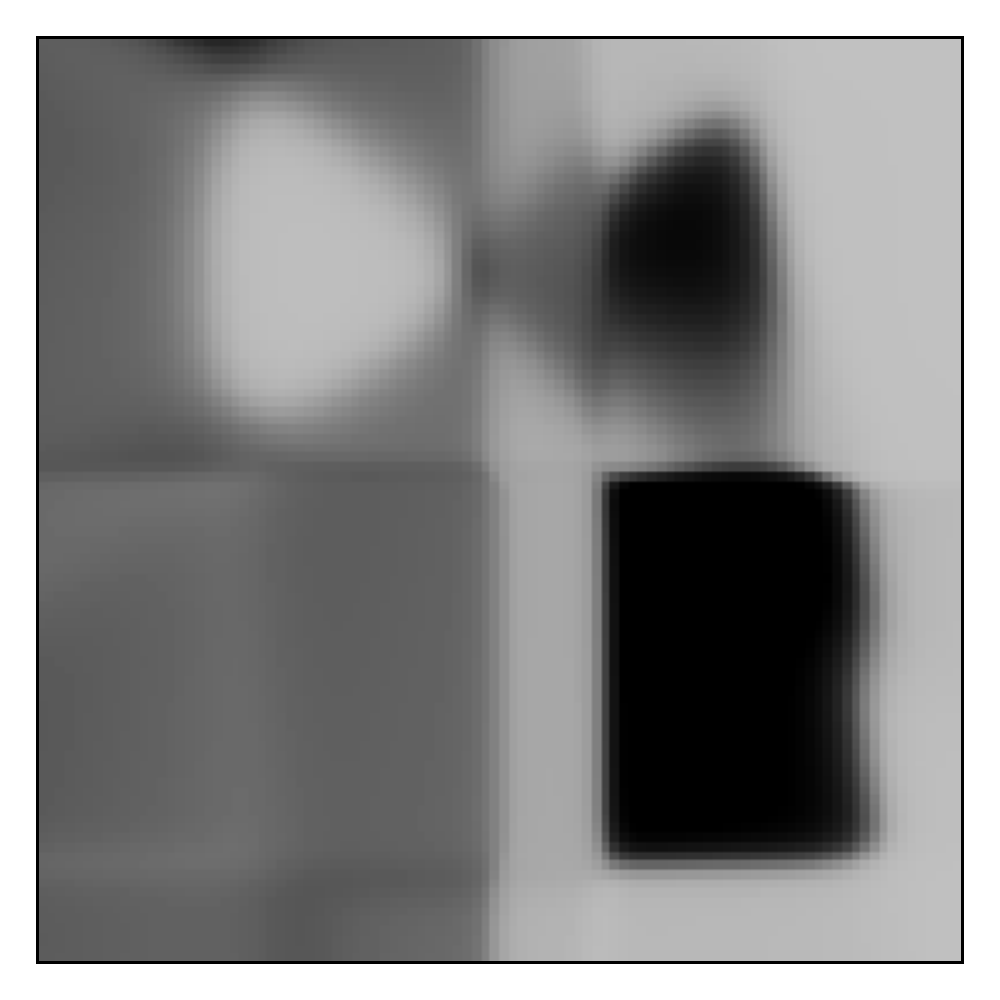}};
 \node[above=of img12, node distance=0cm, yshift=-1.25cm,font=\color{black}] {Cauchy-Gaussian};
 \node at (7.8,2cm) (img13) {\includegraphics[scale=0.39]{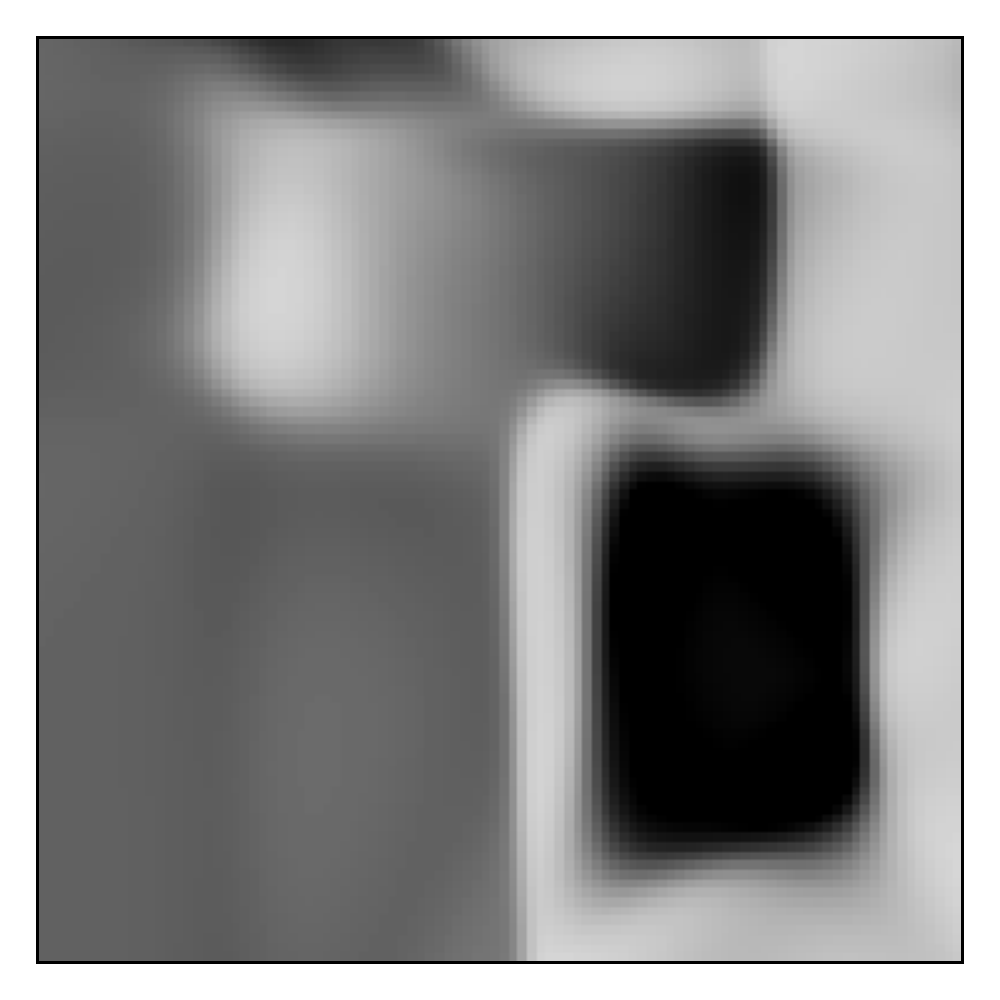}};
 \node[above=of img13, node distance=0cm, yshift=-1.25cm,font=\color{black}] {Cauchy};
\node at (10.0, 2cm) (img13) {\includegraphics[scale=0.333]{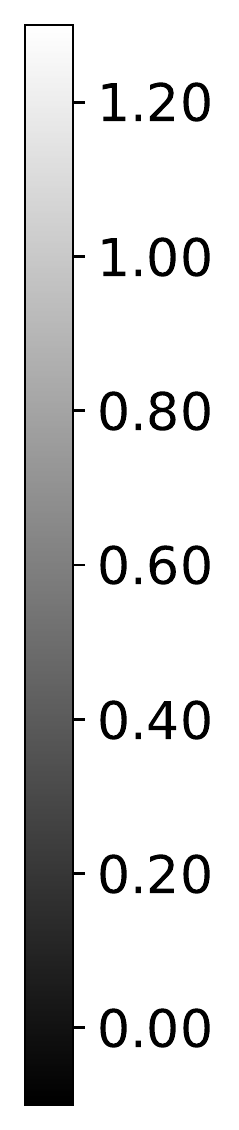}};

% Std
\node at (0,-2cm) (img1) {\includegraphics[scale=0.4]{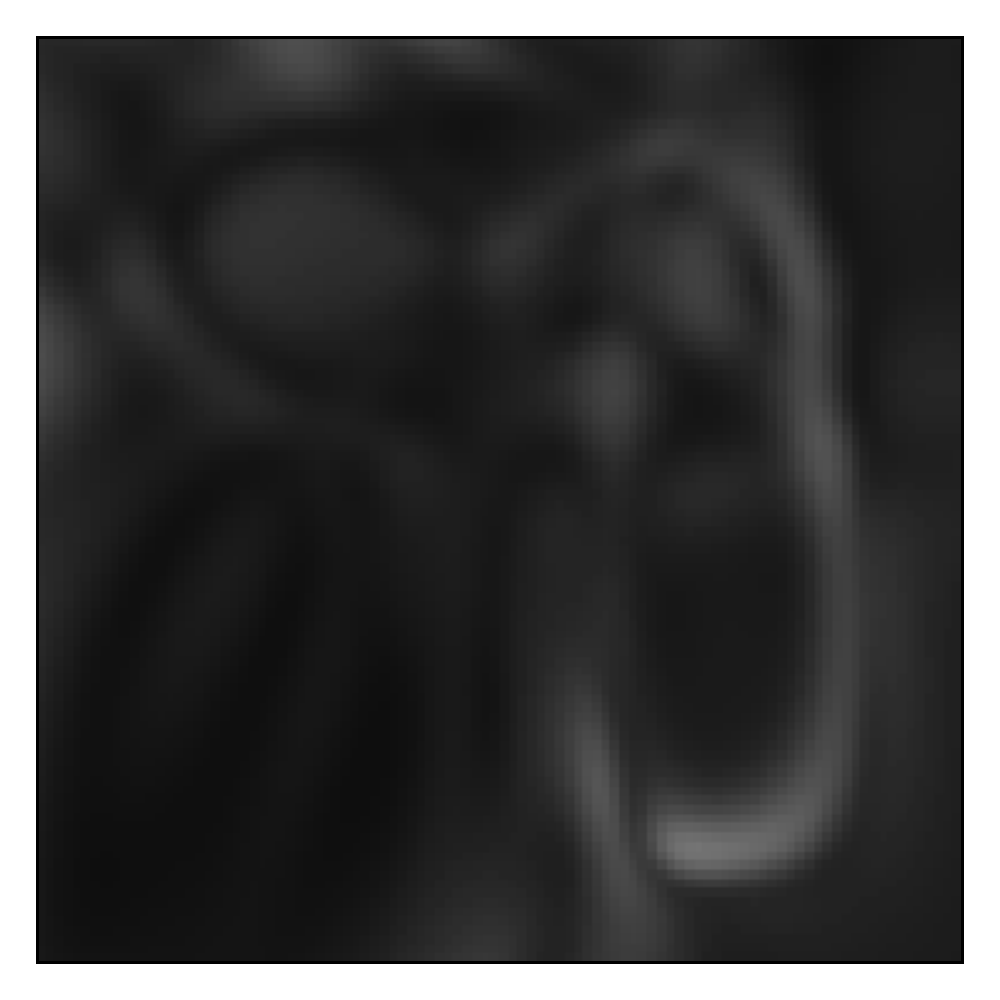}};
\node[left=0cm of img1, node distance=0cm, rotate=90, anchor=center, yshift = 0cm, font=\color{black}] {Std. dev.};
 \node at (3.9,-2cm) (img12) {\includegraphics[scale=0.39]{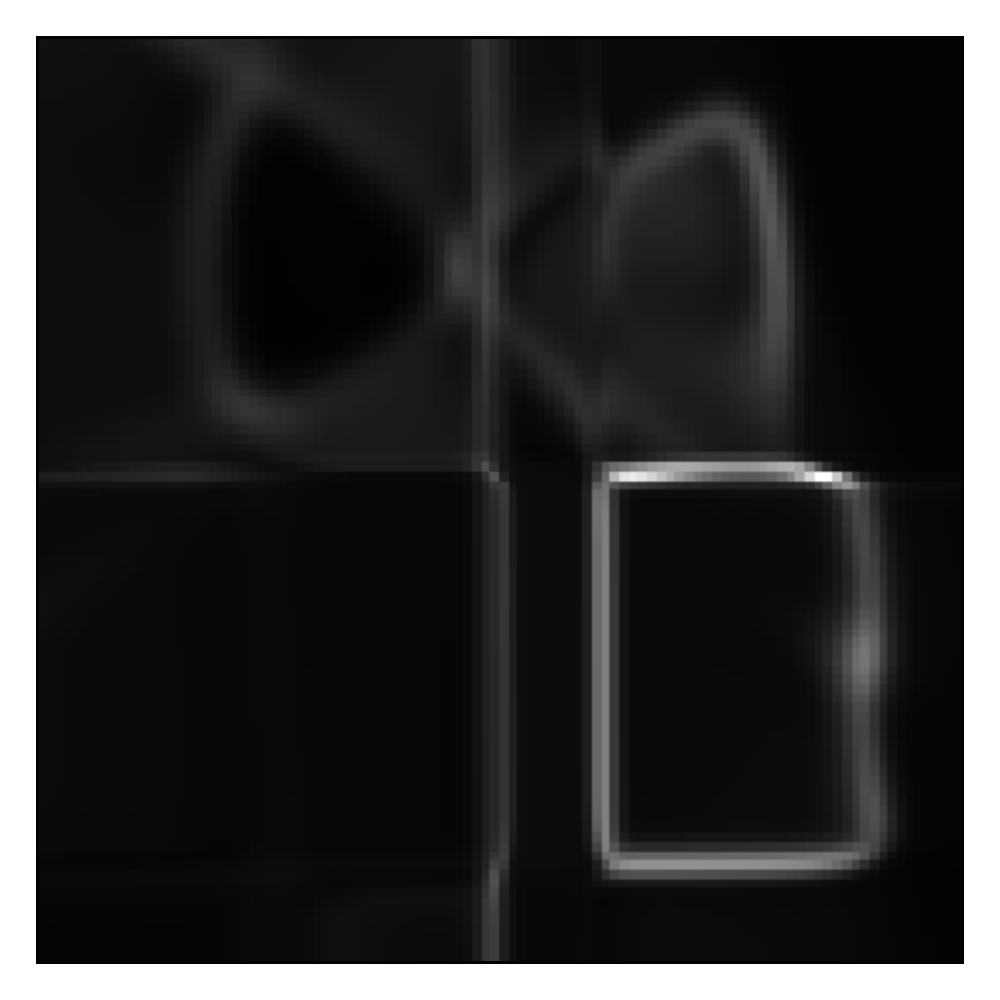}};
 \node at (7.8,-2cm) (img13) {\includegraphics[scale=0.39]{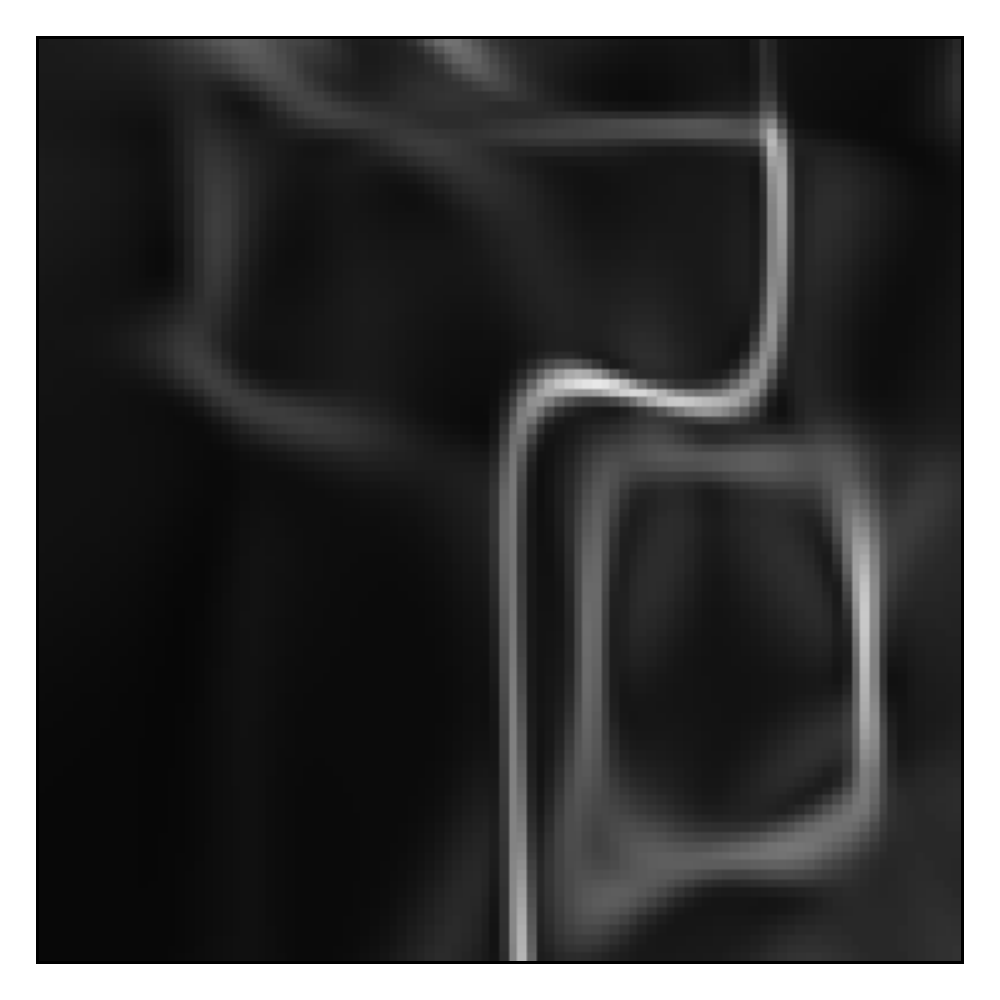}};
\node at (10.0, -2cm) (img13) {\includegraphics[scale=0.333]{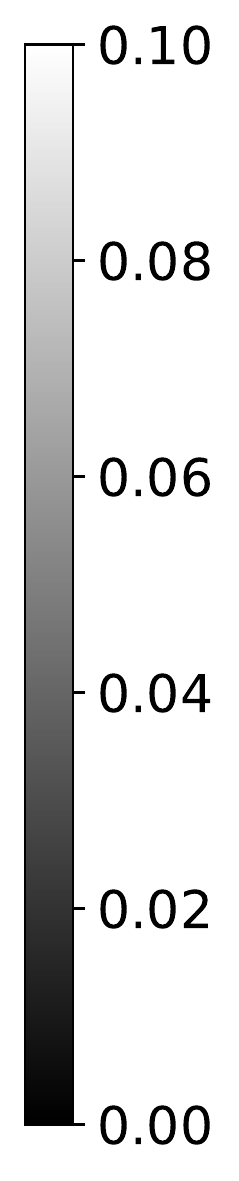}};
\end{tikzpicture}
\caption{MCMC sampling results for Problem~\ref{ex2:deblurring}. Shown
  are the means (top row) and standard deviations (bottom row) for
  neural network priors with Gaussian weights (left), Cauchy-Gaussian
  weights (middle), and fully Cauchy weights
  (right). \label{fig:2dpcn}}
\end{figure}
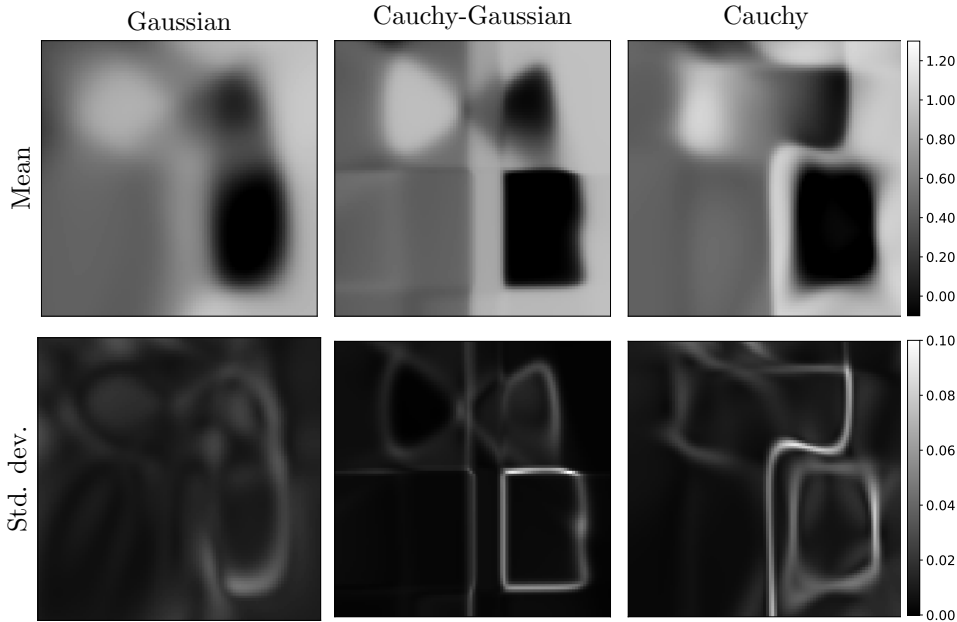

\section{Summary and conclusions}\label{sec:conclu}

The main target of this work is to study neural
network priors for infinite-dimensional Bayesian inverse
problems. Samples of these priors are outputs of neural networks with random weights from different
distributions. Theoretically, we study finite-width neural networks with
$\alpha$-stable weights, and show that the derivative of the output at
a fixed point is heavy-tailed for Cauchy weights. We also present a
numerical comparison of Cauchy and Gaussian neural networks
priors in Bayesian inverse problems. We conclude that:
(1) Neural network priors are able to capture discontinuities and they
are discretization-independent by design. Conditioning these priors
with observations can require many hundreds of evaluations of the
network and the forward map to compute point approximations of the
posterior.  Attempting sampling of the posterior distribution requires
tens of thousands of network and forward map evaluations.
(2) Not unexpectedly, the optimization landscapes for optimization
with Bayesian neural networks have multiple local minima even if the
forward map is linear as in the deblurring examples we study.
(3) We observe that upon optimization, most weights of Cauchy neural
networks are close to zero at (local) minimizers. Optimized Cauchy
networks thus have substantial sparsity, which is a consequence of the
regularization resulting from the Cauchy density.
(4) While we only focused on fully connected networks, one could use
block-diagonal weight matrices, thus requiring substantially fewer
weights. We found in numerical experiments (which are not shown here)
that for Cauchy weights, the resulting distributions for the output do
not differ much between block diagonal and fully connected networks.

\appendix
%\section{Appendix}\label{sec:appendix}
\section{Proof of Lemma~\ref{lemma:heavy_product_sum}}
For convenience, we present a compact proof for this lemma, starting
with the first part (i). The sequence
\begin{align*}
  a_N(t) := \int_{-N}^N e^{t|x|} f_X(x)\,\dee x,
\end{align*}
is increasing with respect to $N$ and $t$ and converges to infinity as
$N \to \infty$ for any given $t > 0$ by definition of heavy-tailed
distributions. Since $Y$ is symmetric, one can find an
interval $[a, b]$, $0 < a < b$,
in which $f_Y(y) \geq \varepsilon > 0$. The density function of
%the random variable
$Z = X Y$ is given by \cite[p. 141]{rohatgi1976introduction}
\begin{align*}
	f_Z(z) = \int_{-\infty}^{\infty} f_X \left(\frac z y \right) f_Y(y) \frac{1}{|y|} \,\dee y.
	\end{align*}
	Thus for any $t > 0$ we have, using a change of variables and
        Fubini's theorem,
	\begin{align*}
		\int_{-N}^{N} e^{t|z|} f_{Z}(z) \,\dee z &= \int_{-N}^{N} e^{t|z|} \int_{-\infty}^{\infty} f_X \left( \frac z y \right) f_Y(y) \frac{1}{|y|} \,\dee y\,\dee z \\
% 		&= \int_{-\infty}^{\infty} f_Y(y)\frac{1}{|y|} \left( \int_{-N}^N e^{t|z|} f_X \left( \frac z y \right) d z\right) dy \\
		&= \int_{-\infty}^{\infty} f_Y(y)\frac{y}{|y|} \left( \int_{-N/y}^{N/y} e^{t|yx|} f_X(x)\,\dee x\right)\,\dee y \\
		&= 2 \int_{0}^{\infty} f_Y(y) \left( \int_{-N/y}^{N/y} e^{t|yx|} f_X(x)\,\dee x\right) \,\dee y \\
		&= 2 \int_{0}^{\infty} f_Y(y) a_{N/y} (ty) \,\dee y 
%		&\leq 2 \varepsilon \int_{a}^b a_{N/y}(ty) \,\dee y \\
		\geq 2 \varepsilon a_{N/b}(ta)(b - a),
	\end{align*}
	which approaches infinity as $N \to \infty$, which gives the
        result.

        For part (ii) of the lemma, since $Y$ is heavy-tailed, for any $t > 0$ at least one of the integrals $\int_{-\infty}^{0} e^{-ty} f_Y(y)\,\dee
y$ or $\int_{0}^{\infty} e^{ty} f_Y(y)\,\dee y$ is infinite. Without
loss of generality, we assume $\int_{0}^{\infty} e^{ty} f_Y(y)\,\dee y
= \infty$. Then for any $t > 0$ and $N > 0$, by the definition of $Z =
X + Y$, we have
\begin{align*}
  \int_{-N}^{N} e^{t|z|} f_{Z}(z)\,\dee z &= \int_{-N}^{N} e^{t|z|}
  \int_{-\infty}^{\infty} f_X(z - y) f_Y(y) \,\dee y \,\dee z \hfill \\
%		&= \int_{0}^{N} e^{tz} \int_{-\infty}^{\infty} f_X(z - y) f_Y(y) dy dz \\
%		&\quad + \int_{-N}^0 e^{-tz} \int_{-\infty}^{\infty} f_X(z - y) f_Y(y) dy dz \\
  &\hspace{-10ex}= \int_{0}^{N} e^{tz} \int_{-\infty}^{\infty} f_X(z - y) f_Y(y) \,\dee y \,\dee z 
  + \int_0^{N} e^{tz} \int_{-\infty}^{\infty} f_X(-z - y) f_Y(y) \,\dee y \,\dee z \\
  &\hspace{-10ex}\geq \int_{-\infty}^{\infty} e^{ty} f_Y(y)
  \int_{0}^{N} \left[e^{t(z - y)} f_X(z - y) + e^{t(-z - y)} f_X(-z - y)  \right] \,\dee z \,\dee y\\
%  + \int_{0}^{\infty} e^{ty} f_Y(y) \int_{0}^{N} e^{t(-z - y)} f_X(-z - y) \,\dee z\,\dee y \\
  &\hspace{-10ex}\geq \int_{0}^{\infty} e^{ty} f_Y(y) \,\dee y \int_{-N - y}^{N-y} e^{tx} f_X(x) \,\dee x,
\end{align*}
where we have used Fubini's theorem and change of variables. Letting
$N\to\infty$, for any $t > 0$ the final expression goes to infinity,
which implies the claim and thus ends the proof.

\section*{Acknowledgments} CL would like to acknowledge helpful
discussions with Yunan Yang. MD would like to thank Neil Chada and Alex Thiery for helpful discussions.

% You may incorporate your references as follows in your main tex file.
% Using BibTex is not recommended but can be handled. If you use BibTex, please include the file with your final paper files.
% AIMS editorial staff will add MR and DOI numbers to your references.

\bibliographystyle{AIMS}
\bibliography{ref}

%\medskip
% The information below will be filled in by AIMS editorial staff
%Received xxxx 20xx; revised xxxx 20xx; early access xxxx 20xx.
%\medskip

\end{document}